\documentclass[]{template}

\usepackage[utf8]{inputenc}             
\usepackage[T1]{fontenc}                
\usepackage{url}                        
\usepackage{booktabs}                   
\usepackage{multirow}
\usepackage{colortbl}
\usepackage{multicol}
\usepackage{amsfonts}                   
\usepackage{nicefrac}                   
\usepackage{microtype}                  
\usepackage[dvipsnames]{xcolor}         

\usepackage{latexsym}

\usepackage{graphicx}
\usepackage{float}
\usepackage{subcaption}
\usepackage{wrapfig}

\usepackage{bm}

\usepackage{tabularx} 
\usepackage{ragged2e} 
\newcolumntype{L}{>{\RaggedRight\hangafter=1\hangindent=0em}X}

\usepackage{enumitem}

\usepackage{amsmath}
\usepackage{amssymb}
\usepackage{mathtools}
\usepackage{amsthm}

\usepackage[linesnumbered,ruled,vlined]{algorithm2e}

\hypersetup{
    colorlinks=true,
    linkcolor=red,
    citecolor=Cerulean,
    filecolor=magenta,      
    urlcolor=magenta,
}

\usepackage[capitalize,noabbrev]{cleveref}
\crefname{section}{§}{§§}
\Crefname{section}{§}{§§}

\usepackage{calligra}
\DeclareMathAlphabet{\mathcalligra}{T1}{calligra}{m}{n}

\usepackage{pifont}

\theoremstyle{plain}
\newtheorem{theorem}{Theorem}[section]
\newtheorem{proposition}[theorem]{Proposition}

\theoremstyle{definition}

\theoremstyle{remark}

\renewcommand{\paragraph}[1]{\vspace{1mm}\noindent\textbf{#1}}

\DeclareCaptionLabelFormat{cont}{#1~#2\alph{ContinuedFloat}}
\usepackage[most]{tcolorbox}
\tcbset{
  promptbox/.style={
    top=10pt,
    colback=lightgray!20,
    colframe=Black,
    colbacktitle=NavyBlue,
    enhanced,
    center,
    attach boxed title to top center={yshift=-0.1in,xshift=0.0in},
    boxed title style={boxrule=0pt,colframe=white,},
  }
}
\newtcolorbox{promptbox}[2][]{promptbox, title=#2,#1}
\tcbset{
  takeawaybox/.style={
    top=10pt,
    colback=lightgray!20,
    colframe=Black,
    colbacktitle=BurntOrange,
    enhanced,
    center,
    attach boxed title to top center={yshift=-0.1in,xshift=0.0in},
    boxed title style={boxrule=0pt,colframe=white,},
  }
}
\newtcolorbox{takeawaybox}[2][]{takeawaybox, title=#2,#1}
\tcbset{
  observationbox/.style={
    top=10pt,
    colback=lightgray!20,
    colframe=Black,
    colbacktitle=YellowGreen,
    enhanced,
    center,
    attach boxed title to top center={yshift=-0.1in,xshift=0.0in},
    boxed title style={boxrule=0pt,colframe=white,},
  }
}
\newtcolorbox{observationbox}[2][]{observationbox, title=#2,#1}

\usepackage{xspace}
\newcommand{\name}[0]{\textsc{RaML}\xspace}

\newcommand\blfootnote[1]{%
  \begingroup
  \renewcommand\thefootnote{}\footnote{#1}%
  \addtocounter{footnote}{-1}%
  \endgroup
}

\usepackage{CJK}

\title{Deciphering Trajectory-Aided LLM Reasoning: An Optimization Perspective}

\author[$\aleph$]{Junnan Liu}
\author[$\aleph$]{Hongwei Liu}
\author[$\aleph$]{Linchen Xiao}
\author[$\aleph$]{Shudong Liu}
\author[$\aleph$]{Taolin Zhang}
\author[$\aleph$]{Zihan Ma}
\author[$\aleph$,$\dagger$]{\\Songyang Zhang}
\author[$\aleph$,$\dagger$]{Kai Chen}

\affil[$\aleph$]{Shanghai AI Laboratory}

\begin{abstract}
We propose a novel framework for comprehending the reasoning capabilities of large language models (LLMs) through the perspective of meta-learning. 
By conceptualizing reasoning trajectories as pseudo-gradient descent updates to the LLM's parameters, we identify parallels between LLM reasoning and various meta-learning paradigms. 
We formalize the training process for reasoning tasks as a meta-learning setup, with each question treated as an individual task, and reasoning trajectories serving as the inner loop optimization for adapting model parameters. 
Once trained on a diverse set of questions, the LLM develops fundamental reasoning capabilities that can generalize to previously unseen questions. 
Extensive empirical evaluations substantiate the strong connection between LLM reasoning and meta-learning, exploring several issues of significant interest from a meta-learning standpoint. 
Our work not only enhances the understanding of LLM reasoning but also provides practical insights for improving these models through established meta-learning techniques.
\end{abstract}

\begin{document}

\blfootnote{$\dagger$ Corresponding authors: Songyang Zhang (zhangsongyang@pjlab.org.cn), Kai Chen (chenkai@pjlab.org.cn)}
\blfootnote{$*$ Code is at \url{https://github.com/open-compass/RaML}}

\maketitle

\section{Introduction}

Recent advancements in large language models~(LLMs) \citep{abs-2407-21783,abs-2412-15115,abs-2303-08774,abs-2412-19437} have significantly improved their capacity to perform complex reasoning tasks. 
Current LLMs often utilize chain-of-thought~(CoT) reasoning, i.e., intermediate reasoning trajectories \citep{Wei0SBIXCLZ22,abs-2503-09567}, to facilitate systematic problem-solving through coherent, step-by-step logical deductions. 
Among them, state-of-the-art LLMs, such as OpenAI-o1 \citep{openaiO1}, DeepSeek-R1 \citep{abs-2501-12948}, Kimi-k1.5 \citep{abs-2501-12599}, QwQ \citep{qwq}, and Gemini-2.5-Pro \citep{gemini25}, exhibit exceptional proficiency in addressing intricate mathematical and programming challenges.
These models employ long reasoning trajectories, characterized by an iterative and detailed process of exploration and reflection, to enhance test-time scaling capabilities \citep{abs-2501-10069,abs-2502-12018,shah2025rethinking}. 
This iterative approach has driven significant progress in complex reasoning while motivating the studies to illuminate the potential of supervised fine-tuning (SFT) and reinforcement learning (RL) methods to refine the learning and application of extended reasoning processes \citep{abs-2410-18982,abs-2412-09413}.

Despite significant advancements, \textbf{\textit{comprehending and interpreting how LLMs achieve prominent reasoning capabilities through reasoning trajectories}} remains crucial for further enhancement and generalization~\citep{JiangXAN20,FengZGY0W23}.
The opaque nature of LLMs' internal mechanisms hinders efforts to comprehend their operations \citep{abs-2503-08200}.
Recent studies \citep{MerrillSS22,0001CP23,GiannouRS0LP23,LiuAGKZ23} have explored the representational power of reasoning trajectories, showing that LLMs equipped with these trajectories can solve complex problems. 
Other research \citep{GatmirySRJK24,abs-2502-21212} demonstrates that reasoning trajectories can effectively describe complex learning algorithms.
Nevertheless, there is a notable gap in research exploring the fundamental role of reasoning trajectories in LLM reasoning and connecting diverse training approaches to enhance these capabilities.
To address this, we propose \textbf{\name} (\underline{R}easoning \underline{a}s \underline{M}eta-\underline{L}earning), a methodology that analyzes LLM reasoning through a meta-learning perspective \citep{Schmidhuber09,AndrychowiczDCH16,RaviL17,FinnAL17,HospedalesAMS22}. 
We conceptualize reasoning trajectories as pseudo-gradient descent updates to model parameters, leveraging established meta-learning methodologies, such as Model-Agnostic Meta-Learning (MAML) \citep{FinnAL17} and Learn to Optimize (L2O) \citep{AndrychowiczDCH16}, to enhance both the understanding and optimization of LLM reasoning.

To be more specific, \name frames the training regimen for reasoning tasks as a meta-learning framework, wherein each question constitutes a distinct task, reasoning trajectories serve as inner-loop optimization for parameter adaptation, and answers act as the query set to optimize LLMs. 
In the context of \name, the training process optimizes the LLM to develop generalized reasoning abilities, identifying an effective \textit{meta-initialization} that enables efficient parameter adaptation through reasoning trajectories to produce accurate responses. 
This approach provides a theoretical foundation for analyzing LLM reasoning capabilities and training, while facilitating the application of meta-learning insights to advance LLM reasoning research.

\name is complemented by comprehensive experiments and analysis involving both models trained from scratch and publicly available models. 
Combining the studies in meta-learning~\citep{LiuWSFZH20,TriantafillouZD20,TriantafillouLZ21,AgarwalYS21,LeeMRS19,CollinsMS22}, we conduct experiments to investigate key factors influencing LLM reasoning by instantiating them from a meta-learning lens. 
Specifically, we explore the effectiveness of two typical LLM reasoning training techniques, SFT and RL, from the perspective of more stable and effective inner loop optimization.
Then, we reveal a parallel between reasoning trajectory count and support set size.
Furthermore, we examine the relationship between inner loop optimization step and reasoning trajectory token, highlighting the varying contributions and roles of tokens to the optimization.
Lastly, we verify the similarities in task generalization between LLM reasoning and meta-learning.
Finally, based on our analyses, we suggest potential pathways for improving LLM reasoning capabilities and conduct preliminary confirmative experiments.
Our contributions are summarized as follows:
\begin{itemize}[leftmargin=5mm]
    \item To elucidate the reasoning processes of LLMs, we introduce \name, an interpretation methodology for LLM reasoning from a meta-learning perspective, supported by a comprehensive theoretical analysis.

    \item We provide empirical evidence and detailed analysis, demonstrating a strong correspondence between LLM reasoning and meta-learning principles.

    \item We contextualize recent advances in LLM reasoning within our framework, offering comprehensions into their success.

    \item We present significant insights to enhance LLM reasoning, building on the existing meta-learning research and analysis addressed in this paper.
\end{itemize}

\section{\name: Interpreting LLM Reasoning as Meta-Learning}

In this section, we elucidate the interpretation methodology for the large language model (LLM) reasoning from a meta-learning~\citep{Schmidhuber09,AndrychowiczDCH16,RaviL17,FinnAL17,HospedalesAMS22} perspective, i.e., \name. 
First, we conceptualize the reasoning trajectories as a \textit{pseudo gradient update} to the parameters of the LLM~(\Cref{sec:reasoning-trajectories-as-parameters-update}) and subsequently develop a meta-learning framework to model the training process for the reasoning task~(\Cref{sec:meta-Learning-perspective-on-llm-reasoning}). 
Lastly, we establish connections between various training techniques and our proposed definition~(\Cref{sec:connection-to-different-training-techniques}).
The notations used in this section are listed in \Cref{app:notations}.

\begin{figure}[t]
    \centering
    \includegraphics[width=.9\linewidth]{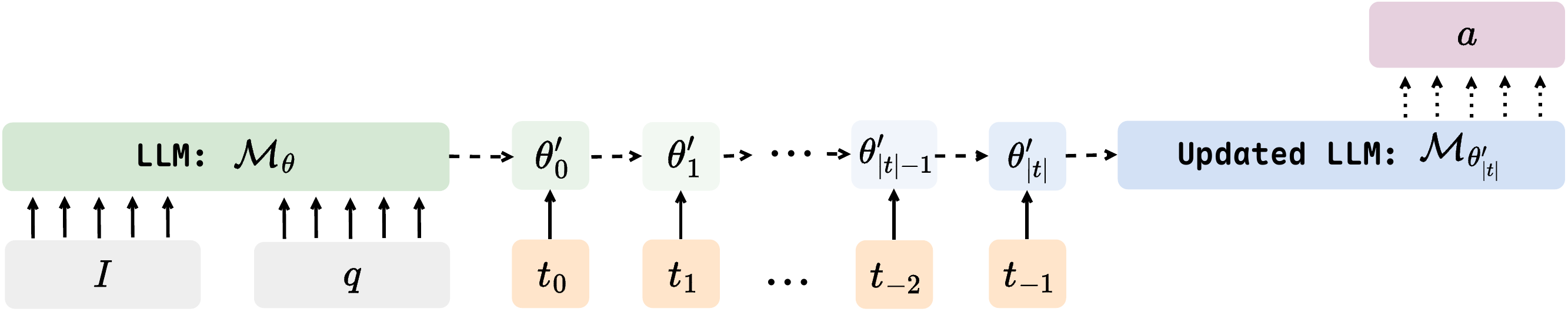}
    \caption{Illustration of the reasoning trajectory~($t$) as the optimization of the LLM parameters $\theta$.}
    \label{fig:reason_trajectories_as_optimization}
\end{figure}

\subsection{Setup}

In this paper, we represent the large language model (LLM) as $\mathcal{M}_\theta$, where $\theta$ signifies the parameters of the LLM.
We focus on a specific reasoning task that involves a set of questions denoted as $\mathcal{Q} = \{q_i\}_{i\in[1,n]}$ and its corresponding answers $\mathcal{A} = \{a_i\}_{i\in[1,n]}$. 
Typically, the LLM is prompted to generate the answer $a$ based on the instruction $I$ and the question $q_i$: 
\begin{equation}
    \mathrm{d}\left( \mathcal{M}_\theta(I;q_i) \right) \rightarrow a_i,
\end{equation}
where $\mathrm{d}$ denotes the autoregressive decoding mechanism~\citep{VaswaniSPUJGKP17,radford2018improving}, which is specifically defined as follows:
\begin{equation} \label{eq:decoding-distribution}
    p_\theta(a_i) = \prod_{0 \le j < \vert a_i \vert} \mathrm{Softmax}\left( \bm{E}_O^T \cdot \mathcal{M}_\theta \left( I,q_i, t, a_i^0, ..., a_i^{j - 1} \right) \right) \left[{a_i^j}\right].
\end{equation}
Here, $\{a_i^0, \ldots, a_i^{\vert a_i \vert}\}$ denotes the token set of the answer $a$, $t$ represents the possible intermediate reasoning trajectories, and $\bm{E}_O$ indicates the output embeddings of the entire token set (i.e., vocabulary). 
Intuitively, $\mathcal{M}_\theta \left( I, q_i, t, a_i^0, \ldots, a_i^{j - 1} \right)$ represents the \textit{activation} determined by the parameters $\theta$ and the inputs, while the predicted probability is computed through the inner product between the output embeddings and the activation. 
The activation, representing the output at the final position, is computed iteratively through the self-attention and feed-forward layers of the LLM.
During training, the LLM $\mathcal{M}$ is optimized to deliver accurate answers to each question: 
\begin{equation} \label{eq:objective}
    \mathcal{O} = \max \sum_{(q_i,a_i) \in \mathcal{Q} \times \mathcal{A}} s(a^\prime_i, a_i),
\end{equation}
where $a^\prime_i$ means the predicted answer for $q_i$, and $s(a^\prime, a)$ indicates the plausibility of $a^\prime$ w.r.t. $a$, which also defines a \textit{loss landscape}.

\subsection{Reasoning Trajectories as Parameter Update} \label{sec:reasoning-trajectories-as-parameters-update}

Recent works~\citep{BrownMRSKDNSSAA20,Wei0SBIXCLZ22,openaiO1,abs-2501-12948} demonstrate that incorporating intermediate reasoning steps can significantly enhance the capacity of large language models to execute complex reasoning tasks.
Moreover, some studies~\citep{dai2023can,BaiCWXM23,GiannouRS0LP23,GatmirySRJK24,FuCJS24,abs-2502-21212} theoretically demonstrate that models based on the transformer architecture can learn to perform iterative algorithms like multi-step GD with the enhancement of CoT~(which we called reasoning trajectories in this paper).
However, these studies primarily focus on explicit numerical optimization problems, such as linear regression, and demonstrate that LLMs can learn optimization algorithms like multi-step GD in the reasoning trajectories to solve the problem. 
In contrast, we conceptualize \textit{the reasoning trajectories of an LLM $\mathcal{M}$ as a multi-step gradient descent process of the model's parameters $\theta$}, which could be formally represented by: 
\begin{equation} \label{eq:multi-step-update}
    \theta^\prime_i \leftarrow \theta^\prime_{i-1} + \Delta \mathcal{M}_{\theta^\prime_{i-1}}(I,q,t_{\le i}), \quad \theta_0^\prime=\theta, \quad 1 \le i \le \vert t \vert,
\end{equation}
where $t$ denotes a reasoning trajectory, $\Delta \mathcal{M}_{\theta^\prime_{i-1}}(I,q,t_{\le i}) = -\eta \nabla_{\theta^\prime_{i-1}} \mathcal{L}_q(\theta^\prime_{i-1}) $ represents the \textit{pseudo gradient update} associated with the reasoning trajectory $t_{\le i}$, and $\theta^\prime_{\vert t \vert}$ signifies the updated parameters of the LLM in response to the instruction $I$, the query $q$, and the reasoning trajectory $t$.

In summary, we conceptualize each question $q_i$ as a sophisticated optimization task, with the LLM $\mathcal{M}$ being optimized to produce the corresponding answer $a_i$. 
Prior to generating the final answer, the LLM is guided by an intermediate reasoning trajectory, which serves as a parameter update mechanism.
The overall process is illustrated in \Cref{fig:reason_trajectories_as_optimization}.

\paragraph{Pseudo Gradient Update. }
Without loss of generality, we consider a classic transformer model~\citep{VaswaniSPUJGKP17} comprising a single self-attention layer and a two-layer feed-forward network while disregarding normalization layers and other components.
When using $ l = \{I, q\} $ as input, its activation can be expressed as follows:
\begin{equation}
    \bm{W}_2^T \left( \sigma \left( \bm{W}_1^T \left( \mathrm{Softmax} \left( \bm{E}_{l,-1} \bm{W}_q \bm{W}_k^T \bm{E}_{l,:}^T \right)  \bm{E}_{l,:} \bm{W}_v \right) + b_1\right) \right) + b_2,
\end{equation}
where $\bm{E}_{l,:}$ indicates the input embeddings of the whole sequence $l$ and $\bm{E}_{l,-1}$ indicates the input embeddings of the last position of $l$.
In this context, the parameters $\theta$ refers to $\left\{\bm{W}_q, \bm{W}_k, \bm{W}_v, \bm{W}_1, \bm{W}_2, b_1, b_2 \right\}$.
Then, given a reasoning trajectory $t$, when attending to the first token $t^0$ of $t$ activation is changed to:
\begin{equation} \label{eq:activation-attend-first-traj-token}
    \bm{W}_2^T \left( \sigma \left( \bm{W}_1^T \left( \mathrm{Softmax} \left( \bm{E}_{t,0} \bm{W}_q \bm{W}_k^T  
        \left [ \begin{matrix}
            \bm{E}_{l,:} \\
            \bm{E}_{t,0} \\
        \end{matrix} \right ]^T
        \right) 
        \left [ \begin{matrix}
            \bm{E}_{l,:} \\
            \bm{E}_{t,0} \\
        \end{matrix} \right ]
        \bm{W}_v \right) + b_1\right) \right) + b_2.
\end{equation}
\begin{proposition}[One-Step Pseudo Gradient Update]\label{pro:one-step-pseudo-update}
    There exists a set of parameters, denoted as $\theta_t^\prime$, which includes $\left\{\bm{W}_q^\prime, \bm{W}_k^\prime, \bm{W}_v^\prime, \bm{W}_1^\prime, \bm{W}_2^\prime, b_1^\prime, b_2^\prime \right\}$, allowing \Cref{eq:activation-attend-first-traj-token} to be expressed in the following form:
    \begin{equation}
        \bm{W}_2^{\prime T} \left( \sigma \left( \bm{W}_1^{\prime T} \left( \mathrm{Softmax} \left( \bm{E}_{l,-1:} \bm{W}_q^\prime \bm{W}_k^{\prime T} \bm{E}_{l,:}^T \right)  \bm{E}_{l,:} \bm{W}_v^\prime \right) + b_1^\prime \right) \right) + b_2^\prime,
    \end{equation}
     where $\theta_t^\prime$ represents the one-step update of $\theta$ and the increment $\Delta \mathcal{M}_{\theta}(I,q,t^0)$ is only associated with $\theta$, $I$, $q$, and $t^0$.
\end{proposition}
According to \Cref{pro:one-step-pseudo-update}~(the proof can be found in \Cref{sec:proof_one-step-pseudo-update}), as the model progressively attends to the entire reasoning trajectory, the model parameters $\theta$ are updated incrementally, a process referred to as the \textit{pseudo gradient update}.

\paragraph{Empirical Examples. }
We provide empirical examples of the pseudo-gradient update using QwQ-32B~\citep{qwq} to perform reasoning on AIME24, in which the model's confidence in the answer functions serves as a probe. 
Specifically, we calculate the negative log-probability of the answer at each position~(denoted as $\mathcal{\widehat{L}}$) within the generated trajectories by appending \texttt{Final Answer$\backslash$n$\backslash$boxed\{..answer..\}} at each position. 
This method provides an alternative approach to observing the overall optimization objective~(\Cref{eq:objective}), with models becoming more optimal as $\mathcal{\widehat{L}}$ decreases.
\Cref{fig:loss_landscape} displays the corresponding landscape of the negative log-probability. 
As shown in \Cref{fig:qwq_gradient_decent}, the negative log-probability progressively decreases along the reasoning trajectories which aligns with our definition.
Additional visualizations are provided in \Cref{app:qwen3_update}.

\begin{figure}[tb]
    \centering
    \includegraphics[width=.9\linewidth]{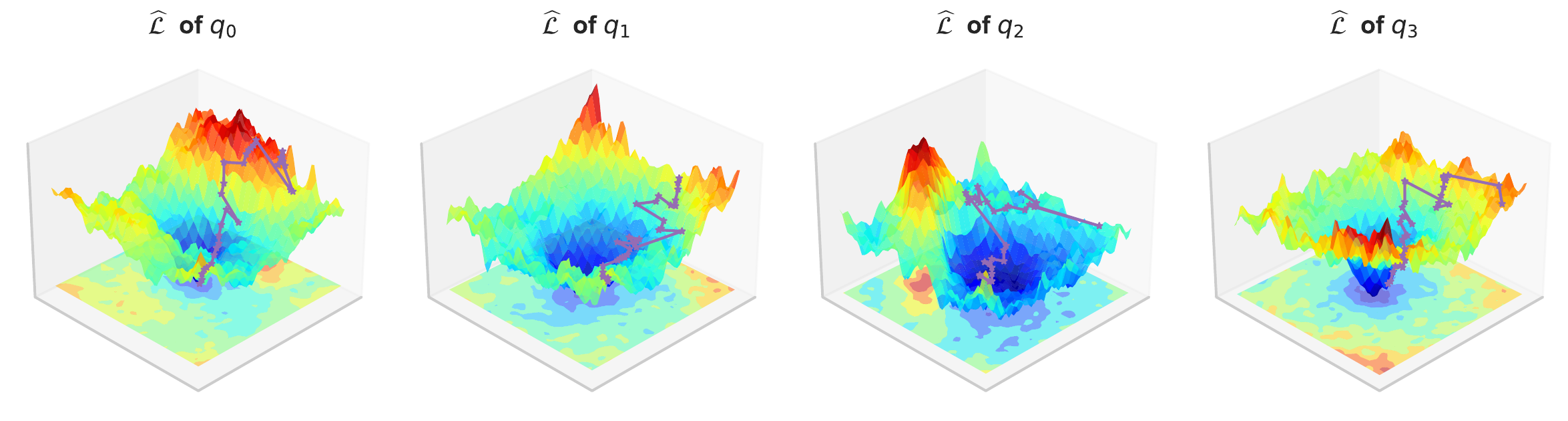}
    \caption{Landscape of the plausibility regarding LLMs to generate accurate answers. We apply the methodology proposed by Li et al.~\citep{Li0TSG18}. The questions $q_0, q_1, q_2, q_3$ are selected from AIME24. Additionally, we project the trajectory of the pseudo-gradient update onto the landscape~(\textcolor{Purple}{purple} line). Please refer to \Cref{sec:demonstrated-questions} for more details.}
    \label{fig:loss_landscape}
    \vspace{-0.5em}
\end{figure}
\begin{figure}[tb]
    \centering
    \includegraphics[width=.9\linewidth]{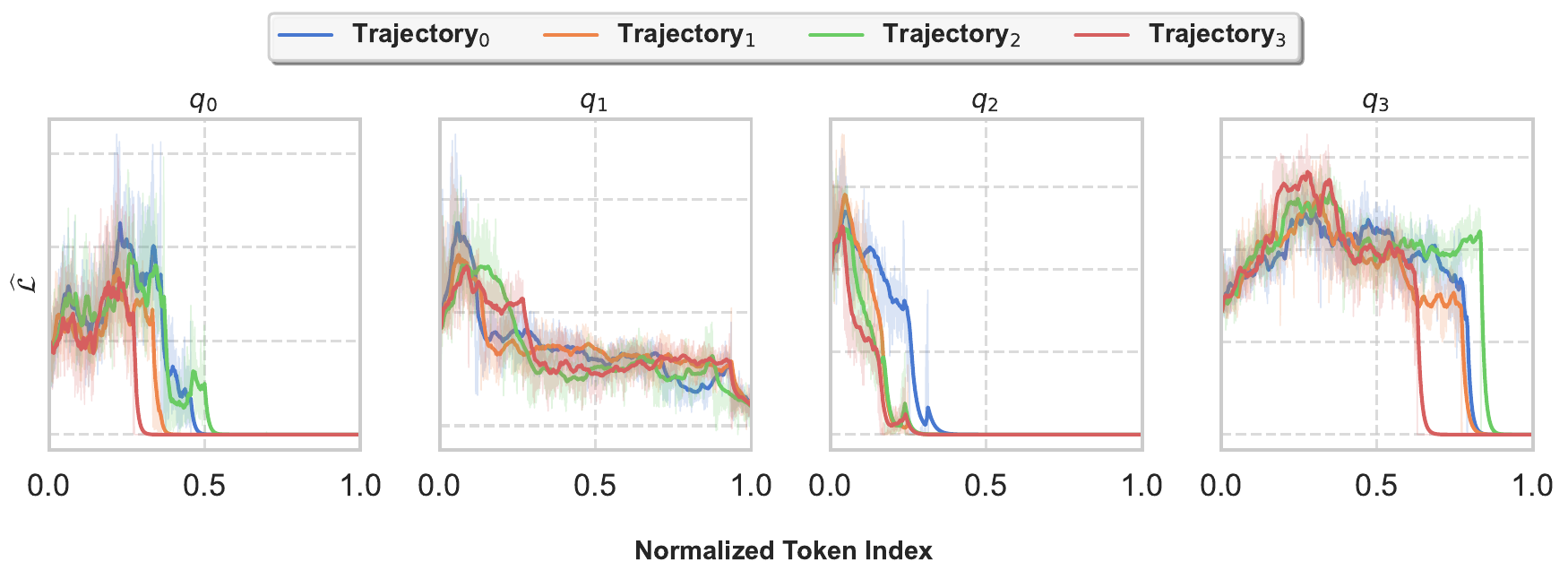}
    \caption{Visualization of the \textit{pseudo-gradient update}: The $x$-axis represents the normalized indices of corresponding trajectories. $q_0,q_1,q_2,q_3$ are question selected from AIME24, refer to \Cref{sec:demonstrated-questions} for more details.}
    \label{fig:qwq_gradient_decent}
    \vspace{-1.0em}
\end{figure}

\subsection{A Meta-Learning Perspective on LLM Reasoning} \label{sec:meta-Learning-perspective-on-llm-reasoning}

Building upon the previous discussion, we present a meta-learning perspective, named \name, for modeling the process and the training of the LLM reasoning capability.
We consider each question $q_i$ within the question set as an independent task in the meta-learning.
During training (e.g., supervised fine-tuning~\citep{RuderH18,radford2018improving,DevlinCLT19} or reinforcement learning~\citep{MnihKSGAWR13,SchulmanWDRK17,Ouyang0JAWMZASR22,abs-2501-12948}) on the question set $\{q_i\}$, the LLM $\mathcal{M}_\theta$ is prompted to solve a new question $q_i$ by following a reasoning trajectory $t$. 
Initially, the parameters $\theta$ are updated to $\theta_t^\prime$ using one or more \textit{pseudo gradient descent update} indicated by reasoning trajectory $t$. 
Subsequently, the LLM is optimized by \textit{pseudo second order gradient}, formally expressed as follows:
\begin{equation}
    \min_\theta \sum_{q_i \in \mathcal{Q}} \sum_{t \in \mathcal{T}_i} \mathcal{L}_{q_i} \left( \mathcal{M}_{\theta^\prime_t} \right) = \min_\theta \sum_{q_i \in \mathcal{Q}} \sum_{t \in \mathcal{T}_i} \mathcal{L}_{q_i} \left( \mathcal{M}_{\theta + \Delta \mathcal{M}_{\theta}(I,q,t)} \right),
\end{equation}
where $\mathcal{T}_i$ denotes the set of reasoning trajectories corresponding to the question $q_i$ and $\theta + \Delta \mathcal{M}_{\theta}(I,q,t)$ indicates the multi-step update of $\theta$ as detailed in \Cref{eq:multi-step-update}.
\begin{figure}
    \centering
    \begin{minipage}[t]{.45\linewidth}
    \centering
    \vspace{-1.0em}
    \IncMargin{1em}
    \begin{algorithm}[H]
        \KwIn{$p\left( \mathcal{T} \right)$: distribution over tasks, $\alpha,\beta$: step size hyperparameters.}
        \BlankLine
        Randomly initialize $\theta$ \;
        \While{not done}{
            Sample batch of tasks $\mathcal{T}_i \sim p\left( \mathcal{T} \right)$ \;
            \For{\textbf{all} $\mathcal{T}_i$}{
                Evaluate $\nabla \mathcal{L}_{\mathcal{T}_i} \left( f_\theta \right)$ with respect to $K$ examples \;
                Compute adapted parameters with gradient descent: $\theta^\prime = \theta - \alpha \nabla_\theta \mathcal{L}_{\mathcal{T}_i} \left( f_\theta \right)$ \;
            }
        }
        \BlankLine
        Update $\theta \leftarrow \theta - \beta \nabla \sum_{\mathcal{T}_i \sim p \left( \mathcal{T} \right)} \mathcal{L}_{\mathcal{T}_i} \left( f_{\theta^\prime} \right)$ \;
        \caption{Model-Agnostic Meta-Learning} \label{algo:maml}
    \end{algorithm}
    \end{minipage}
    \begin{minipage}[t]{.54\linewidth}
    \centering
    \vspace{-1.0em}
    \IncMargin{1em}
    \begin{algorithm}[H]
        \KwIn{$\mathcal{M}_\theta$: LLM, $I$: instruction, $\mathcal{Q}$: question set, $\mathcal{T}_i$: reasoning trajectories for each question $i$.}
        \While{not training done}{
            Sample batch of questions $q_i$ from $\mathcal{Q}$ \;
            \For{\textbf{all} $q_i$}{
                Obtain reasoning trajectories $\mathcal{T}_i$ of each $q_i$ through training data or rollout \;
                Update $\theta$ to $\theta^\prime_{t_j}$ by reasoning trajectory $t_j \in \mathcal{T}_i$ refer to \Cref{eq:multi-step-update} \;
                Optimize $\theta$ through $\sum_{t_j \in \mathcal{T}_i} \mathcal{L}_{q_i} \left( \mathcal{M}_{\theta^\prime_{t_j}} \right)$ for each $q_i$ \;
            }
        }
        \caption{Meta-Learning Perspective on LLM Reasoning} \label{algo:meta-learning-perspective-on-llm-reasoning}
    \end{algorithm}
    \end{minipage}
\end{figure}

Intuitively, \name can be perceived as a variant of \textbf{Model-Agnostic Meta-Learning} (MAML, detailed in \Cref{algo:maml})~\citep{FinnAL17}, where the update of $\theta$ using reasoning trajectories function as the inner loop, while the final optimization of the answer decoding distribution constitutes the outer loop, as outlined in \Cref{algo:meta-learning-perspective-on-llm-reasoning}. 
In \name, the gradient update associated with the latent \textit{support set} is represented by the reasoning trajectories, whereas the answer denotes the \textit{query set}.
There are no explicit evaluations~(i.e., loss computation and backward) during the inner loop, as the gradient update is implicitly dictated by the reasoning trajectories. 
Although the model weights have not been directly updated, the pseudo update enables the LLM to simulate question-specific optimization within a specific reasoning trajectory, thereby significantly enhancing the accuracy and stability of LLM reasoning.
During the training process, the parameters of the LLM, denoted as $\theta$, are updated to $\theta^\prime_{q_i,t_j}$ according to the given reasoning trajectory $t_j$ in the inner loop. 
In the outer loop, the parameters of the LLM are optimized using the second-order gradient ($\mathcal{L}_{q_i} \rightarrow \theta^\prime_{q_i,t_j} \rightarrow \theta$).
The LLM is optimized to provide an effective and robust foundation for answering questions~(tasks), \textit{allowing its parameters to be easily adapted based on the reasoning trajectories associated with these questions}, thereby facilitating answer generation.

Additionally, in the standard MAML process, a few-shot support set is typically required to fine-tune the model on a new task. 
In the LLM reasoning scenario, this support set, comprising reasoning trajectories, is generally generated by the LLM itself. 
Thus, the inner loop's optimization process in \name resembles \textbf{Learn-to-Optimize} (L2O)~\citep{AndrychowiczDCH16,LiM17,LiM17b}, which involves learning a parameterized optimizer to automate the optimization of various tasks by being trained to decode the reasoning trajectories. 
Specifically, during the LLM reasoning training, the LLM is trained to function as the meta-optimizer, generating an inner loop optimization path tailored to the specific question.

\subsection{Instantiation of Training Techniques within \name} \label{sec:connection-to-different-training-techniques}

\begin{table}[tb]
    \centering
    \caption{Comparison of training techniques, where SFT, PO, and RL mean the abbreviation of supervised fine-tuning, preference optimization, and reinforcement learning, respectively.} \label{tab:comparison-between-training-techniques}
    \vspace{0.5em}
    \resizebox{.95\linewidth}{!}{
    \begin{tabular}{lcc}
        \toprule
        \textbf{Techniques} & \centering \textbf{Reasoning Trajectories} & \textbf{Outer Loop Loss} \\
        \midrule
        SFT & Off-policy & $\mathcal{L} = - \mathbb{E}_{(x, y) \sim \mathcal{D}} \left[ \log p_\theta(y \mid x) \right]$~\citep{radford2018improving} \\
        Off-Policy PO & Off-policy & $\mathcal{L} = - \log \sigma \left( r_\theta(y_{\text{preferred}}) - r_\theta(y_{\text{dispreferred}}) \right)$~\citep{RafailovSMMEF23} \\
        On-Policy RL & On-policy & $\mathcal{L} = - \mathbb{E}_{t} \left[ \min\left( r_t(\theta) \hat{A}_t, \, \text{clip}(r_t(\theta), 1 - \epsilon, 1 + \epsilon) \hat{A}_t \right) \right]$~\citep{SchulmanWDRK17} \\
        \bottomrule
    \end{tabular}
    }
    \vspace{-0.5em}
\end{table}

We review various training techniques from a meta-learning perspective, including supervised fine-tuning~\citep{RuderH18,DevlinCLT19}, off-policy preference optimization~\citep{RafailovSMMEF23}, and on-policy reinforcement learning~\citep{SchulmanWDRK17,Ouyang0JAWMZASR22,abs-2402-03300,AhmadianCGFKPUH24}.
We propose to categorize these techniques into two macro-level stages. 
The first stage involves acquiring reasoning trajectories and inputting them into the LLM to update its parameters $\theta$, thereby obtaining the output token distribution through updated $\theta$.
Subsequently, the LLM parameters $\theta$ are optimized using a specific loss function based on this output distribution. 
Since various loss functions lead to the same maximum likelihood estimation (MLE)~\citep{abs-2503-01067}, we attribute the essential difference between different training techniques to their inner loop optimization.
Inner loop optimization is crucial in meta-learning training, as the meta-gradient is essential for enhancing meta-learning performance.
As illustrated in \Cref{tab:comparison-between-training-techniques}, off-policy training techniques obtain reasoning trajectories through manual collection or synthesis, while on-policy training techniques generate reasoning trajectories based on the model distribution.
From the perspective of learning to optimize, off-policy training techniques are equivalent to learning from an \textit{optimal meta-optimizer}, directing the optimization of the inner loop. 
In contrast, RL requires independently exploring the inner loop's optimization path, presenting challenges due to increased freedom but allowing for potentially greater optimization outcomes.

\section{Empirical Analysis on LLM Reasoning from Meta-Learning Perspective}

Building on a meta-learning perspective of LLM reasoning, this section explores key factors that influence LLM reasoning.
Specifically, we study and analyze key issues of interest in the research community regarding LLM reasoning by instantiating them within the framework of meta-learning, referencing relevant research findings in this domain.
We focus on the following problems: 
\begin{enumerate}[leftmargin=5mm]
    \item[\ding{182}] Which training strategy, SFT or RL, is more effective for LLM reasoning, and why (\Cref{sec:source-of-rt})?
    \item[\ding{183}] Why do longer reasoning trajectories enhance reasoning performance (\Cref{sec:length-of-rt})? 
    \item[\ding{184}] What principles behind reasoning-efficiency methodology contribute to the trade-off between cost and performance (\Cref{sec:length-of-rt})?
    \item[\ding{185}] Does trajectory-aided reasoning generalize effectively across different domains (\Cref{sec:ood-question})?
\end{enumerate}

\begin{table}[tb]
    \centering
    \caption{The evaluation performance of Qwen2.5-7B-Base trained using both SFT and Zero-GRPO training methods. In this context, "Qwen" refers to the abbreviation for Qwen2.5-Math-Instruct, and "Distil-Qwen" denotes DeepSeek-R1-Distill-Qwen-14B. \textcolor{green!50!black}{\textbf{Green}} cells indicate the best performance in each column, while \textcolor{blue!70!black}{\textbf{Blue}} cells indicate the second-best performance.}
    \label{tab:sft-vs-rl-performance}
    \resizebox{.95\linewidth}{!}{
    \begin{tabular}{llcccccc}
        \toprule
        \multirow{2}{*}{\textbf{Techniques}} & \multirow{2}{*}{\textbf{Source}} & \multicolumn{2}{c}{\textbf{AIME24}} & \multicolumn{2}{c}{\textbf{MATH500-L5}} & \multicolumn{2}{c}{\textbf{LiveMathBench-Hard}} \\
        & & \textbf{Pass@}$8$~$\uparrow$ & \textbf{mG-Pass@}$8$~$\uparrow$ & \textbf{Pass@}$8$~$\uparrow$ & \textbf{mG-Pass@}$8$~$\uparrow$ & \textbf{Pass@}$8$~$\uparrow$ & \textbf{mG-Pass@}$8$~$\uparrow$ \\
        \midrule
        \multirow{2}{*}{SFT} & Qwen & $20.34$ & \cellcolor{blue!10}$7.43$ & $58.42$ & \cellcolor{blue!10}$35.65$ & \cellcolor{blue!10}$26.77$ & $7.43$  \\
         & Distill-Qwen & \cellcolor{green!20}$36.69$ & \cellcolor{green!20}$10.29$ & \cellcolor{green!20}$82.98$ & \cellcolor{green!20}$45.79$ & $25.15$ & \cellcolor{green!20}$10.46$  \\
        Zero-GRPO & - & \cellcolor{blue!10}$27.37$ & $4.08$ & \cellcolor{blue!10}$71.66$ & $30.48$ & \cellcolor{green!20}$27.48$ & \cellcolor{blue!10}$8.21$  \\
        \bottomrule
    \end{tabular}
    }
    \vspace{-0.2em}
\end{table}

\subsection{Experiment Setup}
\paragraph{Reasoning Task. } In this paper, we mainly focus on the mathematical reasoning task due to its broad applicability and prominence in the research community and we also include other reasoning tasks in \Cref{sec:ood-question} for the study of generalization.

\paragraph{Training. } \label{sec:experiment-setup-training}
To minimize the impact of the post-training, we train Qwen2.5-7B-Base~\citep{abs-2412-15115} from scratch and conduct experiments on it. 
We involve SFT for the off-policy training and (\textbf{Zero}-)GRPO~\citep{abs-2402-03300} for the on-policy training. 
The training data, sourced from Open Reasoner Zero \citep{hu2025open}, initially comprised approximately $57$k questions, refined to $39$k through filtering~(see \Cref{app:training-details} for details). 
Synthetic reasoning trajectories are generated using Qwen2.5-Math-72B-Instruct \citep{abs-2409-12122} and DeepSeek-R1-Distill-Qwen-14B \citep{abs-2501-12948}. 
Further training details are provided in \Cref{app:training-details}.

\paragraph{Evaluation. } 
We primarily evaluate performance using three mathematical reasoning benchmarks orthogonal to the training data: AIME24 \footnote{\url{https://huggingface.co/datasets/AI-MO/aimo-validation-aime}}, MATH500~\citep{LightmanKBEBLLS24}~(Level 5 questions selected for greater discrimination), and LiveMathBench-Hard~\citep{abs-2412-13147}. 
We also include GPQA~\citep{abs-2311-12022} and LiveCodeBench~\citep{abs-2403-07974} to assess generalization. 
Model outputs are generated with a temperature of $1.0$, top-$p$ of $0.8$, top-$k$ of $50$, and a maximum output length of $16,384$ tokens. We report mG-Pass@$k$~\citep{abs-2412-13147} for stability and Pass@$k$~\citep{abs-2107-03374} for performance. 
Additional evaluation details are in~\Cref{app:evaluation-details}.

\begin{figure}[tb]
    \centering
    \includegraphics[scale=.4]{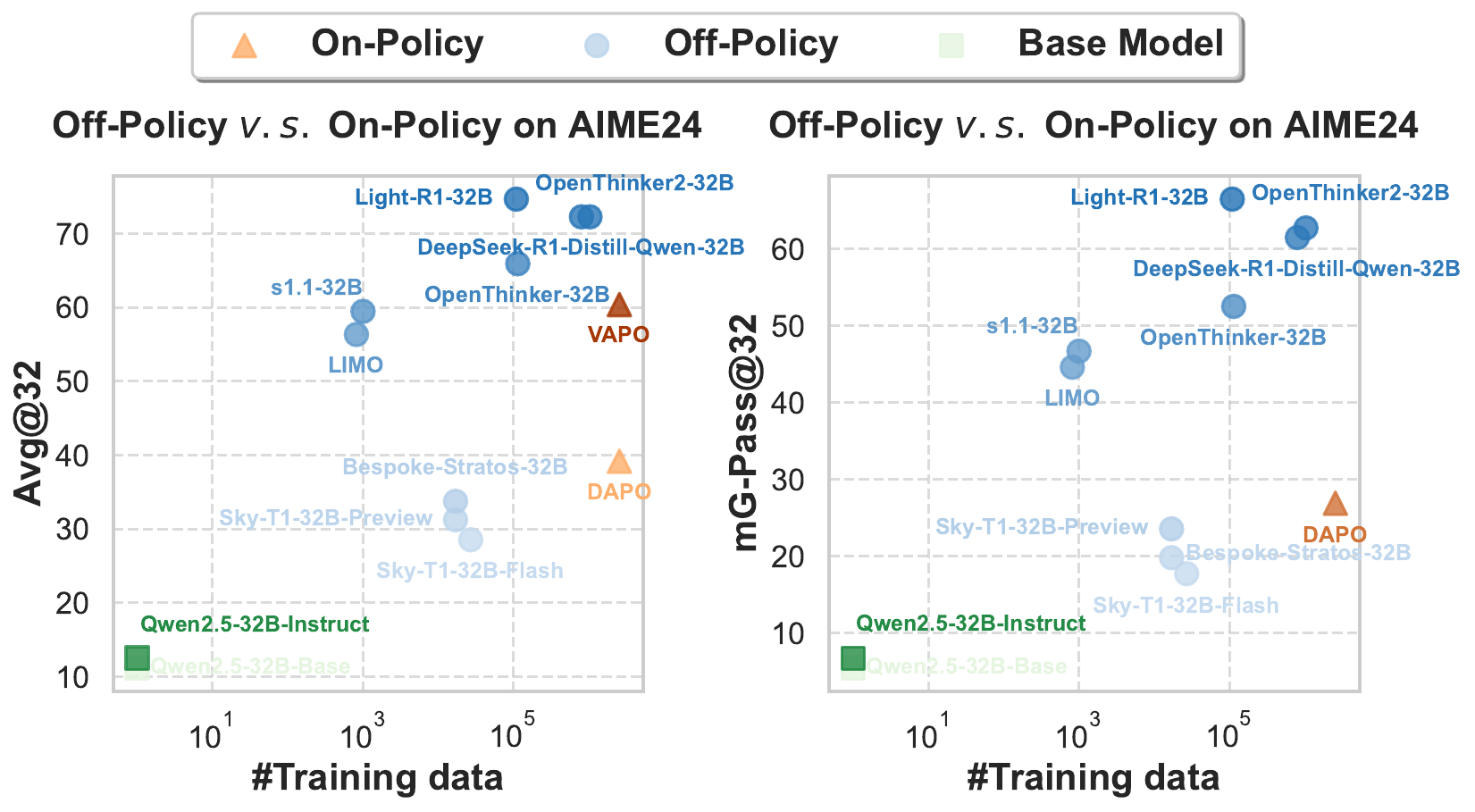}
    \caption{Performance of base models, models trained on off-policy data, and models trained on on-policy data using the AIME24 dataset, with the $x$-axis representing the amount of training data. We generate $64$ for each question and report Pass@$32$ and mG-Pass@$32$. The evaluation includes prominent models such as Sky-T1-32B~\citep{sky_t1_2025}, Bespoke-Stratos-32B~\citep{bespoke_stratos}, LIMO~\citep{abs-2502-03387}, s1.1-32B~\citep{abs-2501-19393}, OpenThinker-32B~\citep{openthoughts}, Light-R1-32B~\citep{abs-2503-10460}, DeepSeek-R1-Distill-Qwen-32B~\citep{abs-2501-12948}, DAPO-32B~\citep{abs-2503-14476}, and VAPO-32B~\citep{abs-2504-05118}. These models are based on either Qwen2.5-32B or Qwen2.5-32B-Instruct only through SFT or RL~(\textbf{Zero-RL}). Since VAPO is not open source, we copy its results from the original paper.} \label{fig:sft_vs_rl}
    \vspace{-0.2em}
\end{figure}

\subsection{Inner Loop Optimization \textit{v.s.} Reasoning Trajectory Source} \label{sec:source-of-rt}
The inner loop optimization is crucial in meta-learning, as the results of this process significantly impact the stability and performance of the final model. 
In the definition of \name, the LLM needs to learn as the inner loop optimizer, generating an optimization trajectory for each query. 
It is well-documented that training learned optimizers presents considerable challenges \citep{LanMYX24}. 
In this section, we will discuss the development of inner loop optimizers for SFT and GRPO training techniques.
The biggest difference between them is the source of the reasoning trajectories used for training: 1) on-policy, where the trajectories are generated by the LLM currently being updated, and 2) off-policy, where the trajectories are generated either by other LLMs or by a previously trained version of the same LLM~(e.g., through reject sampling~\citep{abs-2308-01825,SinghCAAPGLH0XP24}).

\paragraph{Status Quo of SFT v.s. RL. } 
Recent studies \citep{0015DYW0J0024,chen2025sft,abs-2501-17161} claim the superiority of the on-policy strategy in LLM reasoning training.
For example, GRPO-based DeepSeek-R1-Zero \citep{abs-2501-12948} outperforms DeepSeek-V3, which is trained on large-scale off-policy synthetic data, in mathematical reasoning tasks, scoring $71.0$ compared to $39.8$ on AIME24, thereby reinforcing the advantages of on-policy strategies.
However, as our results in \Cref{tab:sft-vs-rl-performance} and the evaluation results of community models in \Cref{fig:sft_vs_rl}, GRPO-trained models \textbf{do not} consistently outperform SFT-trained models for the same base LLM, consistent with findings in \citep{abs-2501-12948,abs-2503-10460,qwen3}. 

\paragraph{SFT Leads to Stable Inner Loop Optimization. } 
Learning to optimize frequently encounters challenges such as unstable training, easy divergence, and limited generalization. 
To address these issues, researchers~\citep{Premont-Schwarz22,abs-2406-00153} have suggested employing optimal optimizers as ``guardian'' optimizers, integrating their features to ensure convergence and stability. 
The training reasoning trajectories used by SFT originate from human-annotated or other advanced reasoning models.
These trajectories can be viewed as guides from an \textit{oracle optimizer}. 
Consequently, SFT achieves a stable and effective inner loop optimization process, leading to superior performance. 
However, this does not imply that reinforcement learning always has disadvantages. RL provides greater freedom to explore optimization paths and, given sufficient model capability and exploration steps, offers a higher theoretical upper limit.

\begin{table}[tb]
    \centering
    \caption{Performance of Zero-GRPO model and GRPO model based on the SFT cold start.}
    \label{tab:rl-after-sft-performance}
    \resizebox{.95\linewidth}{!}{
    \begin{tabular}{lcccccc}
        \toprule
        \multirow{2}{*}{\textbf{Techniques}} & \multicolumn{2}{c}{\textbf{AIME24}} & \multicolumn{2}{c}{\textbf{MATH500-L5}} & \multicolumn{2}{c}{\textbf{LiveMathBench-Hard}} \\
        & \textbf{Pass@}$8$~$\uparrow$ & \textbf{mG-Pass@}$8$~$\uparrow$ & \textbf{Pass@}$8$~$\uparrow$ & \textbf{mG-Pass@}$8$~$\uparrow$ & \textbf{Pass@}$8$~$\uparrow$ & \textbf{mG-Pass@}$8$~$\uparrow$ \\
        \midrule
        Zero-GRPO & $27.37$ & $4.08$ & $71.66$ & $30.48$ & $27.48$ & $8.21$  \\
        \;\;\;\; + SFT Cold Start & $35.87$\textcolor{green!50!black}{$_{\uparrow 31\%}$} & $11.23$\textcolor{green!50!black}{$_{\uparrow 175\%}$} & $82.42$\textcolor{green!50!black}{$_{\uparrow 15\%}$} & $44.92$\textcolor{green!50!black}{$_{\uparrow 47\%}$} & $42.17$\textcolor{green!50!black}{$_{\uparrow 53\%}$} & $18.84$\textcolor{green!50!black}{$_{\uparrow 129\%}$}  \\
        \bottomrule
    \end{tabular}
    }
    \vspace{-0.6em}
\end{table}

\paragraph{Combination of SFT and RL for Stable Inner Loop Optimization. }
A straightforward idea involves training the LLM using an optimal optimizer to stabilize its performance. Subsequently, reinforcement learning can be employed to explore improved paths for inner-loop optimization. As evidenced in \Cref{tab:rl-after-sft-performance}, the RL model demonstrates substantial enhancements after supervised fine-tuning.

\begin{takeawaybox}{Takeaway for SFT v.s. RL}
\begin{enumerate}[leftmargin=5mm]
    \item[\ding{182}] SFT provides stable inner loop optimization by training on trajectories from \textit{oracle inner loop optimizer} compared with RL.
    \item[\ding{183}] Combining SFT with RL shows significant performance improvements by utilizing SFT for initializing inner-loop optimization and RL for further exploration.
\end{enumerate}
\end{takeawaybox}

\begin{figure}[t]
    \centering
    \includegraphics[width=.9\linewidth]{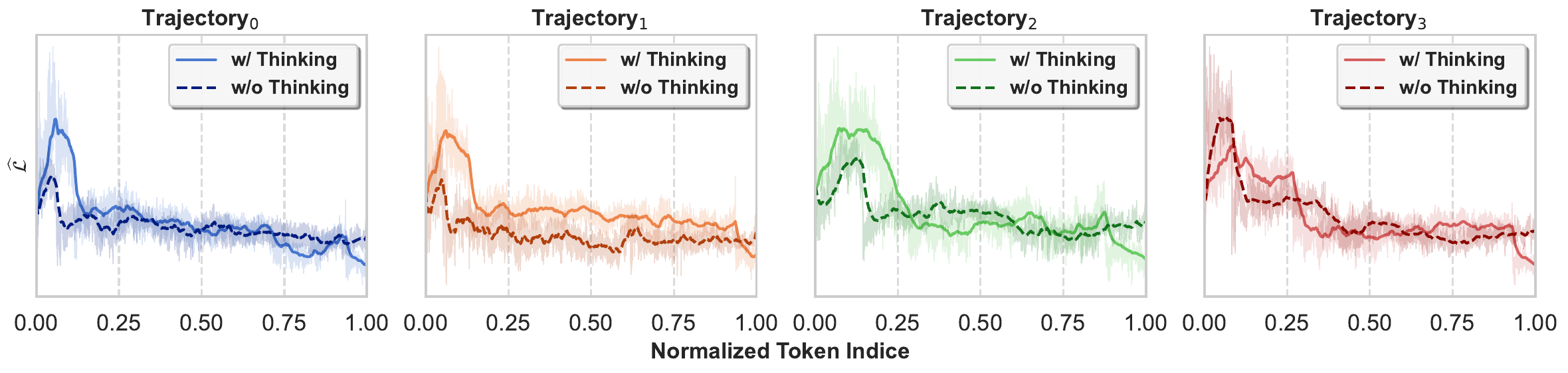}
    \caption{Illustration of QwQ's \textit{pseudo-gradient update} for both thinking and non-thinking modes and refer to \Cref{app:qwen3_update} for more examples. We visualize four pairs of correct reasoning trajectories for one question in AIME24. Compared with thinking trajectories, no-thinking trajectories converge more quickly, which also easily falls into local optimal points.}
    \label{fig:nothinking_gradient_decent}
    \vspace{-0.2em}
\end{figure}

\subsection{Inner Loop Optimization Steps \textit{v.s.} Reasoning Trajectory Tokens} \label{sec:length-of-rt}

In \name, each token in a single reasoning trajectory corresponds to an individual optimization step, and the length of the trajectory indicates the total number of update steps.
We examine these factors by integrating experimental results with related research studies.

\paragraph{Long Reasoning Trajectories Lead to Superior Performance. }
As shown in \Cref{tab:sft-vs-rl-performance}, models trained with longer reasoning trajectories consistently outperform those with shorter trajectories, aligning with meta-learning findings that extended inner-loop updates enhance performance. 
This observation is consistent with the superior performance of long CoT reasoning models, such as DeepSeek-R1 \citep{abs-2501-12948}, suggesting that longer trajectories increase inner-loop update steps, thereby improving LLM reasoning capabilities.

\paragraph{Different Reasoning Trajectory Tokens Represent Varying Roles of Update. }
We focus on discussing two intriguing findings in LLM reasoning. 
First, advanced reasoning in LLMs has been observed to have an \textbf{\textit{aha moment}}~\citep{abs-2501-12948}. 
This refers to specific \textit{reflection tokens} that prompt LLMs to devote additional time to thinking about questions. 
These tokens are also utilized to implement test-time scaling~\citep{abs-2501-19393,ma2025reasoning}.
Following the settings described in \Cref{sec:reasoning-trajectories-as-parameters-update}, we measure the relative changes in the $\mathcal{\widehat{L}}$ value before and after each token position. 
The results are presented in \Cref{fig:reflection_pattern}. 
We observe that reflection tokens such as ``Wait'' and ``Alternatively'' indicate a significant change in the objective. 
From an optimization perspective, we propose that these reflection tokens assist the model in escaping saddle points. 
As the model gradually approaches a stable state, these tokens provide a larger gradient, thereby expanding the exploration space to find a better parameter space.
In the following part, we explore the concept of reasoning efficiency, as discussed by various researchers \citep{abs-2503-21614, abs-2502-18080, zhang2025reasoning, ma2025reasoning, qwen3}. 
This concept involves optimizing the balance between decoding cost and performance utilizing specific segments, such as the end-of-thinking token delimiter.
We hypothesize that these termination delimiters enhance convergence at the optimization level, akin to the role of \textit{momentum} in optimization, facilitating rapid convergence of model parameters within a flatter region. 
However, this acceleration does not always lead to the optimal point.
Also refer to the settings described in \Cref{sec:reasoning-trajectories-as-parameters-update}, we append the end-of-thinking token delimiter \texttt{Therefore, after all this, I believe the answer is} following the thinking token delimiter \texttt{<think>}. 
\Cref{fig:nothinking_gradient_decent} demonstrates that trajectories using the end-of-thinking token delimiter achieve quicker convergence, confirming our hypothesis to some extent.
Since the QwQ model does not completely adapt to the no-thinking mode, we include additional experiments related to Qwen3 in \Cref{app:qwen3_update}. These experiments further substantiate our conclusions.

\begin{takeawaybox}{Takeaway for Reasoning Trajectory Tokens and Inner Loop Optimization Steps}
\begin{enumerate}[leftmargin=5mm]
    \item[\ding{182}] Long reasoning trajectories are analogous to performing additional steps of inner-loop optimization, which improves inner-loop optimization and further enhances outer-loop optimization, i.e., the reasning performance of LLMs.
    \item[\ding{183}] Different tokens serve distinct functions in the inner loop optimization process. For instance, tokens associated with reflection patterns promote the exploration of optimization paths, whereas special tokens regulating the length of reasoning in the recent Long-CoT LLMs facilitate fast-converging optimization steps.
\end{enumerate}
\end{takeawaybox}



\begin{figure}
    \begin{minipage}[t]{0.4\textwidth}
        \centering
        \includegraphics[width=\textwidth]{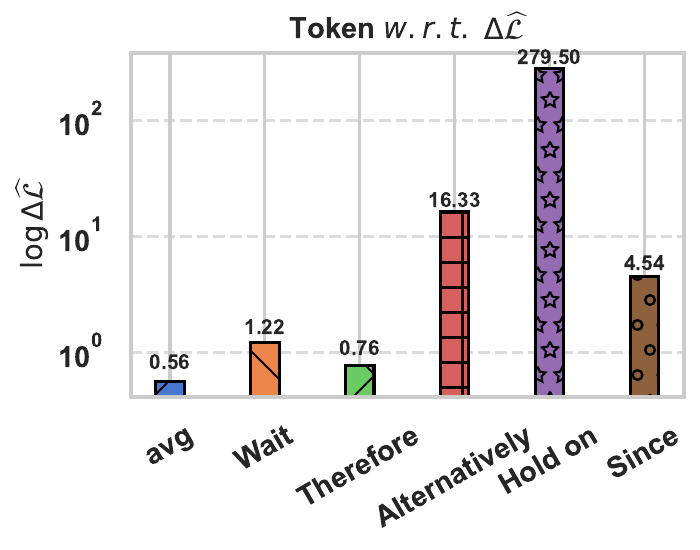}
        \captionof{figure}{Illustration of the influence of reflection tokens. Reflection tokens lead to sharper objective change.} \label{fig:reflection_pattern}
        \vspace{-0.2em}
    \end{minipage}
    \hfill
    \begin{minipage}[t]{0.57\textwidth}
        \centering
        \includegraphics[width=\textwidth]{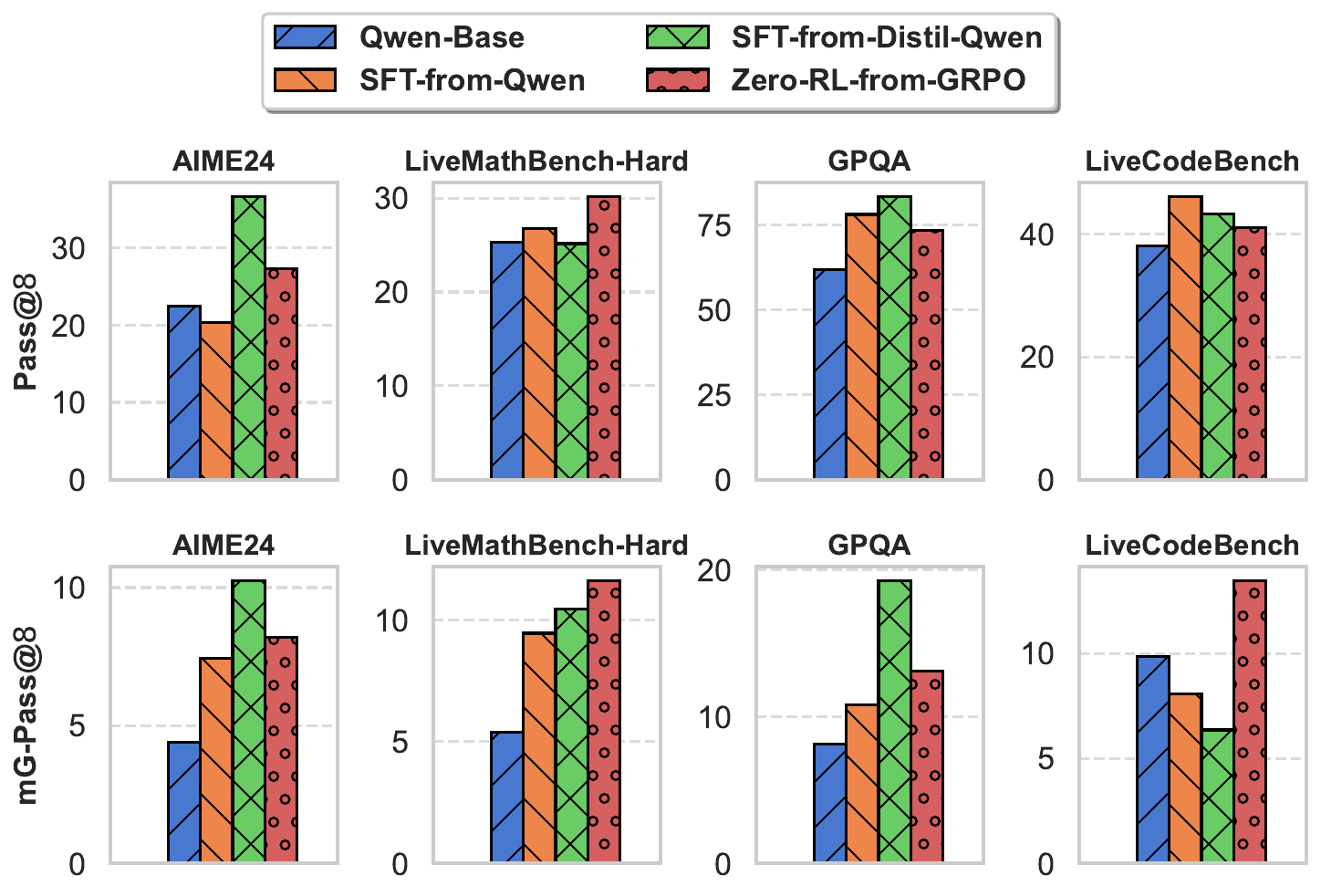}
        \captionof{figure}{Evaluation results of base, SFT and GRPO models on AIME24, LiveMathBench-Hard, GPQA-Diamond, and LiveCodeBench.} \label{fig:generalized_performance}
        \vspace{-0.2em}
    \end{minipage}
\end{figure}

\subsection{Task Generalization \textit{v.s.} Reasoning Generalization} \label{sec:ood-question}

Meta-learning models typically exhibit strong generalization across tasks with similar distributions, since the models learn generalized representations for these tasks. 
We investigate whether this applies to LLM reasoning, where each question is a distinct task but shares fundamental reasoning skills, suggesting a similar distribution.

We analyze generalization from two perspectives: within-domain generalization~(same reasoning domain) and cross-domain generalization~(different reasoning domains). 
Training data, sourced from Open Reasoner Zero (\Cref{sec:experiment-setup-training}), consist of mathematical problems from MATH, making AIME24 and LiveMathBench-Hard suitable for within-domain evaluation. 
Results in \Cref{fig:generalized_performance} show improved performance for both SFT and GRPO models on these benchmarks. 
For cross-domain generalization, we evaluated SFT and GRPO models on GPQA (scientific reasoning) and LiveCodeBench (code reasoning). As illustrated in the right section of \Cref{fig:generalized_performance}, all trained models outperformed the base model on both benchmarks.

Our findings align with existing research. 
For instance, studies have shown that large models trained on code datasets exhibit strong logical reasoning capabilities \citep{abs-2107-03374}. 
Additionally, research indicates that training on mathematical and code corpora mutually enhances performance \citep{WangRZLLSZSZ024,abs-2409-12186,abs-2401-14196}.

\begin{takeawaybox}{Takeaway for Reasoning Generalization}
Training LLMs using trajectories facilitates the learning of shared features across diverse reasoning problems. 
This process, akin to meta-learning, enables the parameters of LLMs to adapt efficiently by new trajectories and demonstrate generalization to out-of-distribution questions.
\end{takeawaybox}

\subsection{More Discussions about LLM Reasoning}
Additional discussions on current research advancements in LLM reasoning are provided in \Cref{app:disscussion-recent-reasoning-progress} within the framework of \name. 
Numerous improvements can be interpreted through the lens of meta-learning, highlighting the significant potential of comprehending LLM reasoning through \name.

\section{Advancing LLM Reasoning Drawing Inspiration from Meta-Learning}

In this section, we undertake several preliminary investigations informed by meta-learning to improve LLM reasoning, demonstrating the feasibility of enhancing LLM reasoning by integrating meta-learning studies.

\subsection{Manipulating Training Reasoning Trajectories per Question} \label{sec:number-of-rt}

Previous studies \citep{AgarwalYS21,ChenWLLZC20} highlight that the size of the support set is of paramount importance in improving performance, stability, and convergence in meta-learning. 
In \name, the support set is intrinsically connected with the number of reasoning trajectories trained per question, prompting the question: \textit{Can enhancements in support set size contribute to more effective training in LLM reasoning?}
In this preliminary study, we investigate the impact of increasing the number of training reasoning trajectories per question on LLM reasoning.

\begin{figure}[tb]
    \centering
    \includegraphics[width=.9\linewidth]{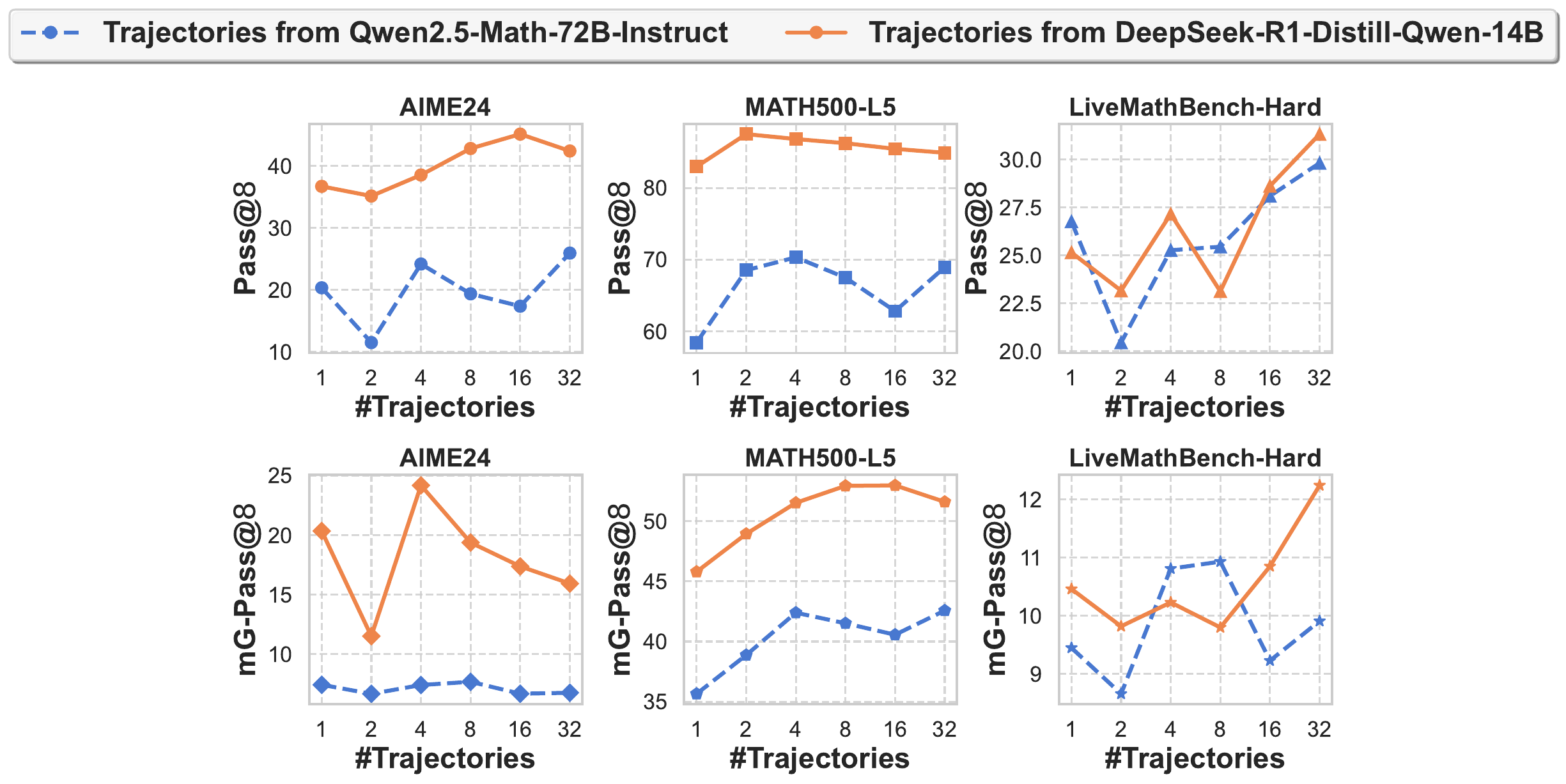}
    \caption{Evaluation results \textit{w.r.t.} different number of reasoning trajectories of SFT models on AIME24, MATH500-L5, and LiveMathBench-Hard.}
    \label{fig:performance_wrt_trajectories}
    \vspace{-0.2em}
\end{figure}

\paragraph{SFT. } 
We train Qwen2.5-7B-Base through SFT with $\{1, 2, 4, 8, 16, 32\}$ synthetic reasoning trajectories, ensuring equal training frequency per question. 
Evaluation results (\Cref{fig:performance_wrt_trajectories}) show that increasing the number of trajectories improves performance and reasoning stability across all benchmarks, suggesting that additional trajectories enhance supervised fine-tuning outcomes \citep{abs-2409-12122}.

\begin{figure}[t]
    \centering
    \includegraphics[width=.95\linewidth]{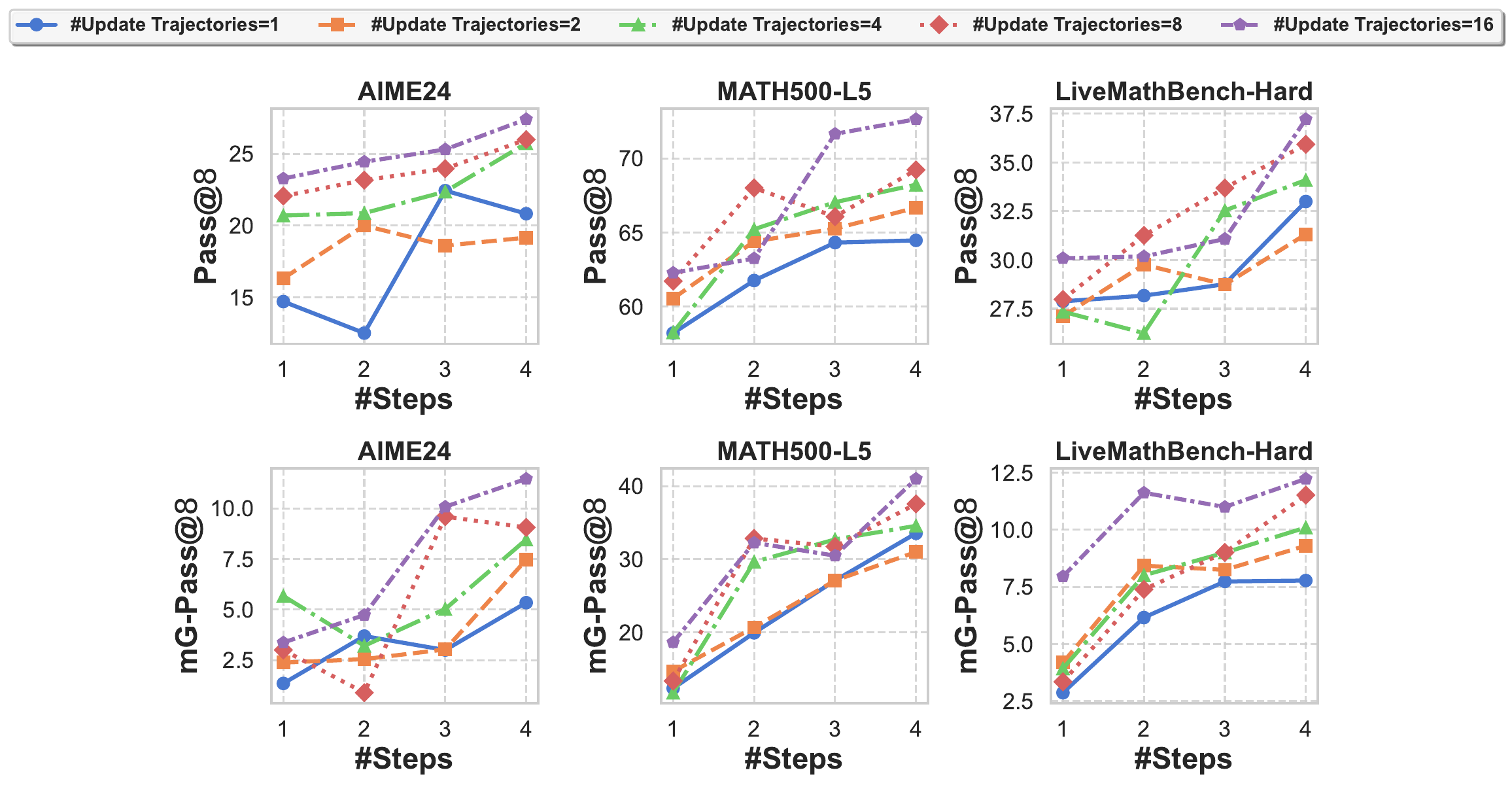}
    \caption{Evaluation results \textit{w.r.t.} different number of reasoning trajectories of GRPO models on AIME24, MATH500-L5, and LiveMathBench-Hard.} \label{fig:grpo_performance_wrt_trajectories}
    \vspace{-0.2em}
\end{figure}

\paragraph{GRPO. }
Regarding to GRPO, the support set size corresponds to the number of trajectories in the rollout group for each prompt~(question). 
To maintain stable advantage estimation in GRPO and ensure a fair comparison, we fix the rollout group size at 16 and calculated the advantage for each trajectory. 
During gradient updates, however, we randomly select $ n \in \{1, 2, 4, 8, 16\} $ trajectories to calculate the gradient.
Experimental results shown in \Cref{fig:grpo_performance_wrt_trajectories} demonstrate that: 
1) multiple trajectories for a single question significantly enhance model performance and stability; 
2) a larger number of trajectories accelerates convergence. 
These findings explain the superior performance and stability of GRPO-based~(or other similar RL algorithms) reasoning models~\citep{abs-2412-13147}, as the GRPO mechanism inherently optimizes for multiple trajectories per question compared with PPO~\citep{SchulmanWDRK17} and naive SFT.
It is noteworthy that some studies have attempted to enhance PPO through group sampling \citep{hu2025open,abs-2504-05118} and achieve competitive performance compared with original PPO.

\begin{observationbox}{Takeaway for Manipulating Training Reasoning Trajectories per Question}
    Training on multiple trajectories per question is akin to enlarging the support set to enhance inner-loop optimization, which consequently leads to better training performance.
\end{observationbox}

\subsection{Incentivizing Reasoning Efficiency by Optimization Lens}
\begin{wrapfigure}{L}{.3\textwidth}
    \centering
    \includegraphics[width=.3\textwidth]{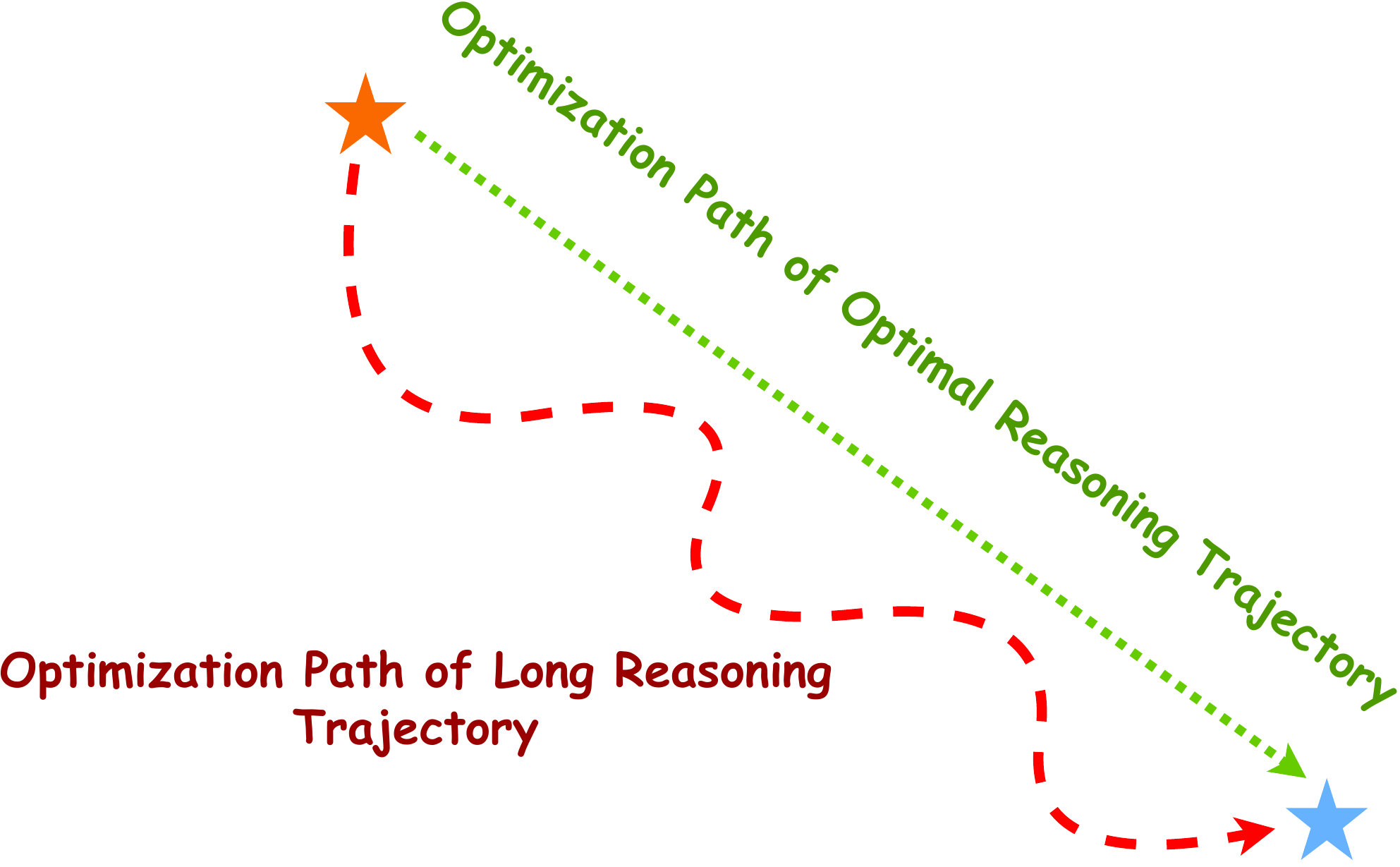}
    \caption{Given a long reasoning trajectory, there exists an optimal corresponding reasoning trajectory which leverage less tokens.} \label{fig:long2short}
    \vspace{-0.2em}
\end{wrapfigure}

Recent advanced reasoning models face limitations due to inefficient and excessively lengthy reasoning trajectories. 
Although several studies \citep{ma2025reasoning,qwen3} have attempted to minimize the number of decoding tokens to mitigate overhead, these approaches frequently lead to decreased reasoning performance, thereby presenting a fundamental question: \textit{Can we reduce the number of reasoning tokens without compromising reasoning performance?} 
As previously discussed, each reasoning trajectory corresponds to an inner-loop optimization trajectory, thus reframing the inquiry as follows: \textit{Can there be a s more effective inner loop optimization path?}
From an optimization perspective, as illustrated in \Cref{fig:long2short}, there exists such an optimal inner loop optimization path. 
In this section, we present a straightforward yet convincing experiment to validate the existence of this inner loop optimization path.

Specifically, we employ Qwen3-32B \citep{qwen3} to generate 16 reasoning trajectories for each question in AIME24, MATH500-L5, and LiveMathBench-Hard. 
These trajectories serve as foundational optimization paths, and our goal is to refine them to discover more optimal paths. 
We propose an heuristic method which condense the reasoning trajectories by using a LLM to summarize the original trajectories into shorter variants, which are then used to prompt Qwen3-32B for answer generation. 
The summarizations are generated by Qwen2.5-32B-Instruct, deliberately excluding answers from the summarized reasoning trajectories. 
We performed four summary generators to reduce the impact of randomness.
\Cref{fig:long2short_exp} displays the experimental results. Notably, we observe that Qwen3-32B's performance with summarized reasoning trajectories is comparable to its performance in thinking mode especially for the Pass@$16$ metric, while significantly reducing the number of tokens in the reasoning trajectories. 
Moreover, Qwen3-32B's performance using summarized reasoning trajectories surpasses that in no-thinking mode, even that tThe latter has more tokens.

\begin{figure}
    \centering
    \includegraphics[width=.85\textwidth]{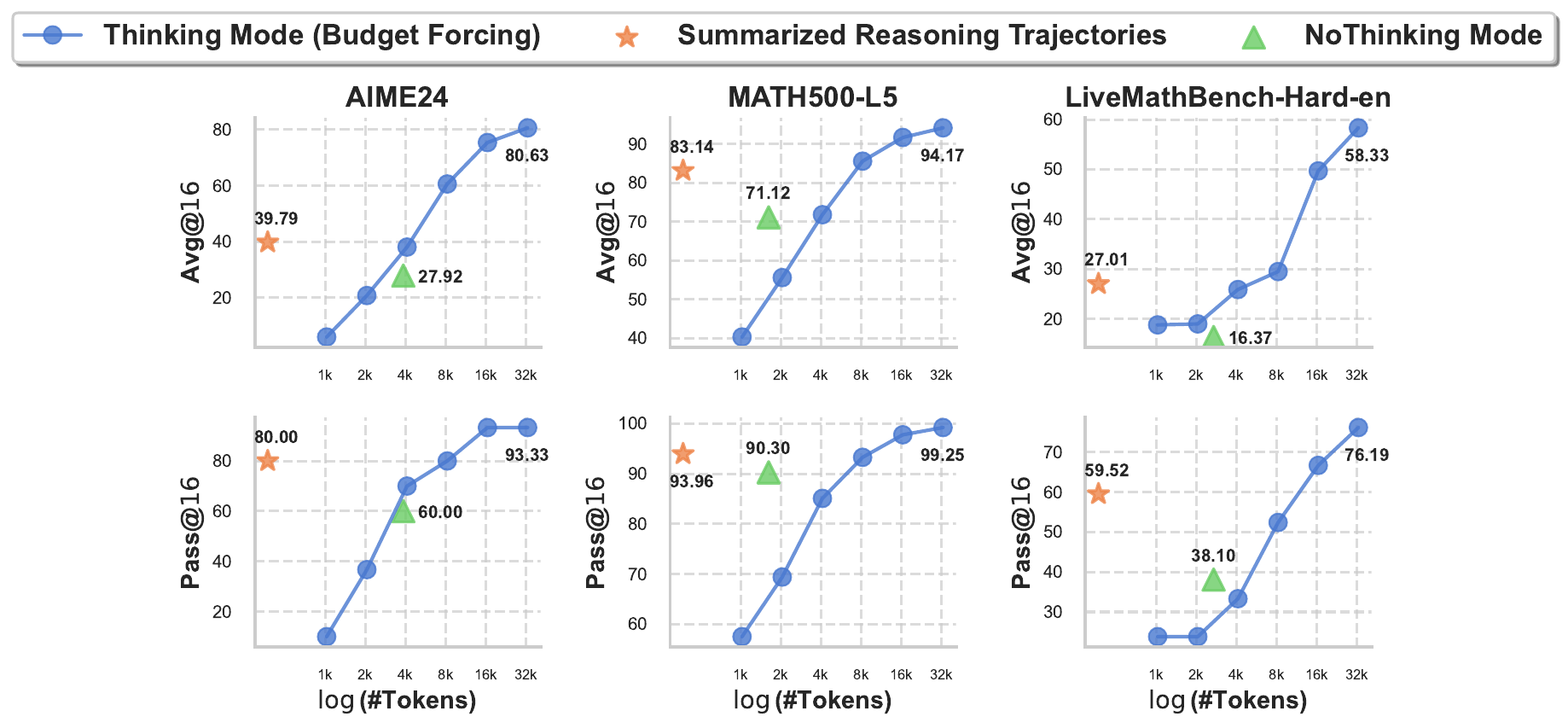}
    \caption{Experimental results of Qwen3-32B with 1) thinking mode, which includes a budget-forcing curve achieved by truncating decoding with various thinking tokens and prompting the model to generate the answer by appending \texttt{$\backslash$n$\backslash$n**Final Answer**$\backslash$n$\backslash$boxed\{}, 2) summarized trajectory, and 3) nothinking mode on AIME24, MATH500-L5, and LiveMathBench-Hard-en.}
    \label{fig:long2short_exp}
    \vspace{-0.2em}
\end{figure}

Our experiments demonstrate that trained long-CoT LLMs have the potential to achieve optimal reasoning trajectories that require fewer tokens while maintaining comparable reasoning performance. 
We approximate these trajectories using a straightforward method, leaving the exploration of more advanced approaches for future work.

\begin{observationbox}{Takeaway for Incentivizing Reasoning Efficiency by Optimization Lens}
    As an optimization, long reasoning trajectories have corresponding optimal trajectories that maintain reasoning performance while using fewer tokens. Find a method to reproduce these optimal trajectories during the decoding stage can enhance reasoning efficiency.
\end{observationbox}

\subsection{Future Directions}

Building on previous discussions and experiments, we demonstrate the feasibility of advancing LLM reasoning by drawing inspiration from meta-learning research. In this section, we outline several potential directions for future work, including:
\begin{itemize}[leftmargin=5mm]
    \item \textbf{Further Understanding Reasoning Trajectories}
    \begin{enumerate}
        \item[\ding{182}] Unlike typical meta-learning frameworks with predefined support sets, the reasoning trajectories in LLMs are self-generated. This implies that LLMs inherently learn gradient update steps without needing explicit support sets for fine-tuning. \textit{Investigating how LLMs learn to form effective reasoning trajectories, namely, gradient update steps}, presents an intriguing challenge.
        \item[\ding{183}] Tokens contribute differently to the modification of model parameters. \textit{What accounts for this disparity among tokens? Is it connected to their semantic properties, and if so, in what manner}?
        \item[\ding{184}] Trajectory-aided reasoning in large language models (LLMs) demonstrates comparable generalization abilities across various tasks. \textit{What aspects of the learning process contribute to this generalization ability, and which meta-features are developed through the optimization of reasoning trajectories?}
    \end{enumerate}
    \item \textbf{Towards Enhancing LLM Reasoning}
    \begin{enumerate}
        \item[\ding{182}] \textit{Improved Reasoning Trajectory Selection Strategy in LLM Training. } During both supervised fine-tuning and reinforcement learning, reasoning trajectories usually remain constant. Could implementing an adaptive sampling mechanism, similar to those utilized in meta-learning~\citep{YaoWWZMLF21,LiuWSFZH20}, enhance training efficacy?
        \item[\ding{183}] \textit{Enhancing Reasoning Efficiency Through an Optimization Perspective. } Given that each token offers a unique contribution to optimization, is there a strategy to discern these contributions to filter out superfluous tokens, thereby improving reasoning efficiency?
        \item[\ding{184}] \textit{Task~(Question) Ratio to Enhance Generalization Across Different Domains. } Insights from related studies, such as Collins et al.~\citep{CollinsMS20} and Wang et al.~\citep{WangGSQ22}, in meta-learning, suggest methods to bolster the reasoning capability of LLMs, enabling them to generalize across domains—for instance, training on mathematical data and inferring insights from coding data.
    \end{enumerate}
\end{itemize}

\section{Related Work}

\paragraph{LLM Reasoning. }
The reasoning capabilities of large language models~(LLMs) have progressively advanced through the development of several key technologies, which have substantially enhanced their performance on complex tasks. 
In-Context Learning~\citep{BrownMRSKDNSSAA20,RubinHB22,MinLHALHZ22,Dong0DZMLXX0C0S24,abs-2405-00200} enables models to perform tasks by interpreting examples provided in prompts without requiring additional training. 
However, this approach relies heavily on the model's pre-trained knowledge and careful prompt design, limiting its effectiveness for complex reasoning tasks~\citep{MinLZH22}. 
The introduction of Chain of Thought~(CoT)~\citep{Wei0SBIXCLZ22,YaoYZS00N23,BestaBKGPGGLNNH24} prompting has significantly improved LLM performance in areas such as mathematical reasoning, commonsense reasoning, and symbolic reasoning by guiding the models to produce intermediate reasoning steps. 
Supervised Fine-Tuning~(SFT) further refines the reasoning capabilities of LLMs by training them with labeled datasets tailored to specific tasks~\citep{numina_math_7b,openthoughts,sky_t1_2025,abs-2502-03387}. 
Reinforcement Learning~(RL), through the use of reward mechanisms, has become a critical approach to optimizing model behavior and enhancing reasoning abilities. 
Recently, Long-Chain of Thought~(Long-CoT) models have emerged as a notable trend in reasoning research, generating detailed reasoning steps to better address complex tasks~\citep{openaiO1,abs-2501-12948,abs-2501-12599,qwq,gemini25}.

\paragraph{Understanding LLMs. }
The remarkable success of LLMs has spurred extensive research into their capabilities.
Early studies \citep{YunBRRK20,YunCBRRK20} explored function approximation, demonstrating that sufficiently large transformers \citep{VaswaniSPUJGKP17} can universally approximate any continuous sequence-to-sequence mapping on a compact domain. 
Subsequent research investigated the computational power of transformers \citep{DehghaniGVUK19,BhattamishraAG20,YaoPPN20,HewittHGLM20,WeissGY21,MerrillSS22,0001CP23,GiannouRS0LP23,LiuAGKZ23}. 
For example, Feng et al. \citep{FengZGY0W23} validated the necessity of chain-of-thought (CoT) prompting for solving complex problems using circuit complexity theory. 
Other works \citep{GatmirySRJK24,abs-2502-21212} demonstrate that transformers can learn to implement learning algorithms, such as gradient descent, within trajectories. 
Several studies \citep{XieRL022,AkyurekSA0Z23,dai2023can,BaiCWXM23,abs-2209-11895,GatmirySRJK24} have focused on understanding in-context learning (ICL) \citep{BrownMRSKDNSSAA20,Dong0DZMLXX0C0S24}, examining the role of demonstration examples. 

\paragraph{Meta-Learning. }
Meta-learning, commonly referred to as ``learning to learn'', aims to enable models to enhance their learning strategies by leveraging prior experience across multiple tasks. 
Early research in this area~\citep{BengioOptimality92,ThrunP98} explored methods for acquiring learning rules applicable to new tasks, with a particular emphasis on lifelong learning. 
These foundational efforts established the basis for creating more adaptable and flexible learning algorithms, paving the way for subsequent advancements. 
Recent meta-learning approaches can generally be categorized into three groups: 1) metric-based methods, which focus on learning a feature space to efficiently compare samples~\citep{SnellSZ17,ChenZWC20,TangLPT20,ZhangWSFC23}; 2) model-based methods, which utilize memory mechanisms or other structures to store and retrieve task-specific information~\citep{WestonCB14,SukhbaatarSWF15,SantoroBBWL16}; and 3) optimization-based methods, which refine the learning process to facilitate rapid adaptation~\citep{FinnAL17,RajeswaranFKL19,YeLYGXZ21}.

\paragraph{Additional Discussion. }
In this part, we aim to elucidate and discuss our work in comparison with several related studies.
Dai et al. \citep{dai2023can} interpret in-context learning (ICL) as LLMs generating meta-gradients from demonstration examples, which are applied to the base GPT model to construct an ICL system. 
In this work, each demonstration example serves as one data sample to update the parameters of LLMs.
In contrast, our research emphasizes trajectory-aided reasoning, viewing each token as an update step and drawing extensive connections to supervised fine-tuning and reinforcement learning, rather than focusing on demonstration examples. 
Additionally, our approach incorporates more general training techniques with explicit parameter optimization, whereas ICL is constrained by limited demonstration examples, which poses certain limitations. 
Another research avenue explored by studies such as Gatmiry et al. \citep{GatmirySRJK24} and others \citep{abs-2502-21212} shows that transformers can learn to implement learning algorithms, such as gradient descent, within a chain of thought. 
However, these studies primarily investigate whether transformers can describe learning algorithms in natural language to solve practical numerical optimization problems, while our work delves into the internal parameter updates of the transformer model, applicable to a broader range of problems.
\section{Conclusion}

This paper introduces a novel perspective on LLM reasoning by integrating it into the meta-learning framework. 
Through both theoretical analysis and empirical validation, we illustrate that reasoning trajectories can be effectively conceptualized as pseudo-gradient updates, enabling a deeper comprehension of how LLMs perform complex reasoning tasks. 
Extensive experiments reveal the correlation between meta-learning and LLM reasoning, suggesting potential avenues for advancing LLM reasoning through meta-learning principles. 
This work not only enhances the understanding of LLM reasoning but also offers practical insights for improving these models using established meta-learning techniques.

\clearpage
\bibliographystyle{plain}
\bibliography{refs}

\begin{thebibliography}{100}

\bibitem{AgarwalYS21}
Mayank Agarwal, Mikhail Yurochkin, and Yuekai Sun.
\newblock On sensitivity of meta-learning to support data.
\newblock In {\em NeurIPS}, pages 20447--20460, 2021.

\bibitem{AhmadianCGFKPUH24}
Arash Ahmadian, Chris Cremer, Matthias Gall{\'{e}}, Marzieh Fadaee, Julia Kreutzer, Olivier Pietquin, Ahmet {\"{U}}st{\"{u}}n, and Sara Hooker.
\newblock Back to basics: Revisiting reinforce-style optimization for learning from human feedback in llms.
\newblock In {\em {ACL}}, pages 12248--12267. Association for Computational Linguistics, 2024.

\bibitem{AkyurekSA0Z23}
Ekin Aky{\"{u}}rek, Dale Schuurmans, Jacob Andreas, Tengyu Ma, and Denny Zhou.
\newblock What learning algorithm is in-context learning? investigations with linear models.
\newblock In {\em {ICLR}}. OpenReview.net, 2023.

\bibitem{AndrychowiczDCH16}
Marcin Andrychowicz, Misha Denil, Sergio~Gomez Colmenarejo, Matthew~W. Hoffman, David Pfau, Tom Schaul, and Nando de~Freitas.
\newblock Learning to learn by gradient descent by gradient descent.
\newblock In {\em {NIPS}}, pages 3981--3989, 2016.

\bibitem{BaiCWXM23}
Yu~Bai, Fan Chen, Huan Wang, Caiming Xiong, and Song Mei.
\newblock Transformers as statisticians: Provable in-context learning with in-context algorithm selection.
\newblock In {\em NeurIPS}, 2023.

\bibitem{numina_math_7b}
Edward Beeching, Shengyi~Costa Huang, Albert Jiang, Jia Li, Benjamin Lipkin, Zihan Qina, Kashif Rasul, Ziju Shen, Roman Soletskyi, and Lewis Tunstall.
\newblock Numinamath 72b cot.
\newblock \url{https://huggingface.co/AI-MO/NuminaMath-72B-CoT}, 2024.

\bibitem{BengioOptimality92}
Samy Bengio, Yoshua Bengio, Jocelyn Cloutier, and Jan Gecsei.
\newblock On the optimization of a synaptic learning rule.
\newblock In {\em Optimality in Artificial and Biological Neural Networks}, pages 6--8, 1992.

\bibitem{abs-2505-00949}
Akhiad Bercovich, Itay Levy, Izik Golan, Mohammad Dabbah, Ran El-Yaniv, Omri Puny, Ido Galil, Zach Moshe, Tomer Ronen, Najeeb Nabwani, Ido Shahaf, Oren Tropp, Ehud Karpas, Ran Zilberstein, Jiaqi Zeng, Soumye Singhal, Alexander Bukharin, Yian Zhang, Tugrul Konuk, Gerald Shen, Ameya~Sunil Mahabaleshwarkar, Bilal Kartal, Yoshi Suhara, Olivier Delalleau, Zijia Chen, Zhilin Wang, David Mosallanezhad, Adi Renduchintala, Haifeng Qian, Dima Rekesh, Fei Jia, Somshubra Majumdar, Vahid Noroozi, Wasi~Uddin Ahmad, Sean Narenthiran, Aleksander Ficek, Mehrzad Samadi, Jocelyn Huang, Siddhartha Jain, Igor Gitman, Ivan Moshkov, Wei Du, Shubham Toshniwal, George Armstrong, Branislav Kisacanin, Matvei Novikov, Daria Gitman, Evelina Bakhturina, Jane~Polak Scowcroft, John Kamalu, Dan Su, Kezhi Kong, Markus Kliegl, Rabeeh Karimi, Ying Lin, Sanjeev Satheesh, Jupinder Parmar, Pritam Gundecha, Brandon Norick, Joseph Jennings, Shrimai Prabhumoye, Syeda~Nahida Akter, Mostofa Patwary, Abhinav Khattar, Deepak Narayanan, Roger Waleffe,
  Jimmy Zhang, Bor-Yiing Su, Guyue Huang, Terry Kong, Parth Chadha, Sahil Jain, Christine Harvey, Elad Segal, Jining Huang, Sergey Kashirsky, Robert McQueen, Izzy Putterman, George Lam, Arun Venkatesan, Sherry Wu, Vinh Nguyen, Manoj Kilaru, Andrew Wang, Anna Warno, Abhilash Somasamudramath, Sandip Bhaskar, Maka Dong, Nave Assaf, Shahar Mor, Omer~Ullman Argov, Scot Junkin, Oleksandr Romanenko, Pedro Larroy, Monika Katariya, Marco Rovinelli, Viji Balas, Nicholas Edelman, Anahita Bhiwandiwalla, Muthu Subramaniam, Smita Ithape, Karthik Ramamoorthy, Yuting Wu, Suguna~Varshini Velury, Omri Almog, Joyjit Daw, Denys Fridman, Erick Galinkin, Michael Evans, Katherine Luna, Leon Derczynski, Nikki Pope, Eileen Long, Seth Schneider, Guillermo Siman, Tomasz Grzegorzek, Pablo Ribalta, Monika Katariya, Joey Conway, Trisha Saar, Ann Guan, Krzysztof Pawelec, Shyamala Prayaga, Oleksii Kuchaiev, Boris Ginsburg, Oluwatobi Olabiyi, Kari Briski, Jonathan Cohen, Bryan Catanzaro, Jonah Alben, Yonatan Geifman, Eric Chung, and Chris
  Alexiuk.
\newblock {Llama-Nemotron}: Efficient reasoning models.
\newblock {\em CoRR}, abs/2505.00949, 2025.

\bibitem{abs-2405-00200}
Amanda Bertsch, Maor Ivgi, Uri Alon, Jonathan Berant, Matthew~R. Gormley, and Graham Neubig.
\newblock In-context learning with long-context models: An in-depth exploration.
\newblock {\em CoRR}, abs/2405.00200, 2024.

\bibitem{BestaBKGPGGLNNH24}
Maciej Besta, Nils Blach, Ales Kubicek, Robert Gerstenberger, Michal Podstawski, Lukas Gianinazzi, Joanna Gajda, Tomasz Lehmann, Hubert Niewiadomski, Piotr Nyczyk, and Torsten Hoefler.
\newblock Graph of thoughts: Solving elaborate problems with large language models.
\newblock In {\em {AAAI}}, pages 17682--17690. {AAAI} Press, 2024.

\bibitem{BhattamishraAG20}
Satwik Bhattamishra, Kabir Ahuja, and Navin Goyal.
\newblock On the ability and limitations of transformers to recognize formal languages.
\newblock In {\em {EMNLP}}, pages 7096--7116. Association for Computational Linguistics, 2020.

\bibitem{bjerhammar1951application}
Arne Bjerhammar.
\newblock Application of calculus of matrices to method of least squares: with special reference to geodetic calculations.
\newblock {\em Trans. Roy. Inst. Tech. Stockholm.}, 1951.

\bibitem{BrownMRSKDNSSAA20}
Tom~B. Brown, Benjamin Mann, Nick Ryder, Melanie Subbiah, Jared Kaplan, Prafulla Dhariwal, Arvind Neelakantan, Pranav Shyam, Girish Sastry, Amanda Askell, Sandhini Agarwal, Ariel Herbert{-}Voss, Gretchen Krueger, Tom Henighan, Rewon Child, Aditya Ramesh, Daniel~M. Ziegler, Jeffrey Wu, Clemens Winter, Christopher Hesse, Mark Chen, Eric Sigler, Mateusz Litwin, Scott Gray, Benjamin Chess, Jack Clark, Christopher Berner, Sam McCandlish, Alec Radford, Ilya Sutskever, and Dario Amodei.
\newblock Language models are few-shot learners.
\newblock In {\em NeurIPS}, 2020.

\bibitem{chen2025sft}
Hardy Chen, Haoqin Tu, Fali Wang, Hui Liu, Xianfeng Tang, Xinya Du, Yuyin Zhou, and Cihang Xie.
\newblock Sft or rl? an early investigation into training r1-like reasoning large vision-language models.
\newblock {\em CoRR}, abs/2504.11468, 2025.

\bibitem{ChenWLLZC20}
Jiaxin Chen, Xiao{-}Ming Wu, Yanke Li, Qimai Li, Li{-}Ming Zhan, and Fu{-}Lai Chung.
\newblock A closer look at the training strategy for modern meta-learning.
\newblock In {\em NeurIPS}, 2020.

\bibitem{ChenZWC20}
Jiaxin Chen, Li{-}Ming Zhan, Xiao{-}Ming Wu, and Fu{-}Lai Chung.
\newblock Variational metric scaling for metric-based meta-learning.
\newblock In {\em {AAAI}}, pages 3478--3485. {AAAI} Press, 2020.

\bibitem{abs-2107-03374}
Mark Chen, Jerry Tworek, Heewoo Jun, Qiming Yuan, Henrique~Pond{\'{e}} de~Oliveira~Pinto, Jared Kaplan, Harri Edwards, Yuri Burda, Nicholas Joseph, Greg Brockman, Alex Ray, Raul Puri, Gretchen Krueger, Michael Petrov, Heidy Khlaaf, Girish Sastry, Pamela Mishkin, Brooke Chan, Scott Gray, Nick Ryder, Mikhail Pavlov, Alethea Power, Lukasz Kaiser, Mohammad Bavarian, Clemens Winter, Philippe Tillet, Felipe~Petroski Such, Dave Cummings, Matthias Plappert, Fotios Chantzis, Elizabeth Barnes, Ariel Herbert{-}Voss, William~Hebgen Guss, Alex Nichol, Alex Paino, Nikolas Tezak, Jie Tang, Igor Babuschkin, Suchir Balaji, Shantanu Jain, William Saunders, Christopher Hesse, Andrew~N. Carr, Jan Leike, Joshua Achiam, Vedant Misra, Evan Morikawa, Alec Radford, Matthew Knight, Miles Brundage, Mira Murati, Katie Mayer, Peter Welinder, Bob McGrew, Dario Amodei, Sam McCandlish, Ilya Sutskever, and Wojciech Zaremba.
\newblock Evaluating large language models trained on code.
\newblock {\em CoRR}, abs/2107.03374, 2021.

\bibitem{abs-2503-09567}
Qiguang Chen, Libo Qin, Jinhao Liu, Dengyun Peng, Jiannan Guan, Peng Wang, Mengkang Hu, Yuhang Zhou, Te~Gao, and Wanxiang Che.
\newblock Towards reasoning era: {A} survey of long chain-of-thought for reasoning large language models.
\newblock {\em CoRR}, abs/2503.09567, 2025.

\bibitem{0001CP23}
David Chiang, Peter Cholak, and Anand Pillay.
\newblock Tighter bounds on the expressivity of transformer encoders.
\newblock In {\em {ICML}}, volume 202 of {\em Proceedings of Machine Learning Research}, pages 5544--5562. {PMLR}, 2023.

\bibitem{abs-2501-17161}
Tianzhe Chu, Yuexiang Zhai, Jihan Yang, Shengbang Tong, Saining Xie, Dale Schuurmans, Quoc~V. Le, Sergey Levine, and Yi~Ma.
\newblock {SFT} memorizes, {RL} generalizes: {A} comparative study of foundation model post-training.
\newblock {\em CoRR}, abs/2501.17161, 2025.

\bibitem{chu2025gpg}
Xiangxiang Chu, Hailang Huang, Xiao Zhang, Fei Wei, and Yong Wang.
\newblock Gpg: A simple and strong reinforcement learning baseline for model reasoning.
\newblock {\em CoRR}, abs/2504.02546, 2025.

\bibitem{CollinsMS20}
Liam Collins, Aryan Mokhtari, and Sanjay Shakkottai.
\newblock Task-robust model-agnostic meta-learning.
\newblock In {\em NeurIPS}, 2020.

\bibitem{CollinsMS22}
Liam Collins, Aryan Mokhtari, and Sanjay Shakkottai.
\newblock How does the task landscape affect {MAML} performance?
\newblock In {\em CoLLAs}, volume 199 of {\em Proceedings of Machine Learning Research}, pages 23--59. {PMLR}, 2022.

\bibitem{cybenko1989approximations}
George Cybenko.
\newblock Approximations by superpositions of a sigmoidal function.
\newblock {\em MCSS}, 2:183--192, 1989.

\bibitem{dai2023can}
Damai Dai, Yutao Sun, Li~Dong, Yaru Hao, Shuming Ma, Zhifang Sui, and Furu Wei.
\newblock Why can gpt learn in-context? language models secretly perform gradient descent as meta-optimizers.
\newblock In {\em {ACL} (Findings)}, pages 4005--4019. Association for Computational Linguistics, 2023.

\bibitem{gemini25}
Google Deepmind.
\newblock Gemini 2.5: Our most intelligent ai model.
\newblock \url{https://blog.google/technology/google-deepmind/gemini-model-thinking-updates-march-2025/\#gemini-2-5-thinking/}, 2025.
\newblock Accessed: 2025-03-26.

\bibitem{abs-2501-12948}
DeepSeek{-}AI, Daya Guo, Dejian Yang, Haowei Zhang, Junxiao Song, Ruoyu Zhang, Runxin Xu, Qihao Zhu, Shirong Ma, Peiyi Wang, Xiao Bi, Xiaokang Zhang, Xingkai Yu, Yu~Wu, Z.~F. Wu, Zhibin Gou, Zhihong Shao, Zhuoshu Li, Ziyi Gao, Aixin Liu, Bing Xue, Bingxuan Wang, Bochao Wu, Bei Feng, Chengda Lu, Chenggang Zhao, Chengqi Deng, Chenyu Zhang, Chong Ruan, Damai Dai, Deli Chen, Dongjie Ji, Erhang Li, Fangyun Lin, Fucong Dai, Fuli Luo, Guangbo Hao, Guanting Chen, Guowei Li, H.~Zhang, Han Bao, Hanwei Xu, Haocheng Wang, Honghui Ding, Huajian Xin, Huazuo Gao, Hui Qu, Hui Li, Jianzhong Guo, Jiashi Li, Jiawei Wang, Jingchang Chen, Jingyang Yuan, Junjie Qiu, Junlong Li, J.~L. Cai, Jiaqi Ni, Jian Liang, Jin Chen, Kai Dong, Kai Hu, Kaige Gao, Kang Guan, Kexin Huang, Kuai Yu, Lean Wang, Lecong Zhang, Liang Zhao, Litong Wang, Liyue Zhang, Lei Xu, Leyi Xia, Mingchuan Zhang, Minghua Zhang, Minghui Tang, Meng Li, Miaojun Wang, Mingming Li, Ning Tian, Panpan Huang, Peng Zhang, Qiancheng Wang, Qinyu Chen, Qiushi Du, Ruiqi Ge,
  Ruisong Zhang, Ruizhe Pan, Runji Wang, R.~J. Chen, R.~L. Jin, Ruyi Chen, Shanghao Lu, Shangyan Zhou, Shanhuang Chen, Shengfeng Ye, Shiyu Wang, Shuiping Yu, Shunfeng Zhou, Shuting Pan, and S.~S. Li.
\newblock Deepseek-r1: Incentivizing reasoning capability in llms via reinforcement learning.
\newblock {\em CoRR}, abs/2501.12948, 2025.

\bibitem{abs-2412-19437}
DeepSeek{-}AI, Aixin Liu, Bei Feng, Bing Xue, Bingxuan Wang, Bochao Wu, Chengda Lu, Chenggang Zhao, Chengqi Deng, Chenyu Zhang, Chong Ruan, Damai Dai, Daya Guo, Dejian Yang, Deli Chen, Dongjie Ji, Erhang Li, Fangyun Lin, Fucong Dai, Fuli Luo, Guangbo Hao, Guanting Chen, Guowei Li, H.~Zhang, Han Bao, Hanwei Xu, Haocheng Wang, Haowei Zhang, Honghui Ding, Huajian Xin, Huazuo Gao, Hui Li, Hui Qu, J.~L. Cai, Jian Liang, Jianzhong Guo, Jiaqi Ni, Jiashi Li, Jiawei Wang, Jin Chen, Jingchang Chen, Jingyang Yuan, Junjie Qiu, Junlong Li, Junxiao Song, Kai Dong, Kai Hu, Kaige Gao, Kang Guan, Kexin Huang, Kuai Yu, Lean Wang, Lecong Zhang, Lei Xu, Leyi Xia, Liang Zhao, Litong Wang, Liyue Zhang, Meng Li, Miaojun Wang, Mingchuan Zhang, Minghua Zhang, Minghui Tang, Mingming Li, Ning Tian, Panpan Huang, Peiyi Wang, Peng Zhang, Qiancheng Wang, Qihao Zhu, Qinyu Chen, Qiushi Du, R.~J. Chen, R.~L. Jin, Ruiqi Ge, Ruisong Zhang, Ruizhe Pan, Runji Wang, Runxin Xu, Ruoyu Zhang, Ruyi Chen, S.~S. Li, Shanghao Lu, Shangyan Zhou,
  Shanhuang Chen, Shaoqing Wu, Shengfeng Ye, Shengfeng Ye, Shirong Ma, Shiyu Wang, Shuang Zhou, Shuiping Yu, Shunfeng Zhou, Shuting Pan, T.~Wang, Tao Yun, Tian Pei, Tianyu Sun, W.~L. Xiao, and Wangding Zeng.
\newblock Deepseek-v3 technical report.
\newblock {\em CoRR}, abs/2412.19437, 2024.

\bibitem{DehghaniGVUK19}
Mostafa Dehghani, Stephan Gouws, Oriol Vinyals, Jakob Uszkoreit, and Lukasz Kaiser.
\newblock Universal transformers.
\newblock In {\em {ICLR}}. OpenReview.net, 2019.

\bibitem{DevlinCLT19}
Jacob Devlin, Ming{-}Wei Chang, Kenton Lee, and Kristina Toutanova.
\newblock {BERT:} pre-training of deep bidirectional transformers for language understanding.
\newblock In {\em {NAACL-HLT}}, pages 4171--4186. Association for Computational Linguistics, 2019.

\bibitem{Dong0DZMLXX0C0S24}
Qingxiu Dong, Lei Li, Damai Dai, Ce~Zheng, Jingyuan Ma, Rui Li, Heming Xia, Jingjing Xu, Zhiyong Wu, Baobao Chang, Xu~Sun, Lei Li, and Zhifang Sui.
\newblock A survey on in-context learning.
\newblock In {\em {EMNLP}}, pages 1107--1128. Association for Computational Linguistics, 2024.

\bibitem{abs-2407-21783}
Abhimanyu Dubey, Abhinav Jauhri, Abhinav Pandey, Abhishek Kadian, Ahmad Al{-}Dahle, Aiesha Letman, Akhil Mathur, Alan Schelten, Amy Yang, Angela Fan, Anirudh Goyal, Anthony Hartshorn, Aobo Yang, Archi Mitra, Archie Sravankumar, Artem Korenev, Arthur Hinsvark, Arun Rao, Aston Zhang, Aur{\'{e}}lien Rodriguez, Austen Gregerson, Ava Spataru, Baptiste Rozi{\`{e}}re, Bethany Biron, Binh Tang, Bobbie Chern, Charlotte Caucheteux, Chaya Nayak, Chloe Bi, Chris Marra, Chris McConnell, Christian Keller, Christophe Touret, Chunyang Wu, Corinne Wong, Cristian~Canton Ferrer, Cyrus Nikolaidis, Damien Allonsius, Daniel Song, Danielle Pintz, Danny Livshits, David Esiobu, Dhruv Choudhary, Dhruv Mahajan, Diego Garcia{-}Olano, Diego Perino, Dieuwke Hupkes, Egor Lakomkin, Ehab AlBadawy, Elina Lobanova, Emily Dinan, Eric~Michael Smith, Filip Radenovic, Frank Zhang, Gabriel Synnaeve, Gabrielle Lee, Georgia~Lewis Anderson, Graeme Nail, Gr{\'{e}}goire Mialon, Guan Pang, Guillem Cucurell, Hailey Nguyen, Hannah Korevaar, Hu~Xu, Hugo
  Touvron, Iliyan Zarov, Imanol~Arrieta Ibarra, Isabel~M. Kloumann, Ishan Misra, Ivan Evtimov, Jade Copet, Jaewon Lee, Jan Geffert, Jana Vranes, Jason Park, Jay Mahadeokar, Jeet Shah, Jelmer van~der Linde, Jennifer Billock, Jenny Hong, Jenya Lee, Jeremy Fu, Jianfeng Chi, Jianyu Huang, Jiawen Liu, Jie Wang, Jiecao Yu, Joanna Bitton, Joe Spisak, Jongsoo Park, Joseph Rocca, Joshua Johnstun, Joshua Saxe, Junteng Jia, Kalyan~Vasuden Alwala, Kartikeya Upasani, Kate Plawiak, Ke~Li, Kenneth Heafield, Kevin Stone, and et~al.
\newblock The llama 3 herd of models.
\newblock {\em CoRR}, abs/2407.21783, 2024.

\bibitem{FengZGY0W23}
Guhao Feng, Bohang Zhang, Yuntian Gu, Haotian Ye, Di~He, and Liwei Wang.
\newblock Towards revealing the mystery behind chain of thought: {A} theoretical perspective.
\newblock In {\em NeurIPS}, 2023.

\bibitem{FinnAL17}
Chelsea Finn, Pieter Abbeel, and Sergey Levine.
\newblock Model-agnostic meta-learning for fast adaptation of deep networks.
\newblock In {\em {ICML}}, volume~70 of {\em Proceedings of Machine Learning Research}, pages 1126--1135. {PMLR}, 2017.

\bibitem{FuCJS24}
Deqing Fu, Tian{-}Qi Chen, Robin Jia, and Vatsal Sharan.
\newblock Transformers learn to achieve second-order convergence rates for in-context linear regression.
\newblock In {\em NeurIPS}, 2024.

\bibitem{GatmirySRJK24}
Khashayar Gatmiry, Nikunj Saunshi, Sashank~J. Reddi, Stefanie Jegelka, and Sanjiv Kumar.
\newblock Can looped transformers learn to implement multi-step gradient descent for in-context learning?
\newblock In {\em {ICML}}. OpenReview.net, 2024.

\bibitem{GiannouRS0LP23}
Angeliki Giannou, Shashank Rajput, Jy{-}yong Sohn, Kangwook Lee, Jason~D. Lee, and Dimitris Papailiopoulos.
\newblock Looped transformers as programmable computers.
\newblock In {\em {ICML}}, volume 202 of {\em Proceedings of Machine Learning Research}, pages 11398--11442. {PMLR}, 2023.

\bibitem{abs-2401-14196}
Daya Guo, Qihao Zhu, Dejian Yang, Zhenda Xie, Kai Dong, Wentao Zhang, Guanting Chen, Xiao Bi, Y.~Wu, Y.~K. Li, Fuli Luo, Yingfei Xiong, and Wenfeng Liang.
\newblock Deepseek-coder: When the large language model meets programming - the rise of code intelligence.
\newblock {\em CoRR}, abs/2401.14196, 2024.

\bibitem{HendrycksBKABTS21}
Dan Hendrycks, Collin Burns, Saurav Kadavath, Akul Arora, Steven Basart, Eric Tang, Dawn Song, and Jacob Steinhardt.
\newblock Measuring mathematical problem solving with the {MATH} dataset.
\newblock In {\em NeurIPS Datasets and Benchmarks}, 2021.

\bibitem{HewittHGLM20}
John Hewitt, Michael Hahn, Surya Ganguli, Percy Liang, and Christopher~D. Manning.
\newblock Rnns can generate bounded hierarchical languages with optimal memory.
\newblock In {\em {EMNLP}}, pages 1978--2010. Association for Computational Linguistics, 2020.

\bibitem{Hornik91}
Kurt Hornik.
\newblock Approximation capabilities of multilayer feedforward networks.
\newblock {\em Neural Networks}, 4(2):251--257, 1991.

\bibitem{HospedalesAMS22}
Timothy~M. Hospedales, Antreas Antoniou, Paul Micaelli, and Amos~J. Storkey.
\newblock Meta-learning in neural networks: {A} survey.
\newblock {\em {IEEE} Trans. Pattern Anal. Mach. Intell.}, 44(9):5149--5169, 2022.

\bibitem{RuderH18}
Jeremy Howard and Sebastian Ruder.
\newblock Universal language model fine-tuning for text classification.
\newblock In {\em {ACL}}, pages 328--339. Association for Computational Linguistics, 2018.

\bibitem{hu2025open}
Jingcheng Hu, Yinmin Zhang, Qi~Han, Daxin Jiang, Xiangyu Zhang, and Heung-Yeung Shum.
\newblock Open-reasoner-zero: An open source approach to scaling up reinforcement learning on the base model.
\newblock {\em CoRR}, abs/2503.24290, 2025.

\bibitem{abs-2502-21212}
Jianhao Huang, Zixuan Wang, and Jason~D. Lee.
\newblock Transformers learn to implement multi-step gradient descent with chain of thought.
\newblock {\em CoRR}, abs/2502.21212, 2025.

\bibitem{abs-2409-12186}
Binyuan Hui, Jian Yang, Zeyu Cui, Jiaxi Yang, Dayiheng Liu, Lei Zhang, Tianyu Liu, Jiajun Zhang, Bowen Yu, Kai Dang, An~Yang, Rui Men, Fei Huang, Xingzhang Ren, Xuancheng Ren, Jingren Zhou, and Junyang Lin.
\newblock Qwen2.5-coder technical report.
\newblock {\em CoRR}, abs/2409.12186, 2024.

\bibitem{abs-2403-07974}
Naman Jain, King Han, Alex Gu, Wen{-}Ding Li, Fanjia Yan, Tianjun Zhang, Sida Wang, Armando Solar{-}Lezama, Koushik Sen, and Ion Stoica.
\newblock Livecodebench: Holistic and contamination free evaluation of large language models for code.
\newblock {\em CoRR}, abs/2403.07974, 2024.

\bibitem{JiangXAN20}
Zhengbao Jiang, Frank~F. Xu, Jun Araki, and Graham Neubig.
\newblock How can we know what language models know.
\newblock {\em Trans. Assoc. Comput. Linguistics}, 8:423--438, 2020.

\bibitem{KwonLZ0ZY0ZS23}
Woosuk Kwon, Zhuohan Li, Siyuan Zhuang, Ying Sheng, Lianmin Zheng, Cody~Hao Yu, Joseph Gonzalez, Hao Zhang, and Ion Stoica.
\newblock Efficient memory management for large language model serving with pagedattention.
\newblock In {\em {SOSP}}, pages 611--626. {ACM}, 2023.

\bibitem{bespoke_stratos}
Bespoke Labs.
\newblock Bespoke-stratos: The unreasonable effectiveness of reasoning distillation, 2025.
\newblock Accessed: 2025-01-22.

\bibitem{LanMYX24}
Qingfeng Lan, A.~Rupam Mahmood, Shuicheng Yan, and Zhongwen Xu.
\newblock Learning to optimize for reinforcement learning.
\newblock {\em {RLJ}}, 2:481--497, 2024.

\bibitem{LeeMRS19}
Kwonjoon Lee, Subhransu Maji, Avinash Ravichandran, and Stefano Soatto.
\newblock Meta-learning with differentiable convex optimization.
\newblock In {\em {CVPR}}, pages 10657--10665. Computer Vision Foundation / {IEEE}, 2019.

\bibitem{Li0TSG18}
Hao Li, Zheng Xu, Gavin Taylor, Christoph Studer, and Tom Goldstein.
\newblock Visualizing the loss landscape of neural nets.
\newblock In {\em NeurIPS}, pages 6391--6401, 2018.

\bibitem{LiM17}
Ke~Li and Jitendra Malik.
\newblock Learning to optimize.
\newblock In {\em {ICLR} (Poster)}. OpenReview.net, 2017.

\bibitem{LiM17b}
Ke~Li and Jitendra Malik.
\newblock Learning to optimize neural nets.
\newblock {\em CoRR}, abs/1703.00441, 2017.

\bibitem{abs-2501-10069}
Xinzhe Li.
\newblock A survey on {LLM} test-time compute via search: Tasks, {LLM} profiling, search algorithms, and relevant frameworks.
\newblock {\em CoRR}, abs/2501.10069, 2025.

\bibitem{LightmanKBEBLLS24}
Hunter Lightman, Vineet Kosaraju, Yuri Burda, Harrison Edwards, Bowen Baker, Teddy Lee, Jan Leike, John Schulman, Ilya Sutskever, and Karl Cobbe.
\newblock Let's verify step by step.
\newblock In {\em {ICLR}}. OpenReview.net, 2024.

\bibitem{LiuAGKZ23}
Bingbin Liu, Jordan~T. Ash, Surbhi Goel, Akshay Krishnamurthy, and Cyril Zhang.
\newblock Transformers learn shortcuts to automata.
\newblock In {\em {ICLR}}. OpenReview.net, 2023.

\bibitem{LiuWSFZH20}
Chenghao Liu, Zhihao Wang, Doyen Sahoo, Yuan Fang, Kun Zhang, and Steven C.~H. Hoi.
\newblock Adaptive task sampling for meta-learning.
\newblock In {\em {ECCV}}, volume 12363 of {\em Lecture Notes in Computer Science}, pages 752--769. Springer, 2020.

\bibitem{abs-2412-13147}
Junnan Liu, Hongwei Liu, Linchen Xiao, Ziyi Wang, Kuikun Liu, Songyang Gao, Wenwei Zhang, Songyang Zhang, and Kai Chen.
\newblock Are your llms capable of stable reasoning?
\newblock {\em CoRR}, abs/2412.13147, 2024.

\bibitem{ma2025reasoning}
Wenjie Ma, Jingxuan He, Charlie Snell, Tyler Griggs, Sewon Min, and Matei Zaharia.
\newblock Reasoning models can be effective without thinking.
\newblock {\em CoRR}, abs/2504.09858, 2025.

\bibitem{MerrillSS22}
William Merrill, Ashish Sabharwal, and Noah~A. Smith.
\newblock Saturated transformers are constant-depth threshold circuits.
\newblock {\em Trans. Assoc. Comput. Linguistics}, 10:843--856, 2022.

\bibitem{merriman1877list}
Mansfield Merriman.
\newblock {\em A List of Writings Relating to the Method of Least Squares: With Historical and Critical Notes}, volume~4.
\newblock Academy, 1877.

\bibitem{MinLZH22}
Sewon Min, Mike Lewis, Luke Zettlemoyer, and Hannaneh Hajishirzi.
\newblock Metaicl: Learning to learn in context.
\newblock In {\em {NAACL-HLT}}, pages 2791--2809. Association for Computational Linguistics, 2022.

\bibitem{MinLHALHZ22}
Sewon Min, Xinxi Lyu, Ari Holtzman, Mikel Artetxe, Mike Lewis, Hannaneh Hajishirzi, and Luke Zettlemoyer.
\newblock Rethinking the role of demonstrations: What makes in-context learning work?
\newblock In {\em {EMNLP}}, pages 11048--11064. Association for Computational Linguistics, 2022.

\bibitem{abs-2412-09413}
Yingqian Min, Zhipeng Chen, Jinhao Jiang, Jie Chen, Jia Deng, Yiwen Hu, Yiru Tang, Jiapeng Wang, Xiaoxue Cheng, Huatong Song, Wayne~Xin Zhao, Zheng Liu, Zhongyuan Wang, and Ji{-}Rong Wen.
\newblock Imitate, explore, and self-improve: {A} reproduction report on slow-thinking reasoning systems.
\newblock {\em CoRR}, abs/2412.09413, 2024.

\bibitem{MnihKSGAWR13}
Volodymyr Mnih, Koray Kavukcuoglu, David Silver, Alex Graves, Ioannis Antonoglou, Daan Wierstra, and Martin~A. Riedmiller.
\newblock Playing atari with deep reinforcement learning.
\newblock {\em CoRR}, abs/1312.5602, 2013.

\bibitem{moore1920reciprocal}
Eliakim~H Moore.
\newblock On the reciprocal of the general algebraic matrix.
\newblock {\em Bull. Am. Math. Soc.}, 26:294--295, 1920.

\bibitem{abs-2501-19393}
Niklas Muennighoff, Zitong Yang, Weijia Shi, Xiang~Lisa Li, Li~Fei{-}Fei, Hannaneh Hajishirzi, Luke Zettlemoyer, Percy Liang, Emmanuel~J. Cand{\`{e}}s, and Tatsunori Hashimoto.
\newblock s1: Simple test-time scaling.
\newblock {\em CoRR}, abs/2501.19393, 2025.

\bibitem{abs-2209-11895}
Catherine Olsson, Nelson Elhage, Neel Nanda, Nicholas Joseph, Nova DasSarma, Tom Henighan, Ben Mann, Amanda Askell, Yuntao Bai, Anna Chen, Tom Conerly, Dawn Drain, Deep Ganguli, Zac Hatfield{-}Dodds, Danny Hernandez, Scott Johnston, Andy Jones, Jackson Kernion, Liane Lovitt, Kamal Ndousse, Dario Amodei, Tom Brown, Jack Clark, Jared Kaplan, Sam McCandlish, and Chris Olah.
\newblock In-context learning and induction heads.
\newblock {\em CoRR}, abs/2209.11895, 2022.

\bibitem{abs-2303-08774}
OpenAI.
\newblock {GPT-4} technical report.
\newblock {\em CoRR}, abs/2303.08774, 2023.

\bibitem{openaiO1}
{OpenAI}.
\newblock Learning to reason with llms.
\newblock \url{https://openai.com/index/learning-to-reason-with-llms/}, 2024.
\newblock Accessed: 2024-09-12.

\bibitem{Ouyang0JAWMZASR22}
Long Ouyang, Jeffrey Wu, Xu~Jiang, Diogo Almeida, Carroll~L. Wainwright, Pamela Mishkin, Chong Zhang, Sandhini Agarwal, Katarina Slama, Alex Ray, John Schulman, Jacob Hilton, Fraser Kelton, Luke Miller, Maddie Simens, Amanda Askell, Peter Welinder, Paul~F. Christiano, Jan Leike, and Ryan Lowe.
\newblock Training language models to follow instructions with human feedback.
\newblock In {\em NeurIPS}, 2022.

\bibitem{PaszkeGMLBCKLGA19}
Adam Paszke, Sam Gross, Francisco Massa, Adam Lerer, James Bradbury, Gregory Chanan, Trevor Killeen, Zeming Lin, Natalia Gimelshein, Luca Antiga, Alban Desmaison, Andreas K{\"{o}}pf, Edward~Z. Yang, Zachary DeVito, Martin Raison, Alykhan Tejani, Sasank Chilamkurthy, Benoit Steiner, Lu~Fang, Junjie Bai, and Soumith Chintala.
\newblock Pytorch: An imperative style, high-performance deep learning library.
\newblock In {\em NeurIPS}, pages 8024--8035, 2019.

\bibitem{penrose1955generalized}
Roger Penrose.
\newblock A generalized inverse for matrices.
\newblock In {\em Math. Proc. Cambridge Philos. Soc.}, volume~51, pages 406--413. Cambridge University Press, 1955.

\bibitem{Premont-Schwarz22}
Isabeau Pr{\'{e}}mont{-}Schwarz, Jaroslav Vitku, and Jan Feyereisl.
\newblock A simple guard for learned optimizers.
\newblock In {\em {ICML}}, volume 162 of {\em Proceedings of Machine Learning Research}, pages 17910--17925. {PMLR}, 2022.

\bibitem{abs-2410-18982}
Yiwei Qin, Xuefeng Li, Haoyang Zou, Yixiu Liu, Shijie Xia, Zhen Huang, Yixin Ye, Weizhe Yuan, Hector Liu, Yuanzhi Li, and Pengfei Liu.
\newblock {O1} replication journey: {A} strategic progress report - part 1.
\newblock {\em CoRR}, abs/2410.18982, 2024.

\bibitem{abs-2503-21614}
Xiaoye Qu, Yafu Li, Zhaochen Su, Weigao Sun, Jianhao Yan, Dongrui Liu, Ganqu Cui, Daizong Liu, Shuxian Liang, Junxian He, Peng Li, Wei Wei, Jing Shao, Chaochao Lu, Yue Zhang, Xian{-}Sheng Hua, Bowen Zhou, and Yu~Cheng.
\newblock A survey of efficient reasoning for large reasoning models: Language, multimodality, and beyond.
\newblock {\em CoRR}, abs/2503.21614, 2025.

\bibitem{radford2018improving}
Alec Radford, Karthik Narasimhan, Tim Salimans, Ilya Sutskever, et~al.
\newblock Improving language understanding by generative pre-training.
\newblock {\em OpenAI}, 2018.

\bibitem{RafailovSMMEF23}
Rafael Rafailov, Archit Sharma, Eric Mitchell, Christopher~D. Manning, Stefano Ermon, and Chelsea Finn.
\newblock Direct preference optimization: Your language model is secretly a reward model.
\newblock In {\em NeurIPS}, 2023.

\bibitem{RajeswaranFKL19}
Aravind Rajeswaran, Chelsea Finn, Sham~M. Kakade, and Sergey Levine.
\newblock Meta-learning with implicit gradients.
\newblock In {\em NeurIPS}, pages 113--124, 2019.

\bibitem{RaviL17}
Sachin Ravi and Hugo Larochelle.
\newblock Optimization as a model for few-shot learning.
\newblock In {\em {ICLR}}. OpenReview.net, 2017.

\bibitem{abs-2311-12022}
David Rein, Betty~Li Hou, Asa~Cooper Stickland, Jackson Petty, Richard~Yuanzhe Pang, Julien Dirani, Julian Michael, and Samuel~R. Bowman.
\newblock {GPQA:} {A} graduate-level google-proof q{\&}a benchmark.
\newblock {\em CoRR}, abs/2311.12022, 2023.

\bibitem{RubinHB22}
Ohad Rubin, Jonathan Herzig, and Jonathan Berant.
\newblock Learning to retrieve prompts for in-context learning.
\newblock In {\em {NAACL-HLT}}, pages 2655--2671. Association for Computational Linguistics, 2022.

\bibitem{SantoroBBWL16}
Adam Santoro, Sergey Bartunov, Matthew~M. Botvinick, Daan Wierstra, and Timothy~P. Lillicrap.
\newblock One-shot learning with memory-augmented neural networks.
\newblock {\em CoRR}, abs/1605.06065, 2016.

\bibitem{Schmidhuber09}
J{\"{u}}rgen Schmidhuber.
\newblock Evolutionary principles in self-referential learning, or on learning how to learn: The meta-meta-. hook.
\newblock Master's thesis, Technical University of Munich, Germany, 1987.

\bibitem{SchulmanWDRK17}
John Schulman, Filip Wolski, Prafulla Dhariwal, Alec Radford, and Oleg Klimov.
\newblock Proximal policy optimization algorithms.
\newblock {\em CoRR}, abs/1707.06347, 2017.

\bibitem{shah2025rethinking}
Darsh~J Shah, Peter Rushton, Somanshu Singla, Mohit Parmar, Kurt Smith, Yash Vanjani, Ashish Vaswani, Adarsh Chaluvaraju, Andrew Hojel, Andrew Ma, et~al.
\newblock Rethinking reflection in pre-training.
\newblock {\em CoRR}, abs/2504.04022, 2025.

\bibitem{abs-2402-03300}
Zhihong Shao, Peiyi Wang, Qihao Zhu, Runxin Xu, Junxiao Song, Mingchuan Zhang, Y.~K. Li, Y.~Wu, and Daya Guo.
\newblock Deepseekmath: Pushing the limits of mathematical reasoning in open language models.
\newblock {\em CoRR}, abs/2402.03300, 2024.

\bibitem{ShengZYWZZPL025}
Guangming Sheng, Chi Zhang, Zilingfeng Ye, Xibin Wu, Wang Zhang, Ru~Zhang, Yanghua Peng, Haibin Lin, and Chuan Wu.
\newblock Hybridflow: {A} flexible and efficient {RLHF} framework.
\newblock In {\em EuroSys}, pages 1279--1297. {ACM}, 2025.

\bibitem{abs-2503-08200}
Wei Shi, Sihang Li, Tao Liang, Mingyang Wan, Gojun Ma, Xiang Wang, and Xiangnan He.
\newblock Route sparse autoencoder to interpret large language models.
\newblock {\em CoRR}, abs/2503.08200, 2025.

\bibitem{SinghCAAPGLH0XP24}
Avi Singh, John~D. Co{-}Reyes, Rishabh Agarwal, Ankesh Anand, Piyush Patil, Xavier Garcia, Peter~J. Liu, James Harrison, Jaehoon Lee, Kelvin Xu, Aaron~T. Parisi, Abhishek Kumar, Alexander~A. Alemi, Alex Rizkowsky, Azade Nova, Ben Adlam, Bernd Bohnet, Gamaleldin~Fathy Elsayed, Hanie Sedghi, Igor Mordatch, Isabelle Simpson, Izzeddin Gur, Jasper Snoek, Jeffrey Pennington, Jiri Hron, Kathleen Kenealy, Kevin Swersky, Kshiteej Mahajan, Laura Culp, Lechao Xiao, Maxwell~L. Bileschi, Noah Constant, Roman Novak, Rosanne Liu, Tris Warkentin, Yundi Qian, Yamini Bansal, Ethan Dyer, Behnam Neyshabur, Jascha Sohl{-}Dickstein, and Noah Fiedel.
\newblock Beyond human data: Scaling self-training for problem-solving with language models.
\newblock {\em Trans. Mach. Learn. Res.}, 2024, 2024.

\bibitem{SnellSZ17}
Jake Snell, Kevin Swersky, and Richard~S. Zemel.
\newblock Prototypical networks for few-shot learning.
\newblock In {\em {NIPS}}, pages 4077--4087, 2017.

\bibitem{SukhbaatarSWF15}
Sainbayar Sukhbaatar, Arthur Szlam, Jason Weston, and Rob Fergus.
\newblock End-to-end memory networks.
\newblock In {\em {NIPS}}, pages 2440--2448, 2015.

\bibitem{abs-2503-01067}
Gokul Swamy, Sanjiban Choudhury, Wen Sun, Zhiwei~Steven Wu, and J.~Andrew Bagnell.
\newblock All roads lead to likelihood: The value of reinforcement learning in fine-tuning.
\newblock {\em CoRR}, abs/2503.01067, 2025.

\bibitem{TangLPT20}
Hao Tang, Zechao Li, Zhimao Peng, and Jinhui Tang.
\newblock Blockmix: Meta regularization and self-calibrated inference for metric-based meta-learning.
\newblock In {\em {ACM} Multimedia}, pages 610--618. {ACM}, 2020.

\bibitem{abs-2501-12599}
Kimi Team, Angang Du, Bofei Gao, Bowei Xing, Changjiu Jiang, Cheng Chen, Cheng Li, Chenjun Xiao, Chenzhuang Du, Chonghua Liao, Chuning Tang, Congcong Wang, Dehao Zhang, Enming Yuan, Enzhe Lu, Fengxiang Tang, Flood Sung, Guangda Wei, Guokun Lai, Haiqing Guo, Han Zhu, Hao Ding, Hao Hu, Hao Yang, Hao Zhang, Haotian Yao, Haotian Zhao, Haoyu Lu, Haoze Li, Haozhen Yu, Hongcheng Gao, Huabin Zheng, Huan Yuan, Jia Chen, Jianhang Guo, Jianlin Su, Jianzhou Wang, Jie Zhao, Jin Zhang, Jingyuan Liu, Junjie Yan, Junyan Wu, Lidong Shi, Ling Ye, Longhui Yu, Mengnan Dong, Neo Zhang, Ningchen Ma, Qiwei Pan, Qucheng Gong, Shaowei Liu, Shengling Ma, Shupeng Wei, Sihan Cao, Siying Huang, Tao Jiang, Weihao Gao, Weimin Xiong, Weiran He, Weixiao Huang, Wenhao Wu, Wenyang He, Xianghui Wei, Xianqing Jia, Xingzhe Wu, Xinran Xu, Xinxing Zu, Xinyu Zhou, Xuehai Pan, Y.~Charles, Yang Li, Yangyang Hu, Yangyang Liu, Yanru Chen, Yejie Wang, Yibo Liu, Yidao Qin, Yifeng Liu, Ying Yang, Yiping Bao, Yulun Du, Yuxin Wu, Yuzhi Wang, Zaida Zhou,
  Zhaoji Wang, Zhaowei Li, Zhen Zhu, Zheng Zhang, Zhexu Wang, Zhilin Yang, Zhiqi Huang, Zihao Huang, Ziyao Xu, and Zonghan Yang.
\newblock Kimi k1.5: Scaling reinforcement learning with llms.
\newblock {\em CoRR}, abs/2501.12599, 2025.

\bibitem{sky_t1_2025}
NovaSky Team.
\newblock Sky-t1: Train your own o1 preview model within \$450.
\newblock \url{https://novasky-ai.github.io/posts/sky-t1}, 2025.
\newblock Accessed: 2025-01-09.

\bibitem{openthoughts}
OpenThoughts Team.
\newblock {Open Thoughts}.
\newblock \url{https://open-thoughts.ai}, 2025.

\bibitem{qwen3}
Qwen Team.
\newblock Qwen3: Think deeper, act faster.
\newblock \url{https://qwenlm.github.io/blog/qwen3/}, April 2025.

\bibitem{qwq}
Qwen Team.
\newblock Qwq-32b: Embracing the power of reinforcement learning.
\newblock \url{https://qwenlm.github.io/blog/qwq-32b/}, March 2025.

\bibitem{abs-2502-12018}
Fengwei Teng, Zhaoyang Yu, Quan Shi, Jiayi Zhang, Chenglin Wu, and Yuyu Luo.
\newblock Atom of thoughts for markov {LLM} test-time scaling.
\newblock {\em CoRR}, abs/2502.12018, 2025.

\bibitem{abs-2406-00153}
Benjamin Th{\'{e}}rien, Charles{-}{\'{E}}tienne Joseph, Boris Knyazev, Edouard Oyallon, Irina Rish, and Eugene Belilovsky.
\newblock {\(\mu\)}lo: Compute-efficient meta-generalization of learned optimizers.
\newblock {\em CoRR}, abs/2406.00153, 2024.

\bibitem{ThrunP98}
Sebastian Thrun and Lorien~Y. Pratt.
\newblock Learning to learn: Introduction and overview.
\newblock In {\em Learning to Learn}, pages 3--17. Springer, 1998.

\bibitem{TriantafillouLZ21}
Eleni Triantafillou, Hugo Larochelle, Richard~S. Zemel, and Vincent Dumoulin.
\newblock Learning a universal template for few-shot dataset generalization.
\newblock In {\em {ICML}}, volume 139 of {\em Proceedings of Machine Learning Research}, pages 10424--10433. {PMLR}, 2021.

\bibitem{TriantafillouZD20}
Eleni Triantafillou, Tyler Zhu, Vincent Dumoulin, Pascal Lamblin, Utku Evci, Kelvin Xu, Ross Goroshin, Carles Gelada, Kevin Swersky, Pierre{-}Antoine Manzagol, and Hugo Larochelle.
\newblock Meta-dataset: {A} dataset of datasets for learning to learn from few examples.
\newblock In {\em {ICLR}}. OpenReview.net, 2020.

\bibitem{VaswaniSPUJGKP17}
Ashish Vaswani, Noam Shazeer, Niki Parmar, Jakob Uszkoreit, Llion Jones, Aidan~N. Gomez, Lukasz Kaiser, and Illia Polosukhin.
\newblock Attention is all you need.
\newblock In {\em {NIPS}}, pages 5998--6008, 2017.

\bibitem{WangRZLLSZSZ024}
Ke~Wang, Houxing Ren, Aojun Zhou, Zimu Lu, Sichun Luo, Weikang Shi, Renrui Zhang, Linqi Song, Mingjie Zhan, and Hongsheng Li.
\newblock Mathcoder: Seamless code integration in llms for enhanced mathematical reasoning.
\newblock In {\em {ICLR}}. OpenReview.net, 2024.

\bibitem{WangGSQ22}
Zhe Wang, Jake Grigsby, Arshdeep Sekhon, and Yanjun Qi.
\newblock {ST-MAML} : {A} stochastic-task based method for task-heterogeneous meta-learning.
\newblock In {\em {UAI}}, volume 180 of {\em Proceedings of Machine Learning Research}, pages 2066--2074. {PMLR}, 2022.

\bibitem{Wei0SBIXCLZ22}
Jason Wei, Xuezhi Wang, Dale Schuurmans, Maarten Bosma, Brian Ichter, Fei Xia, Ed~H. Chi, Quoc~V. Le, and Denny Zhou.
\newblock Chain-of-thought prompting elicits reasoning in large language models.
\newblock In {\em NeurIPS}, 2022.

\bibitem{WeissGY21}
Gail Weiss, Yoav Goldberg, and Eran Yahav.
\newblock Thinking like transformers.
\newblock In {\em {ICML}}, volume 139 of {\em Proceedings of Machine Learning Research}, pages 11080--11090. {PMLR}, 2021.

\bibitem{abs-2503-10460}
Liang Wen, Yunke Cai, Fenrui Xiao, Xin He, Qi~An, Zhenyu Duan, Yimin Du, Junchen Liu, Lifu Tang, Xiaowei Lv, Haosheng Zou, Yongchao Deng, Shousheng Jia, and Xiangzheng Zhang.
\newblock Light-r1: Curriculum sft, {DPO} and {RL} for long {COT} from scratch and beyond.
\newblock {\em CoRR}, abs/2503.10460, 2025.

\bibitem{WestonCB14}
Jason Weston, Sumit Chopra, and Antoine Bordes.
\newblock Memory networks.
\newblock In {\em {ICLR}}, 2015.

\bibitem{WolfDSCDMCRLFDS20}
Thomas Wolf, Lysandre Debut, Victor Sanh, Julien Chaumond, Clement Delangue, Anthony Moi, Pierric Cistac, Tim Rault, R{\'{e}}mi Louf, Morgan Funtowicz, Joe Davison, Sam Shleifer, Patrick von Platen, Clara Ma, Yacine Jernite, Julien Plu, Canwen Xu, Teven~Le Scao, Sylvain Gugger, Mariama Drame, Quentin Lhoest, and Alexander~M. Rush.
\newblock Transformers: State-of-the-art natural language processing.
\newblock In {\em {EMNLP} (Demos)}, pages 38--45. Association for Computational Linguistics, 2020.

\bibitem{XieRL022}
Sang~Michael Xie, Aditi Raghunathan, Percy Liang, and Tengyu Ma.
\newblock An explanation of in-context learning as implicit bayesian inference.
\newblock In {\em {ICLR}}. OpenReview.net, 2022.

\bibitem{0015DYW0J0024}
Wei Xiong, Hanze Dong, Chenlu Ye, Ziqi Wang, Han Zhong, Heng Ji, Nan Jiang, and Tong Zhang.
\newblock Iterative preference learning from human feedback: Bridging theory and practice for {RLHF} under kl-constraint.
\newblock In {\em {ICML}}. OpenReview.net, 2024.

\bibitem{abs-2412-15115}
An~Yang, Baosong Yang, Beichen Zhang, Binyuan Hui, Bo~Zheng, Bowen Yu, Chengyuan Li, Dayiheng Liu, Fei Huang, Haoran Wei, Huan Lin, Jian Yang, Jianhong Tu, Jianwei Zhang, Jianxin Yang, Jiaxi Yang, Jingren Zhou, Junyang Lin, Kai Dang, Keming Lu, Keqin Bao, Kexin Yang, Le~Yu, Mei Li, Mingfeng Xue, Pei Zhang, Qin Zhu, Rui Men, Runji Lin, Tianhao Li, Tingyu Xia, Xingzhang Ren, Xuancheng Ren, Yang Fan, Yang Su, Yichang Zhang, Yu~Wan, Yuqiong Liu, Zeyu Cui, Zhenru Zhang, and Zihan Qiu.
\newblock Qwen2.5 technical report.
\newblock {\em CoRR}, abs/2412.15115, 2024.

\bibitem{abs-2409-12122}
An~Yang, Beichen Zhang, Binyuan Hui, Bofei Gao, Bowen Yu, Chengpeng Li, Dayiheng Liu, Jianhong Tu, Jingren Zhou, Junyang Lin, Keming Lu, Mingfeng Xue, Runji Lin, Tianyu Liu, Xingzhang Ren, and Zhenru Zhang.
\newblock Qwen2.5-math technical report: Toward mathematical expert model via self-improvement.
\newblock {\em CoRR}, abs/2409.12122, 2024.

\bibitem{abs-2502-18080}
Wenkai Yang, Shuming Ma, Yankai Lin, and Furu Wei.
\newblock Towards thinking-optimal scaling of test-time compute for {LLM} reasoning.
\newblock {\em CoRR}, abs/2502.18080, 2025.

\bibitem{YaoWWZMLF21}
Huaxiu Yao, Yu~Wang, Ying Wei, Peilin Zhao, Mehrdad Mahdavi, Defu Lian, and Chelsea Finn.
\newblock Meta-learning with an adaptive task scheduler.
\newblock In {\em NeurIPS}, pages 7497--7509, 2021.

\bibitem{YaoPPN20}
Shunyu Yao, Binghui Peng, Christos~H. Papadimitriou, and Karthik Narasimhan.
\newblock Self-attention networks can process bounded hierarchical languages.
\newblock In {\em {ACL/IJCNLP}}, pages 3770--3785. Association for Computational Linguistics, 2021.

\bibitem{YaoYZS00N23}
Shunyu Yao, Dian Yu, Jeffrey Zhao, Izhak Shafran, Tom Griffiths, Yuan Cao, and Karthik Narasimhan.
\newblock Tree of thoughts: Deliberate problem solving with large language models.
\newblock In {\em NeurIPS}, 2023.

\bibitem{YeLYGXZ21}
Feiyang Ye, Baijiong Lin, Zhixiong Yue, Pengxin Guo, Qiao Xiao, and Yu~Zhang.
\newblock Multi-objective meta learning.
\newblock In {\em NeurIPS}, pages 21338--21351, 2021.

\bibitem{abs-2502-03387}
Yixin Ye, Zhen Huang, Yang Xiao, Ethan Chern, Shijie Xia, and Pengfei Liu.
\newblock {LIMO:} less is more for reasoning.
\newblock {\em CoRR}, abs/2502.03387, 2025.

\bibitem{abs-2503-14476}
Qiying Yu, Zheng Zhang, Ruofei Zhu, Yufeng Yuan, Xiaochen Zuo, Yu~Yue, Tiantian Fan, Gaohong Liu, Lingjun Liu, Xin Liu, Haibin Lin, Zhiqi Lin, Bole Ma, Guangming Sheng, Yuxuan Tong, Chi Zhang, Mofan Zhang, Wang Zhang, Hang Zhu, Jinhua Zhu, Jiaze Chen, Jiangjie Chen, Chengyi Wang, Hongli Yu, Weinan Dai, Yuxuan Song, Xiangpeng Wei, Hao Zhou, Jingjing Liu, Wei{-}Ying Ma, Ya{-}Qin Zhang, Lin Yan, Mu~Qiao, Yonghui Wu, and Mingxuan Wang.
\newblock {DAPO:} an open-source {LLM} reinforcement learning system at scale.
\newblock {\em CoRR}, abs/2503.14476, 2025.

\bibitem{abs-2308-01825}
Zheng Yuan, Hongyi Yuan, Chengpeng Li, Guanting Dong, Chuanqi Tan, and Chang Zhou.
\newblock Scaling relationship on learning mathematical reasoning with large language models.
\newblock {\em CoRR}, abs/2308.01825, 2023.

\bibitem{abs-2504-05118}
Yu~Yue, Yufeng Yuan, Qiying Yu, Xiaochen Zuo, Ruofei Zhu, Wenyuan Xu, Jiaze Chen, Chengyi Wang, TianTian Fan, Zhengyin Du, Xiangpeng Wei, Xiangyu Yu, Gaohong Liu, Juncai Liu, Lingjun Liu, Haibin Lin, Zhiqi Lin, Bole Ma, Chi Zhang, Mofan Zhang, Wang Zhang, Hang Zhu, Ru~Zhang, Xin Liu, Mingxuan Wang, Yonghui Wu, and Lin Yan.
\newblock {VAPO:} efficient and reliable reinforcement learning for advanced reasoning tasks.
\newblock {\em CoRR}, abs/2504.05118, 2025.

\bibitem{YunBRRK20}
Chulhee Yun, Srinadh Bhojanapalli, Ankit~Singh Rawat, Sashank~J. Reddi, and Sanjiv Kumar.
\newblock Are transformers universal approximators of sequence-to-sequence functions?
\newblock In {\em {ICLR}}. OpenReview.net, 2020.

\bibitem{YunCBRRK20}
Chulhee Yun, Yin{-}Wen Chang, Srinadh Bhojanapalli, Ankit~Singh Rawat, Sashank~J. Reddi, and Sanjiv Kumar.
\newblock O(n) connections are expressive enough: Universal approximability of sparse transformers.
\newblock In {\em NeurIPS}, 2020.

\bibitem{zhang2025reasoning}
Anqi Zhang, Yulin Chen, Jane Pan, Chen Zhao, Aurojit Panda, Jinyang Li, and He~He.
\newblock Reasoning models know when they're right: Probing hidden states for self-verification.
\newblock {\em CoRR}, abs/2504.05419, 2025.

\bibitem{ZhangWSFC23}
Yangguang Zhang, Can Wang, Qihao Shi, Yan Feng, and Chun Chen.
\newblock Adversarial gradient-based meta learning with metric-based test.
\newblock {\em Knowl. Based Syst.}, 263:110312, 2023.

\end{thebibliography}


\clearpage
\appendix
\section{Notations} \label{app:notations}

\begin{table}[htb]
    \centering
    \renewcommand{\arraystretch}{1.15}
    \caption{Illustration of notations used in the paper.} \label{tab:notations}
    \begin{tabularx}{.8\textwidth}{p{0.15\textwidth}X}
        \toprule
        $\mathcal{M}$ & the large language model \\
        $\theta$ & the parameters of the large language model \\
        $q$ & the question \\
        $\mathcal{Q}$ & the question set \\
        $q_i$ & the $i$-th question in the question set \\
        $q_i^j$ & the $j$-th token of the $i$-th question \\
        $a$ & the answer \\
        $\mathcal{A}$ & the answer set \\
        $a_i$ & the $i$-th answer of the $i$-th question \\
        $a_i^j$ & the $j$-th token of the $i$-th answer \\
        $\vert \cdot \vert$ & the length of tokens \\
        $I$ & the instruction \\
        $\mathrm{d}$ & the autoregressive decoding mechanism \\
        $\mathrm{Softmax}$ & the softmax function \\
        $\sigma$ & the activation function \\
        $t$ & the reasoning trajectory \\
        $\mathcal{T}$ & the set of reasoning trajectory \\
        $\bm{E}_o$ & the whole output token embedding of LLM \\
        $\bm{E}_{x,:}$ & the input token embedding of the sequence $x$ \\
        $\bm{E}_{x,i}$ & the input token embedding of the $i$-th token in the sequence $x$ \\
        $\mathrm{Softmax}(\cdot)[x]$ & the value of softmax vector in the entry corresponding to $x$ \\
        $p(\cdot)$ & the probability distribution of one token sequence determined by the LLM \\
        $\Delta \mathcal{M}_{\theta}(\cdot)$ & the variation of the parameter $\theta$ corresponding to the inputs \\
        $\theta^\prime_{i}$ & the $i$-th step updated parameters \\
        $\theta^\prime_{t}$ & the updated parameters $\theta$ corresponding to the reasoning trajectory $t$ \\
        $\left [ \begin{matrix} \bm{x} \\ \bm{y} \end{matrix} \right ]$ & the concatenation of $\bm{x}$ and $\bm{y}$ \\
        $\bm{W}_k,\bm{W}_q,\bm{W}_v$ & the parameters in self-attention layer \\
        $\bm{W}_1,\bm{W}_2,b_1,b_2$ & the parameters in feed-forward network \\
        $\mathcal{L}_q$ & the loss corresponding to the question \\
        \bottomrule
    \end{tabularx}
\end{table}

\section{Proof of \Cref{pro:one-step-pseudo-update}} \label{sec:proof_one-step-pseudo-update}

\begin{proof}

Recall that our objective is to determine the set $\{\bm{W}_q', \bm{W}_k', \bm{W}_v', \bm{W}_1', \bm{W}_2', b_1', b_2'\}$ such that:
\begin{equation}
    \begin{aligned}
        & \bm{W}_2^T \left( \sigma \left( \bm{W}_1^T \left( \mathrm{Softmax} \left( \bm{E}_{t,0} \bm{W}_q \bm{W}_k^T  
        \left [ \begin{matrix}
            \bm{E}_{l,:} \\
            \bm{E}_{t,0} \\
        \end{matrix} \right ]^T
        \right) 
        \left [ \begin{matrix}
            \bm{E}_{l,:} \\
            \bm{E}_{t,0} \\
        \end{matrix} \right ]
        \bm{W}_v \right) + b_1\right) \right) + b_2 = \\
        & \;\;\;\; \bm{W}_2^{\prime T} \left( \sigma \left( \bm{W}_1^{\prime T} \left( \mathrm{Softmax} \left( \bm{E}_{l,-1:} \bm{W}_q^\prime \bm{W}_k^{\prime T} \bm{E}_{l,:}^T \right)  \bm{E}_{l,:} \bm{W}_v^\prime \right) + b_1^\prime \right) \right) + b_2^\prime,
    \end{aligned}
\end{equation}
where $\bm{E}_{l,:} \in \mathbb{R}^{\vert l \vert \times d}$, $\bm{E}_{l,-1}, \bm{E}_{t,0} \in \mathbb{R}^{1 \times d}$, $\bm{W}_q,\bm{W}_q^\prime,\bm{W}_k,\bm{W}_k^\prime,\bm{W}_v,\bm{W}_v^\prime \in \mathbb{R}^{d\times d}$, and $\bm{W}_1,\bm{W}_1^\prime,\bm{W}_2,\bm{W}_2^\prime \in \mathbb{R}^{d\times d}$.

We might as well let $\bm{W}_1^\prime, \bm{W}_2^\prime, b_1^\prime, b_2^\prime$ equal to $\bm{W}_1, \bm{W}_2, b_1, b_2$~(\ding{182}), respectively, as follows:
\begin{equation}
    \begin{aligned}
        \bm{W}_1^\prime &= \bm{W}_1, \\
        \bm{W}_2^\prime &= \bm{W}_2, \\
        b_1^\prime &= b_1, \\
        b_2^\prime &= b_2.
    \end{aligned}
\end{equation}
Then, we only need to establish the following equality:
\begin{equation} \label{eq:attention-equality}
    \begin{aligned}
        & \mathrm{Softmax} \left( \bm{E}_{t,0} \bm{W}_q \bm{W}_k^T  
        \left [ \begin{matrix}
            \bm{E}_{l,:} \\
            \bm{E}_{t,0} \\
        \end{matrix} \right ]^T
        \right) 
        \left [ \begin{matrix}
            \bm{E}_{l,:} \\
            \bm{E}_{t,0} \\
        \end{matrix} \right ]
        \bm{W}_v = \\
        & \;\;\;\; \mathrm{Softmax} \left( \bm{E}_{l,-1:} \bm{W}_q^\prime \bm{W}_k^{\prime T} \bm{E}_{l,:}^T \right)  \bm{E}_{l,:} \bm{W}_v^\prime.
    \end{aligned}
\end{equation}
For simplicity, we refine \Cref{eq:attention-equality} into several parts:
\begin{equation}
    \begin{aligned}
        & \bm{Q} = \bm{E}_{t,0} \bm{W}_q \in \mathbb{R}^{1\times d}, \quad \\
        & \bm{K}^T = \bm{W}_k^T \left[ \bm{E}_{l,:},\bm{E}_{t,0} \right]^T \in \mathbb{R}^{(\vert l\vert + 1)\times d}, \quad \\
        & \bm{V} = \left[ \bm{E}_{l,:},\bm{E}_{t,0} \right] \bm{W}_v \in \mathbb{R}^{(\vert l\vert + 1)\times d}, \\
        & \bm{Q}^\prime = \bm{E}_{l,-1:} \bm{W}_q^\prime \in \mathbb{R}^{1 \times d}, \quad \\
        & \bm{K}^{\prime T} = \bm{W}_k^{\prime T} \bm{E}_{l,:}^T \in \mathbb{R}^{\vert l\vert \times d}, \quad \\
        & \bm{V}^\prime = \bm{E}_{l,:} \bm{W}_v^\prime \in \mathbb{R}^{\vert l\vert \times d},
    \end{aligned}
\end{equation}
where $d$ is the dimension size of embeddings.

Initially, considering the matrices $ \bm{Q} $, we will demonstrate the existence of a linear transformation matrix $ \bm{P} \in \mathbb{R}^{d \times d} $ such that:
\begin{equation} \label{eq:query-equal}
    \bm{E}_{t,0} \bm{P} = \bm{E}_{l,-1:}.
\end{equation}
To support this assertion, we reference \Cref{theo:linear-exist}:
\begin{theorem} \label{theo:linear-exist}
    Let $ U $ and $ V $ be vector spaces, and let $ \{\bm{b}_1, \bm{b}_2, \ldots, \bm{b}_n\} $ denote a basis of $ U $. For $ n $ vectors $ \bm{v}_i \in V $, there exists a linear transformation $ T: U \rightarrow V $ such that $ T(\bm{b}_i) = \bm{v}_i $ for each $ i=1, 2, \ldots, n $.
\end{theorem}
\begin{proof}
    We begin by defining a linear transformation $ T: U \rightarrow V $. Let $ \bm{u} $ be a vector in $ U $, expressed as $ \bm{u} = u_1 \bm{b}_1 + u_2 \bm{b}_2 + \cdots + u_n \bm{b}_n $, where the set $ \{\bm{b}_1, \bm{b}_2, \ldots, \bm{b}_n\} $ constitutes a basis and the coefficients $ u_1, u_2, \ldots, u_n $ are determined by $ \bm{u} $. The linear transformation $ T $ is constructed as follows:
    \begin{equation}
        T(\bm{u}) = u_1 \bm{v_1} + u_2 \bm{v_2} + \cdots + u_n \bm{v_n}.
    \end{equation}
    It is evident that this transformation $ T $ satisfies $ T(\bm{b}_i) = \bm{v}_i $ for each index $ i $.
\end{proof}
Based on \Cref{theo:linear-exist}, if we define a basis which involves $\bm{E}_{t,0}$~(e.g., $ \{ \bm{E}_{t,0}, \bm{0}, \ldots, \bm{0} \} $) and construct a corresponding vector space, we can derive a linear transformation matrix $ \bm{P} $ such that \Cref{eq:query-equal} holds, with $ \bm{P} $ being solely dependent on $ \bm{E}_{t,0} $. 
Therefore, if we let $ \bm{W}_q' = \bm{P} \bm{W}_q $~(\ding{183}), then we have $ \bm{Q} = \bm{Q}' $. 
Consequently, we simplify \Cref{eq:attention-equality} to:
\begin{equation}
    \mathrm{Softmax}\left( \bm{Q}\bm{K}^T \right) \bm{V} = \mathrm{Softmax}\left( \bm{Q}\bm{K}^{\prime T} \right) \bm{V}^\prime
\end{equation}

Now considering the matrix $\left [ \begin{matrix}
            \bm{E}_{l,:} \\
            \bm{E}_{t,0} \\
\end{matrix} \right ] \in \mathbb{R}^{(\vert l \vert + 1) \times d}$ and $\bm{E}_{l,:} \in \mathbb{R}^{\vert l\vert \times d}$, we can consistently identify a vector $\bm{C} \in \mathbb{R}^{1 \times \vert l \vert}$ such that:
\begin{equation}
    \bm{E}_{t,0} \approx \bm{C}\bm{E}_{l,:}.
\end{equation}
We examine the existence of $\bm{C}$ in two cases: 1) if $\bm{E}_{t,0}$ lies within the span of the row vectors of $\bm{E}_{l,;}$, then $\bm{C}$ obviously exists; 2) if $\bm{E}_{t,0}$ does not lie within the span of the row vectors of $\bm{E}_{l,;}$, an approximate solution for $\bm{C}$ can be derived using various methods, such as the \textit{least squares method}~\citep{merriman1877list}.
Then, let $\bm{M} = \left[ \bm{I}_l,\bm{C}^T \right] \in \mathbb{R}^{\vert l\vert \times (\vert l \vert + 1)}$, it follows that:
\begin{equation}
    \left [ \begin{matrix}
            \bm{E}_{l,:} \\
            \bm{E}_{t,0} \\
\end{matrix} \right ] \approx \bm{M}^T \bm{E}_{l,:}.
\end{equation}
We can express this relationship mathematically as follows:
\begin{equation}
    \begin{aligned}
        \mathrm{Softmax}\left( \bm{Q}\bm{K}^T \right) \bm{V} &= \mathrm{Softmax}\left( \bm{Q}\bm{W}_k^T  
        \left [ \begin{matrix}
            \bm{E}_{l,:} \\
            \bm{E}_{t,0} \\
        \end{matrix} \right ]^T
        \right) 
        \left [ \begin{matrix}
            \bm{E}_{l,:} \\
            \bm{E}_{t,0} \\
        \end{matrix} \right ]
        \bm{W}_v \\
        &\approx \mathrm{Softmax}\left( \bm{Q}\bm{W}_k^T \bm{E}_{l,:}^T \bm{M} \right)  \bm{M}^T \bm{E}_{l,:} \bm{W}_v \\
        & \Rightarrow \mathrm{Softmax}\left( \bm{Q}\bm{K}^{\prime T} \right) \bm{V}^\prime \\
        & = \mathrm{Softmax}\left( \bm{Q}\bm{W}_k^{\prime T} \bm{E}_{l,:}^T \right)  \bm{E}_{l,:} \bm{W}_v^\prime.
    \end{aligned}
\end{equation}
Thus, we obtain~(\ding{184}):
\begin{equation}
    \begin{aligned}
        \bm{W}_k^{\prime} &= \bm{E}_{l,:}^\dagger \bm{M}^T \bm{E}_{l,;} \bm{W}_k, \\
        \bm{W}_v^{\prime} &= \bm{E}_{l,:}^\dagger \bm{M}^T \bm{E}_{l,;} \bm{W}_v.
    \end{aligned}
\end{equation}
This construction ensures the validity of \Cref{eq:attention-equality}.
In this context, $\bm{E}_{l,:}^\dagger$ indicates the \textit{Moore–Penrose pseudoinverse}~\citep{moore1920reciprocal,bjerhammar1951application,penrose1955generalized} of $\bm{E}_{l,:}$.

Building upon the previous discussions~(\ding{182},\ding{183},\ding{184}), we demonstrate the existence of a parameter set:
$$\{\bm{W}_q', \bm{W}_k', \bm{W}_v', \bm{W}_1', \bm{W}_2', b_1', b_2'\}$$
such that:
\begin{equation}
    \begin{aligned}
        & \bm{W}_2^T \left( \sigma \left( \bm{W}_1^T \left( \mathrm{Softmax} \left( \bm{E}_{t,0} \bm{W}_q \bm{W}_k^T  
        \left [ \begin{matrix}
            \bm{E}_{l,:} \\
            \bm{E}_{t,0} \\
        \end{matrix} \right ]^T
        \right) 
        \left [ \begin{matrix}
            \bm{E}_{l,:} \\
            \bm{E}_{t,0} \\
        \end{matrix} \right ]
        \bm{W}_v \right) + b_1\right) \right) + b_2 = \\
        & \;\;\;\; \bm{W}_2^{\prime T} \left( \sigma \left( \bm{W}_1^{\prime T} \left( \mathrm{Softmax} \left( \bm{E}_{l,-1:} \bm{W}_q^\prime \bm{W}_k^{\prime T} \bm{E}_{l,:}^T \right)  \bm{E}_{l,:} \bm{W}_v^\prime \right) + b_1^\prime \right) \right) + b_2^\prime,
    \end{aligned}
\end{equation}
which proves the \Cref{pro:one-step-pseudo-update}.
And this parameter set may not be the only viable option. For example, according to the universal approximation theorem~\citep{cybenko1989approximations,Hornik91}, a feed-forward network can be utilized to address differences in attention computations and provide a greater degree of freedom for $\bm{W}^\prime_q$, $\bm{W}^\prime_k$, and $\bm{W}^\prime_v$.

\end{proof}

\section{Implementation Details of Experiments}

\subsection{Implementation Details of Visualization of Pseudo Gradient Update} \label{sec:demonstrated-questions}

\paragraph{Data Preparation. } We select four questions from AIME2024 as follows:
\begin{promptbox}{Details of $q_0$}
    \subsection*{Question}
    Every morning Aya goes for a $9$-kilometer-long walk and stops at a coffee shop afterwards. When she walks at a constant speed of $s$ kilometers per hour, the walk takes her 4 hours, including $t$ minutes spent in the coffee shop. When she walks $s+2$ kilometers per hour, the walk takes her 2 hours and 24 minutes, including $t$ minutes spent in the coffee shop. Suppose Aya walks at $s+\frac{1}{2}$ kilometers per hour. Find the number of minutes the walk takes her, including the $t$ minutes spent in the coffee shop.
    \vspace{20pt}

    \subsection*{Answer}
    204
\end{promptbox}

\begin{promptbox}{Details of $q_1$}
    \subsection*{Question}
    There exist real numbers $x$ and $y$, both greater than 1, such that $\log_x\left(y^x\right)=\log_y\left(x^{4y}\right)=10$. Find $xy$.
    \vspace{15pt}

    \subsection*{Answer}
    025
\end{promptbox}

\begin{promptbox}{Details of $q_2$}
    \subsection*{Question}
    Find the largest possible real part of \[(75+117i)z+\frac{96+144i}{z}\]where $z$ is a complex number with $|z|=4$.
    \vspace{20pt}

    \subsection*{Answer}
    540
\end{promptbox}

\begin{promptbox}{Details of $q_3$}
    \subsection*{Question}
    Let $\triangle ABC$ have circumcenter $O$ and incenter $I$ with $\overline{IA}\perp\overline{OI}$, circumradius $13$, and inradius $6$. Find $AB\cdot AC$.
    \vspace{20pt}

    \subsection*{Answer}
    468
\end{promptbox}

\paragraph{Visualization of Pseudo Gradient Update. } We leverage QwQ-32B~\citep{qwq} to generate trajectories for these four questions.
Then for each trajectory, we calculate the negative log-probability of \texttt{Final Answer$\backslash$n$\backslash$boxed{..answer..}} at each position.
\Cref{algo:empirical_eamples_of_pseudo_gradient_update} outlines the overall process.

\IncMargin{1em}
\begin{algorithm}[h]
    \KwIn{$\mathcal{M}$: QwQ-32B, $I$: instruction, $q$: question from AIME2024, $t$: the trajectory generated by QwQ-32B, $a$: the answer sequence, i.e., \texttt{Final Answer$\backslash$n$\backslash$boxed{...answer...}}, $s$: step size.}
    \BlankLine
    \For{i $\in$ $[0, \vert t \vert)$} {
        Obtain input by $I \oplus q \oplus t_{:i} \oplus a$ \;
        Feed input to $M$ and get logits $l_{a}$ of answer sequence \;
        Compute the negative log-probability using $l_a$ \;
    }
    \caption{Computation of Empirical Examples of Pseudo Gradient Update} \label{algo:empirical_eamples_of_pseudo_gradient_update}
\end{algorithm}

\paragraph{Visualization of Landscape. } 
We refer to the methodology proposed by Li et al.~\citep{Li0TSG18}.
Assuming the set parameters of QwQ-32B is denoted by $\{\theta_k \}$ (excluding the embedding matrix), we randomly select two vectors, $\{\theta_{1,k}\}$ and $\{ \theta_{2,k} \}$, for each parameter. 
We then edit the parameters by adding $\alpha_1 \theta_{1,k} + \alpha_2 \theta_{2,k}$ and compute the negative log-probability given only the instruction and question to form the point set $\{(\alpha_1, \alpha_2, \mathcal{\widehat{L}}_{\alpha_1, \alpha_2})\}$. 
Finally, we visualize this point set to reveal the landscape.
The overall process is summarized as \Cref{algo:landscape}.

\IncMargin{1em}
\begin{algorithm}[h]
    \KwIn{$\mathcal{M}$: QwQ-32B, $I$: instruction, $q$: question from AIME2024, $a$: the answer sequence, i.e., \texttt{Final Answer$\backslash$n$\backslash$boxed{..answer..}}.}
    \BlankLine
    Obtain random vectors $\{ \theta_{1,k}\}$, $\{ \theta_{2,k} \}$ for each parameter $\theta_k$ of $\mathcal{M}$ \;
    \For{i $\in$ $[-1, 1, s]$}{
        \For{j $\in$ $[-1, 1, s]$}{
            Get edited parameters $\{\theta_{k}^\prime\}$ by adding $i \cdot \theta_{1,k} + j \cdot \theta_{2,k}$ \;
            Obtain input by $I \oplus q \oplus a$ \;
            Feed input to $M$ and get logits $l_{a}$ of answer sequence \;
            Compute the negative log-probability using $l_a$ \;
        }
    }
    \caption{Computation of Landscape} \label{algo:landscape}
\end{algorithm}

\paragraph{Project the Pseudo Gradient Update to Landscape. }
To project the trajectory of the pseudo-gradient update onto the landscape, we fix one direction corresponding to the time dimension and identify the closest contour to the corresponding $\mathcal{\widehat{L}}$ to determine another direction.

\subsection{Training Details} \label{app:training-details}

\paragraph{Dataset Processing. }
To maintain the validity and verifiability of the question set, we clean and filter the original dataset. 
Initially, we exclude incomplete questions as well as those lacking answers. 
Subsequently, we remove questions requiring reasoning with images or other external information. 
To further ensure verifiability, we employ Math-Verify~\footnote{\url{https://github.com/huggingface/Math-Verify}} to examine each question and exclude those that could not be verified. Finally, we eliminate irrelevant characters, such as URLs and HTML tags, resulting in approximately $39$k questions with corresponding answers.

\paragraph{Training of SFT. } 
We first synthesize training trajectories from Qwen2.5-Math-72B-Instruct~\citep{abs-2409-12122} and DeepSeek-R1-Distill-Qwen-14B~\citep{abs-2501-12948}. 
From the entire question set, we sample $10$k questions and use the sampling parameters shown in~\Cref{tab:syn-gen-param} to generate reasoning trajectories with the prompt, \textit{Please solve the following mathematical problem step by step and put your final answer in \textbackslash boxed{}}, resulting in $640$k trajectories.
\begin{table}[hbt]
    \centering
    \caption{Sampling parameters leveraging for reasoning trajectories synthesis.} \label{tab:syn-gen-param}
    \begin{tabular}{l|cc}
        \toprule
        & \textbf{Qwen2.5-Math-72B-Instruct} & \textbf{DeepSeek-R1-Distill-Qwen-14B} \\
        \midrule
        Temperature & $0.7$ & $0.7$ \\
        Top-$p$ & $1.0$ & $1.0$ \\
        Top-$k$ & $50$ & $50$ \\
        Max Tokens & $8192$ & $36784$ \\
        Rollout Number & $64$ & $64$ \\
        \bottomrule
    \end{tabular}
\end{table}
We then filter out trajectories with incorrect answers, retaining approximately $\sim470$k for Qwen2.5-Math-72B-Instruct and approximately $\sim550$k for DeepSeek-R1-Distill-Qwen-14B. 
During training, we utilize the parameters listed in \Cref{tab:sft-param}.

\begin{table}[t]
    \centering
    \begin{minipage}[t]{0.45\textwidth}
        \centering
        \captionof{table}{Training parameters for SFT.} \label{tab:sft-param}
        \begin{tabular}{l|c}
            \toprule
            & \textbf{Parameter} \\
            \midrule
            Max Response Length & $18432$ \\
            Train Batch Size & $256$ \\
            Learning Rate & $1e\text{-}5$ \\
            Total Epochs & $1$ \\
            \bottomrule
        \end{tabular}
    \end{minipage}
    \begin{minipage}[t]{0.45\textwidth}
        \centering
        \captionof{table}{Training parameters for GRPO.} \label{tab:grpo-param}
        \begin{tabular}{l|c}
            \toprule
            & \textbf{Parameter} \\
            \midrule
            Max Prompt Length & $1024$ \\
            Max Response Length & $16384$ \\
            Rollout Temperature & $1.0$ \\
            Rollout Number & $16$ \\
            Train Batch Size & $1024$ \\
            Learning Rate & $1e\text{-}6$ \\
            Total Epochs & $1$ \\
            \bottomrule
        \end{tabular}
    \end{minipage}
\end{table}


\begin{promptbox}{System Prompt of GRPO}
    A conversation between a User and an Assistant. The User poses a question, and the Assistant provides a solution. The Assistant's response follows these structured steps: 

    \vspace{5pt}
    1. \textbf{Reasoning Process}: The Assistant reflects on the problem using a reasoning process enclosed within <think> and </think> tags.
    
    2. \textbf{Conclusion}: The Assistant reaches a conclusion, which is enclosed within <conclusion> and </conclusion> tags. The final answer is highlighted within \textbackslash boxed{...final answer...}.
    
    3. \textbf{Answer Format}: The complete response should be formatted as:

    <think>
    
    ...reasoning process...
    
    </think>
    
    <conclusion>
    
    ...conclusion...
    
    The answer is \textbackslash boxed{...final answer...}
    
    </conclusion>
    
\end{promptbox}
\paragraph{Training of GRPO. } 
For the GRPO training, we use the complete question set and apply the parameters listed in \Cref{tab:grpo-param}. 
We adhere to the DeepSeek-R1-style system prompt, as presented in the \textit{System Prompt of GRPO} box.
And for the reward design, we assign the trajectory with the correct answer and correct format the score $1$, the trajectory with the false answer and correct format the score $0.0$, trajectory with the correct answer and false format the score $-0.5$, and trajectory with the false answer and false format the score $-1$, formally:
\begin{equation}
    R(y^\prime, y) = \left\{
    \begin{aligned}
        &1 & \mathrm{answer\_match} (y ^\prime, y) \quad &\mathrm{and}\quad \mathrm{format\_correct}(y^\prime), \\
        &0 & \neg \mathrm{answer\_match} (y ^\prime, y) \quad &\mathrm{and}\quad \mathrm{format\_correct}(y^\prime), \\
        &-0.5 & \mathrm{answer\_match} (y ^\prime, y) \quad &\mathrm{and}\quad \neg \mathrm{format\_correct}(y^\prime), \\
        &-1 & \neg \mathrm{answer\_match} (y ^\prime, y) \quad &\mathrm{and}\quad \neg\mathrm{format\_correct}(y^\prime), \\
    \end{aligned}  
    \right.
\end{equation}
where $y$ indicates the ground-truth and $y^\prime$ indicates the trajectory.
We employ the Math-Verify package to ascertain the equivalence of $y$ and $y'$.

\paragraph{Details of Hardware and Software. }
All the training tasks are conducted based on veRL~\citep{ShengZYWZZPL025}, cooperated with Pytorch~\citep{PaszkeGMLBCKLGA19} 2.6.0, Transformers~\citep{WolfDSCDMCRLFDS20} 4.51.3, vLLM~\citep{KwonLZ0ZY0ZS23} 0.8.4. 
We conduct all experiments on clusters equipped with NVIDIA A800 GPUs and Intel(R) Xeon(R) Platinum 8336C CPUs.

\subsection{Evaluation Details} \label{app:evaluation-details}

\paragraph{Benchmarks. }
The following details describe our evaluation benchmarks:
\begin{itemize}[leftmargin=5mm]
    \item \textbf{AIME24.} AIME24\footnote{\url{https://huggingface.co/datasets/AI-MO/aimo-validation-aime}} consists of $30$ challenging questions from the 2024 American Invitational Mathematics Examination (AIME).
    \item \textbf{MATH500.} The original MATH dataset \citep{HendrycksBKABTS21} comprises $12,500$ problems from American high school mathematics competitions. MATH500 \citep{LightmanKBEBLLS24}, a widely used subset of its test split, includes only Level 5 questions in this study.
    \item \textbf{LiveMathBench.} LiveMathBench \citep{abs-2412-13147} is a continuously updated dataset of challenging mathematical problems. We use the December 2024 hard split, comprising $45$ questions in English and Chinese.
    \item \textbf{GPQA.} GPQA~\citep{abs-2311-12022} dataset is a challenging, professional multiple-choice science question-answering dataset. We use its diamond subset, comprising $198$ questions.
    \item \textbf{LiveCodeBench.} LiveCodeBench~\citep{abs-2403-07974} is a benchmark designed for a comprehensive and uncontaminated evaluation of the code-related capabilities of LLMs. It incorporates questions from LeetCode, AtCoder, and Codeforces.
\end{itemize}

\paragraph{Metrics. }
We use Pass@$k$ and mG-Pass@$k$~\citep{abs-2412-13147} as evaluation metrics. 
We generate $n$ responses for each question and assume the number of correct responses is $c$.
Then the metrics are computed as:
\begin{itemize}[leftmargin=5mm]
    \item \textbf{Pass@$k$. }
    \begin{equation}
        \text{Pass@}k = \mathbb{E}_{\text{questions}} \left[ 1 - \frac{{{n - c} \choose k}}{{n \choose k}} \right].
    \end{equation}
    \item \textbf{mG-Pass@$k$. }
    \begin{equation}
        \text{mG-Pass@}k = \mathbb{E}_{\text{questions}} \left[ \frac{2}{k} \sum_{i= \lceil k / 2 \rceil + 1}^{k}  \sum_{j=i}^c \frac{{c \choose j} \cdot {n - c \choose k - j}}{{n \choose k}}  \right].
    \end{equation}
\end{itemize}

\section{More Discussions on Recent LLM Reasoning Progress} \label{app:disscussion-recent-reasoning-progress}

In this section, we focus on recent research developments and discuss the essential improvements they implemented to enhance performance within our framework.
We involve the following representative works: OpenThoughts~\citep{openthoughts}, Light-R1~\citep{abs-2503-10460}, Open-Reason-Zero~\citep{hu2025open}, DAPO~\citep{abs-2503-14476}, VAPO\citep{abs-2504-05118}, GPG~\citep{chu2025gpg}, Llama Nemotron~\citep{abs-2505-00949}.

\paragraph{Data Filtering. } 
Works such as Light-R1~\citep{abs-2503-10460} use strategies like diversity and difficulty filtering to obtain high-quality data. From a meta-learning perspective, this approach can be seen as a sample mining strategy, optimizing the distribution of training task sets to enhance the efficiency of model training.

\paragraph{Synthetic Data From Strong Reasoning LLMs. } 
Works such as OpenThoughts~\citep{openthoughts} and Llama Nemotron~\citep{abs-2505-00949} utilize a more advanced reasoning LLM, such as DeepSeek-R1, to generate multiple trajectories for each training question, resulting in training data for SFT. 
This approach effectively expands the size of the support set to stabilize inner loop optimization, thereby achieving improved results.
On the other hand, this is equivalent to distilling the optimization path from the already trained model~(strong reasoning LLMs) to the small model.

\paragraph{Clip Higher for Clipper Surrogate Loss of RL. }
DAPO~\citep{abs-2503-14476} proposes using a higher clipping range to promote exploration during the GRPO training process. 
Similarly, removing the KL penalty term and entropy loss in GRPO can achieve the same effect. 
From an optimization perspective, these improvements expand the exploration space of the optimization path, facilitating the model's ability to explore extreme points. 
Increasing the diversity of training data in supervised fine-tuning also contributes to this effect.

\paragraph{Dynamic Sampling During RL. }
Recent studies \citep{chu2025gpg,abs-2503-14476} aim to balance the ratio of correct to incorrect trajectories during rollout by employing dynamic sampling or introducing bias. This strategy equalizes the positive and negative gradients in the inner loop, thereby alleviating model overfitting to a particular class.

\paragraph{Group-Sampling for PPO. }
In classic reinforcement learning methodologies, algorithms typically generate only a single trajectory per problem instance. 
Recent advancements \citep{abs-2504-05118,hu2025open} have introduced group sampling in algorithms such as PPO, allowing the generation of multiple trajectories for each problem. 
From the perspective of this study, this improvement corresponds to expanding the support set, thereby enhancing inner-loop optimization.

\section{More Experimental Results} \label{app:qwen3_update}

\paragraph{QwQ's Pseudo Gradient Update. }
\Cref{fig:qwq-gradient-decent-appendix} illustrates more visualizations of the \textit{pseudo-gradient update} of QwQ.

\paragraph{QwQ3's Pseudo Gradient Update in Thinking/Nothinking Mode. }
\Cref{fig:qwq-nothinking-gradient-decent-appendix} illustrates more visualizations of the \textit{pseudo-gradient update} of QwQ in thinking and nothinking modes.

\paragraph{Qwen3's Pseudo Gradient Update. }
Following the methodology described in \Cref{sec:reasoning-trajectories-as-parameters-update} and \Cref{sec:demonstrated-questions}, we monitor the pseudo-gradient update of Qwen3-32B \citep{qwen3} under thinking mode, as illustrated in \Cref{fig:qwen_gradient_decent}. 
We similarly observe that the reasoning trajectories of Qwen3 exhibit a parameter update effect. 

\paragraph{Qwen3's Pseudo Gradient Update in Thinking/Nothinking Mode. }
Referring to the settings in \Cref{sec:length-of-rt}, we examine the differences between thinking mode and no-thinking mode, as shown in \Cref{fig:qwen_nothinking_gradient_decent}. It is evident that, due to the specific optimization of Qwen3, its no-thinking token delimiter~(i.e., \texttt{</think>}) demonstrates a more pronounced gradient descent effect. The delimiter \texttt{</think>} enables the model to swiftly update to an extreme point in the appropriate direction with a larger step size. However, this update is susceptible to falling into local minima, which accounts for the performance gap between Qwen3's no-thinking mode and thinking mode.

\paragraph{Pseudo Gradient Update of False Reasoning Trajectories. }
\Cref{fig:qwen_nothinking_false_gradient_decent} illustrates the curve of pseudo gradient updates associated with incorrect reasoning trajectories. It is evident that the curve representing these trajectories does not show a downward trend, underscoring the strong connection between reasoning paths and optimization processes.

\begin{figure} \ContinuedFloat*
    \centering
    \begin{subfigure}{\textwidth}
        \centering
        \includegraphics[width=.9\textwidth]{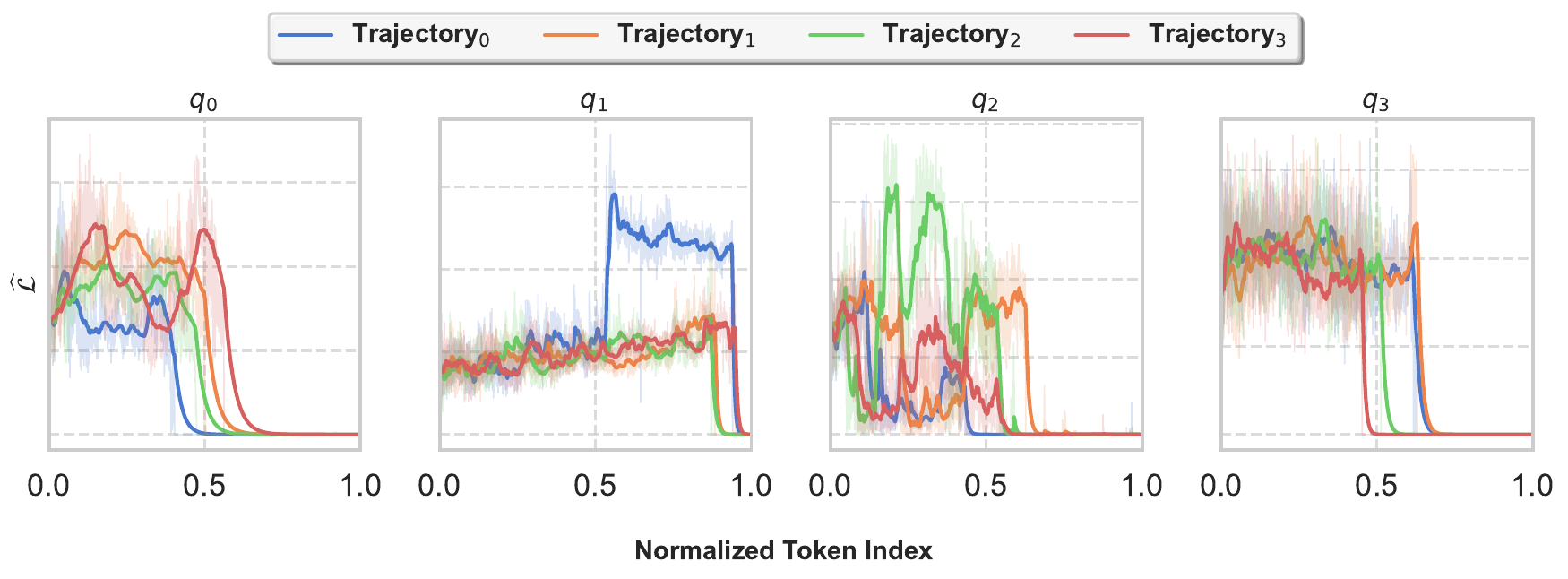}
    \end{subfigure}
    \begin{subfigure}{\textwidth}
        \centering
        \includegraphics[width=.9\textwidth]{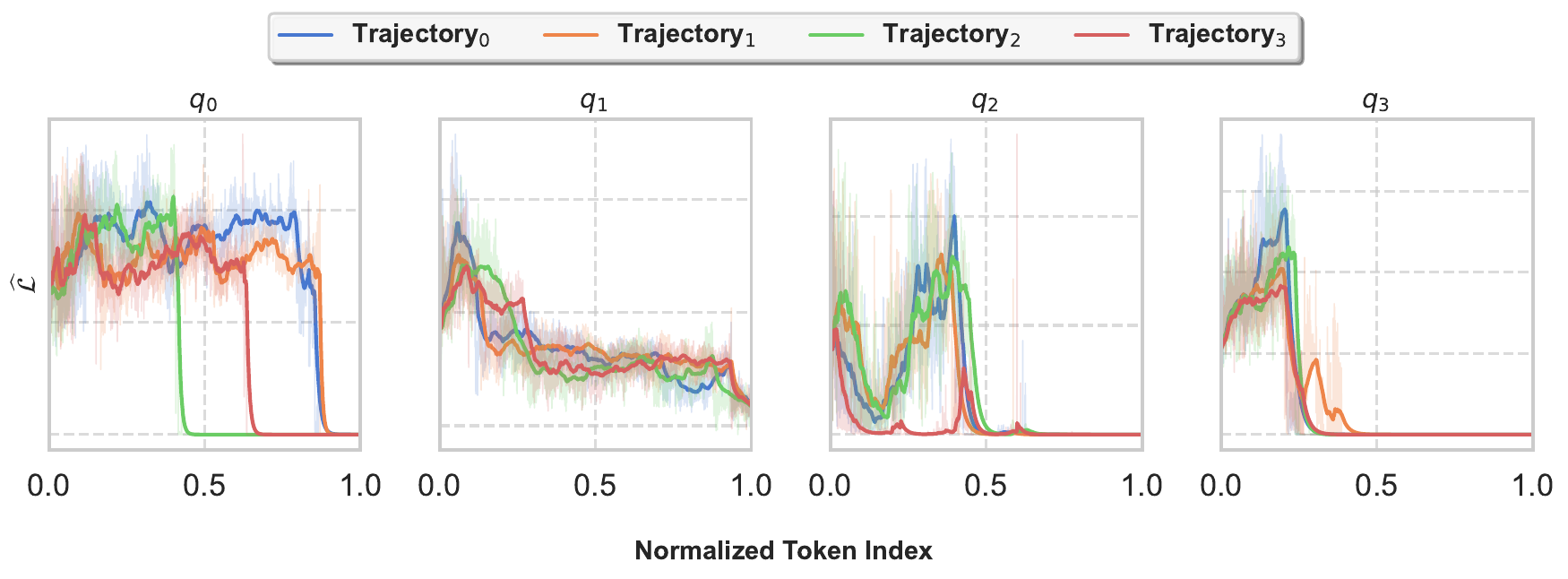}
    \end{subfigure}
    \begin{subfigure}{\textwidth}
        \centering
        \includegraphics[width=.9\textwidth]{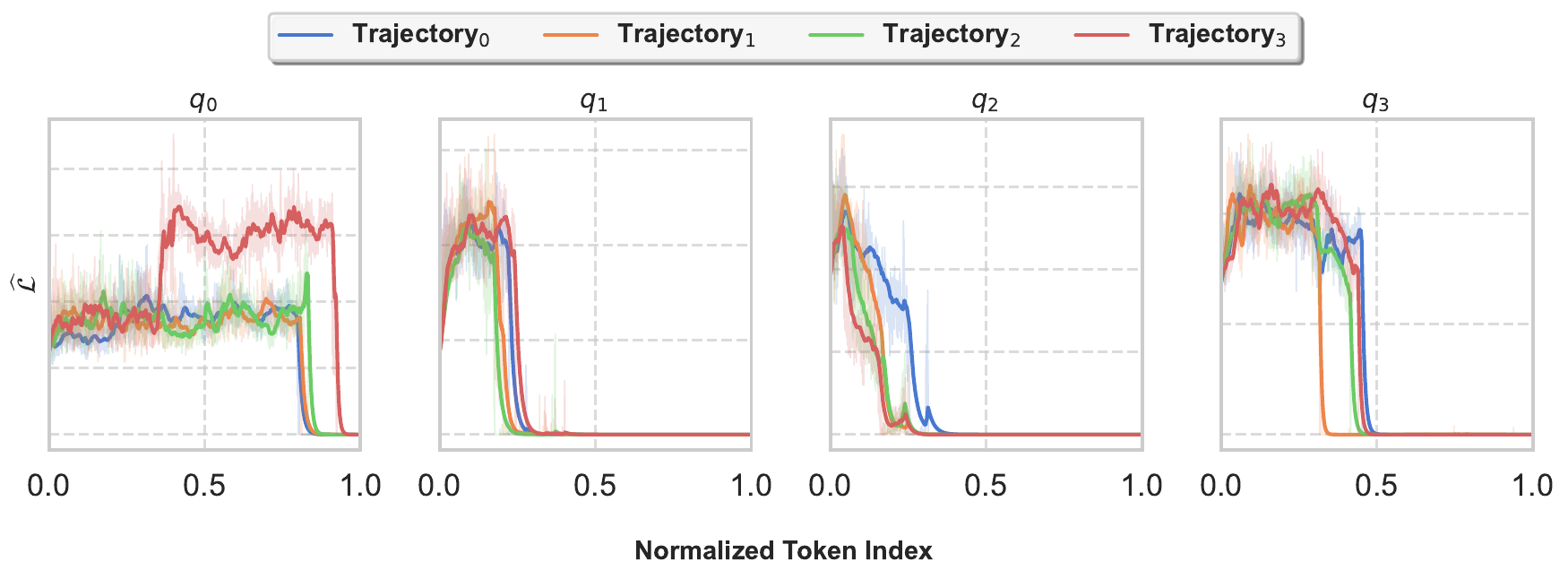}
    \end{subfigure}
    \begin{subfigure}{\textwidth}
        \centering
        \includegraphics[width=.9\textwidth]{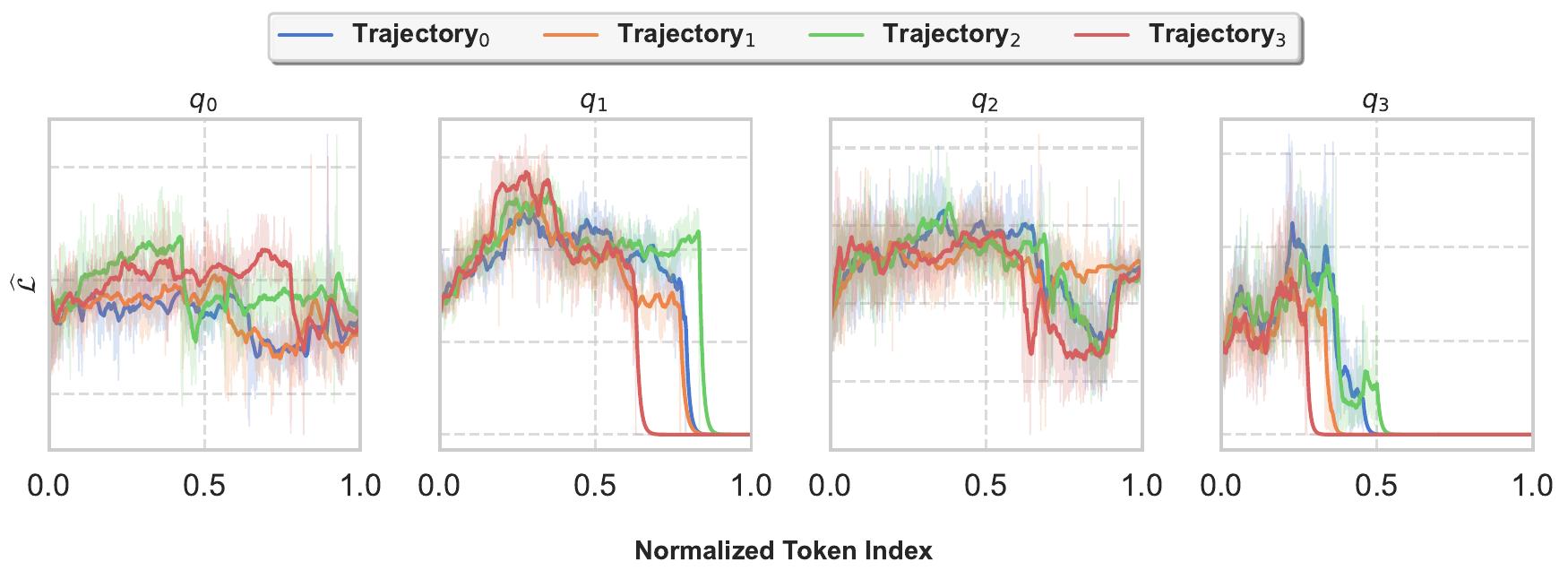}
    \end{subfigure}
    \caption{More visualizations of the \textit{pseudo-gradient update} of QwQ.}
\end{figure}

\begin{figure} \ContinuedFloat
    \begin{subfigure}{\textwidth}
        \centering
        \includegraphics[width=.9\textwidth]{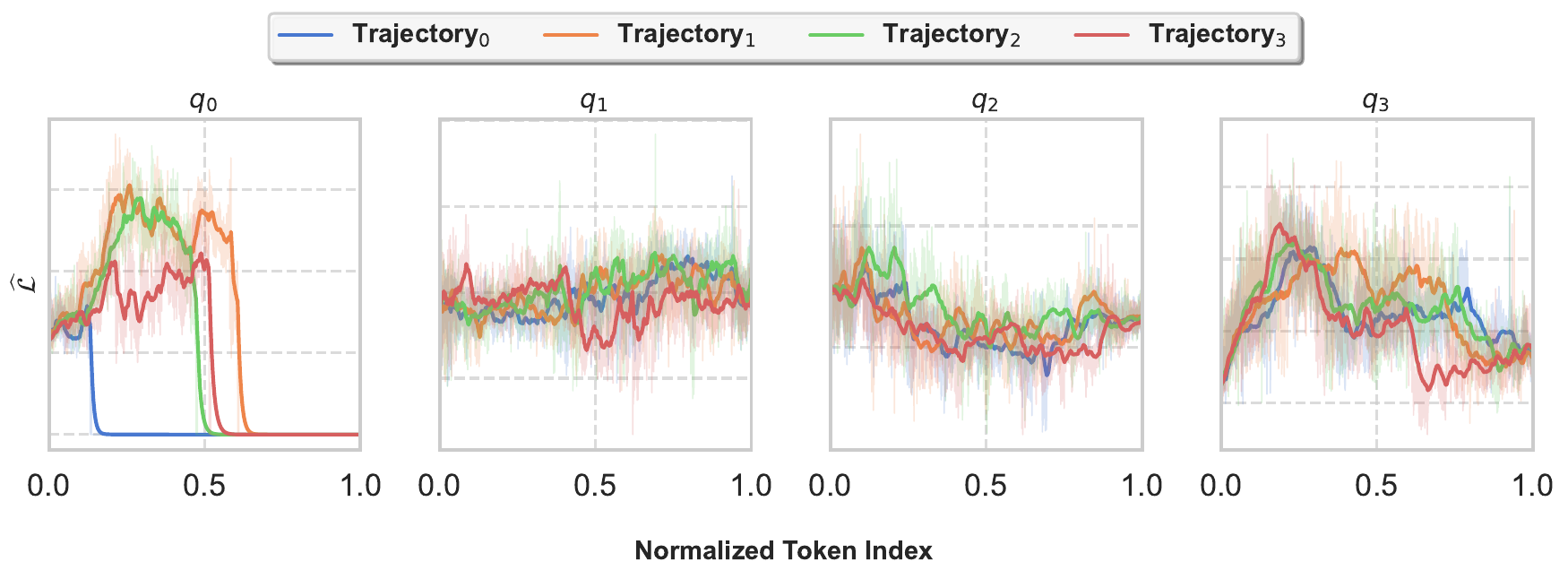}
    \end{subfigure}
    \caption{More visualizations of the \textit{pseudo-gradient update} of QwQ.} \label{fig:qwq-gradient-decent-appendix}
\end{figure}

\begin{figure} \ContinuedFloat*
    \centering
    \begin{subfigure}{\textwidth}
        \centering
        \includegraphics[width=.9\textwidth]{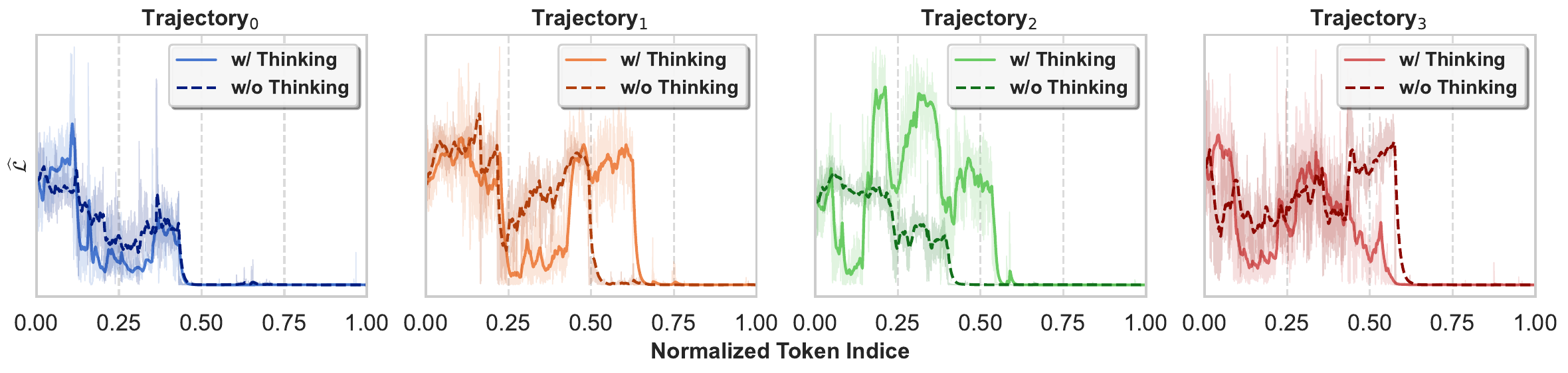}
    \end{subfigure}
    \begin{subfigure}{\textwidth}
        \centering
        \includegraphics[width=.9\textwidth]{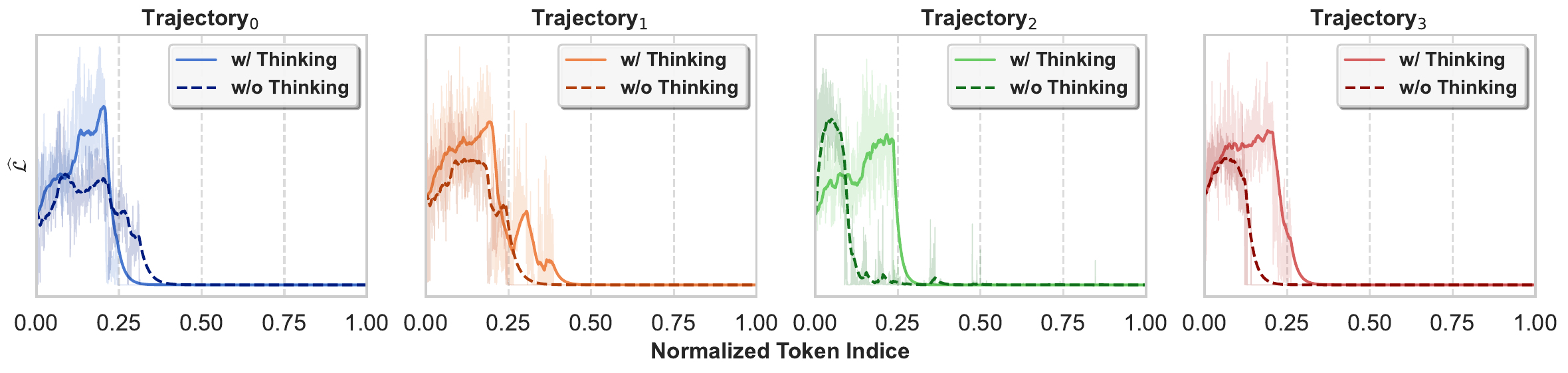}
    \end{subfigure}
    \begin{subfigure}{\textwidth}
        \centering
        \includegraphics[width=.9\textwidth]{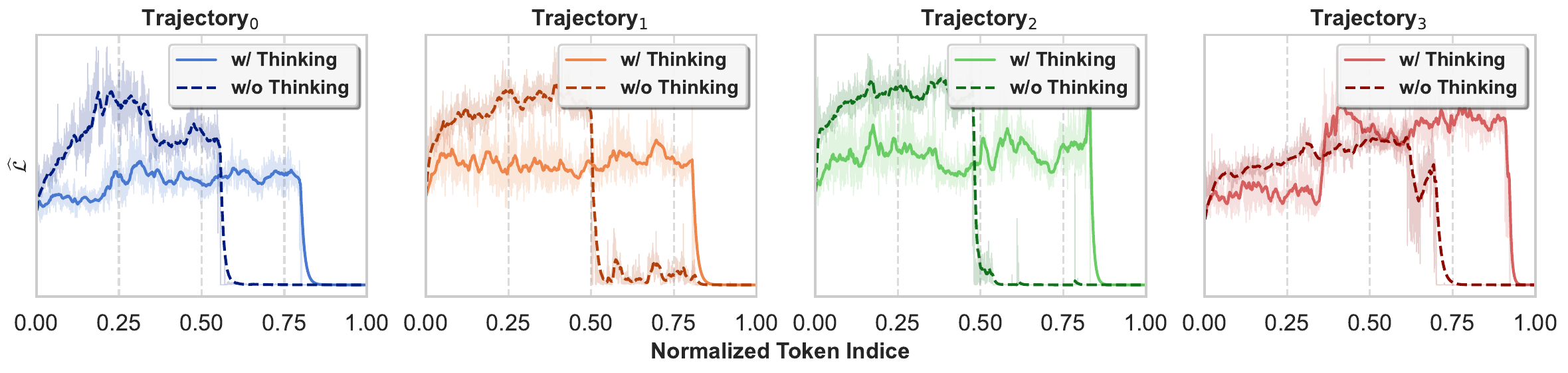}
    \end{subfigure}
    \begin{subfigure}{\textwidth}
        \centering
        \includegraphics[width=.9\textwidth]{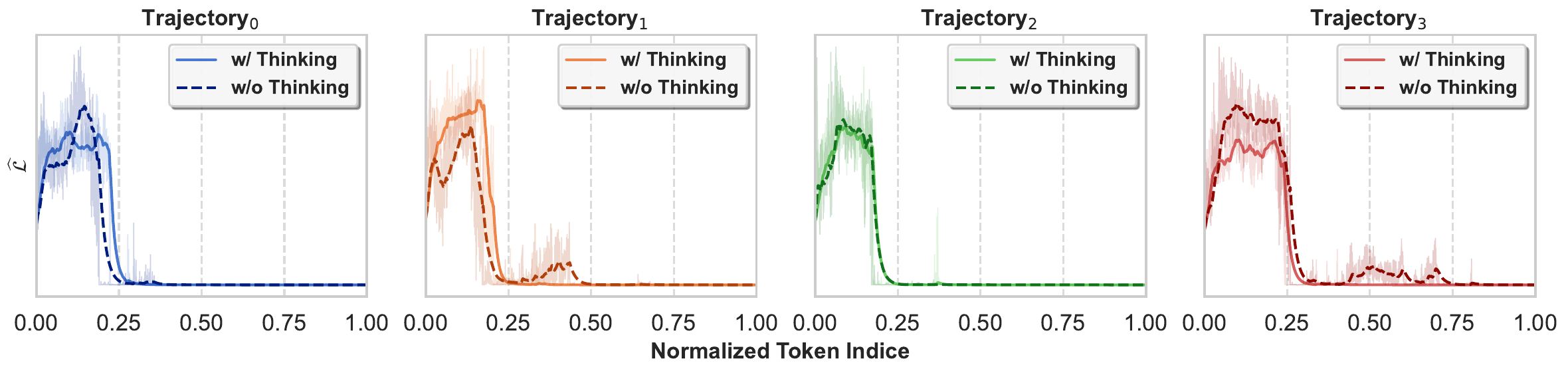}
    \end{subfigure}
    \caption{More visualizations of QwQ's \textit{pseudo-gradient update} for both thinking and no-thinking modes.} \label{fig:qwq-nothinking-gradient-decent-appendix}
\end{figure}

\begin{figure} \ContinuedFloat
    \centering
    \begin{subfigure}{\textwidth}
        \centering
        \includegraphics[width=.9\textwidth]{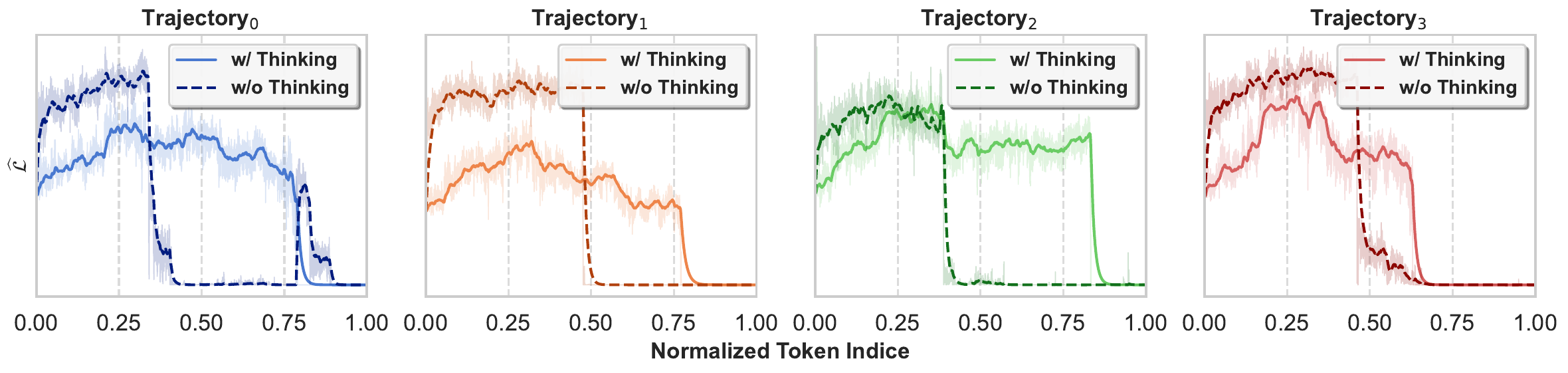}
    \end{subfigure}
    \begin{subfigure}{\textwidth}
        \centering
        \includegraphics[width=.9\textwidth]{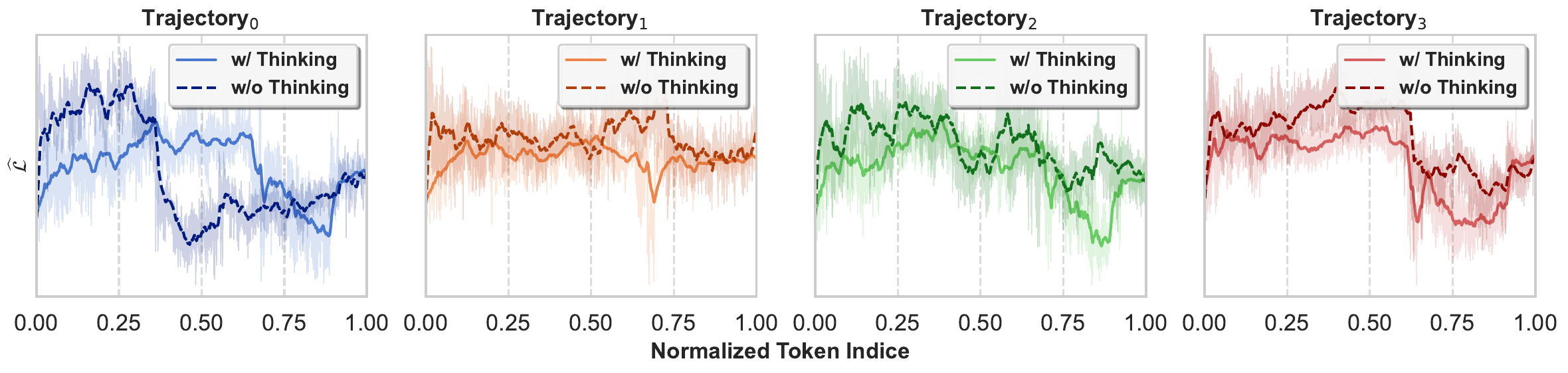}
    \end{subfigure}
    \begin{subfigure}{\textwidth}
        \centering
        \includegraphics[width=.9\textwidth]{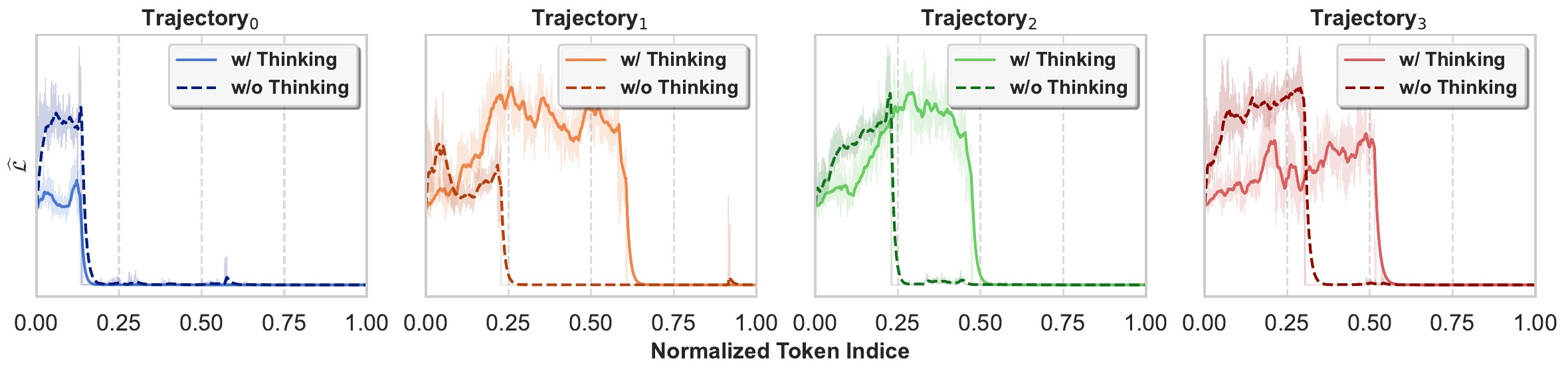}
    \end{subfigure}
    \begin{subfigure}{\textwidth}
        \centering
        \includegraphics[width=.9\textwidth]{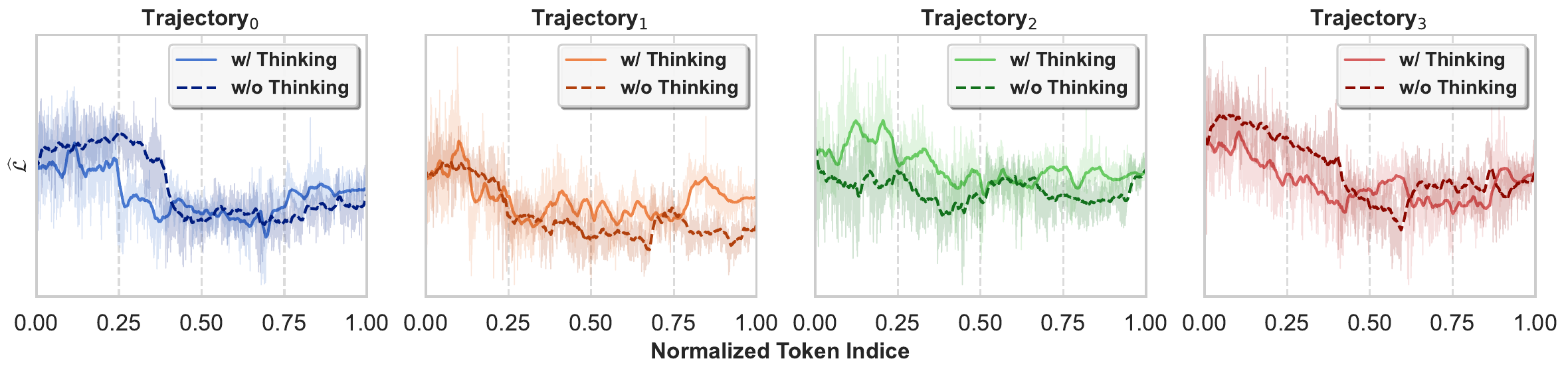}
    \end{subfigure}
    \caption{More visualizations of QwQ's \textit{pseudo-gradient update} for both thinking and no-thinking modes.}
\end{figure}


\begin{figure} \ContinuedFloat*
    \centering
    \begin{subfigure}{\textwidth}
        \centering
        \includegraphics[width=.9\textwidth]{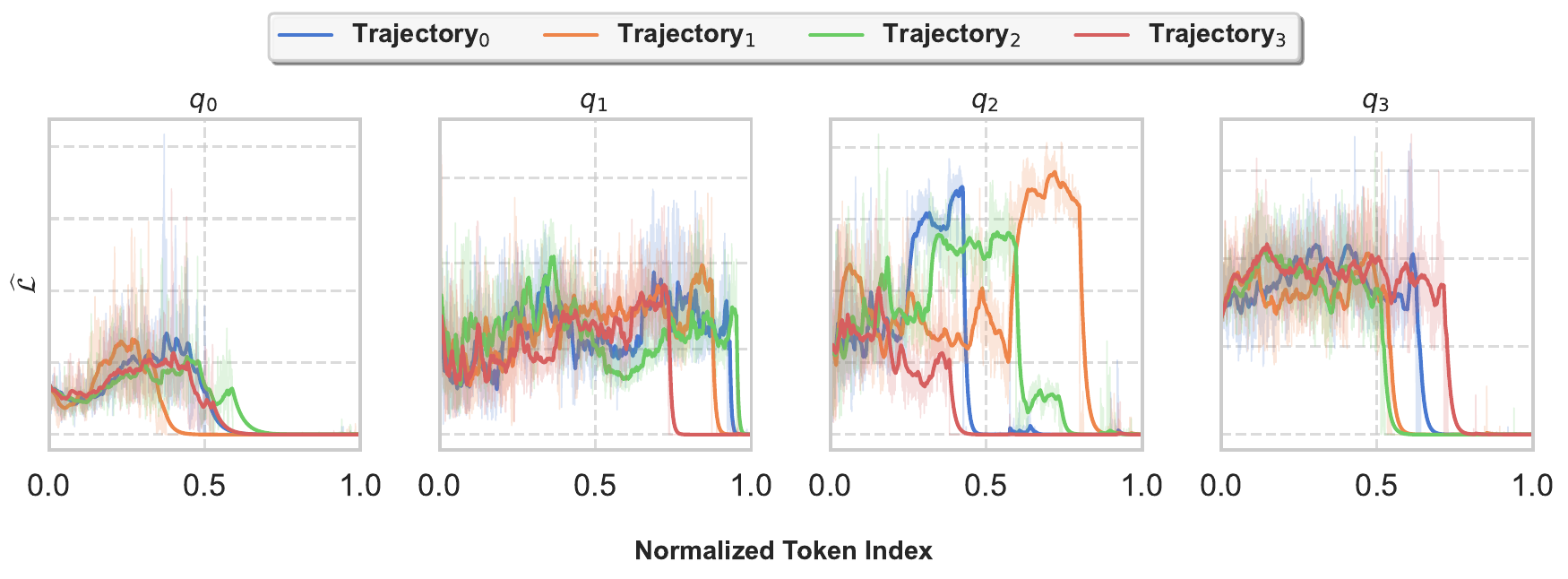}
    \end{subfigure}
    \begin{subfigure}{\textwidth}
        \centering
        \includegraphics[width=.9\textwidth]{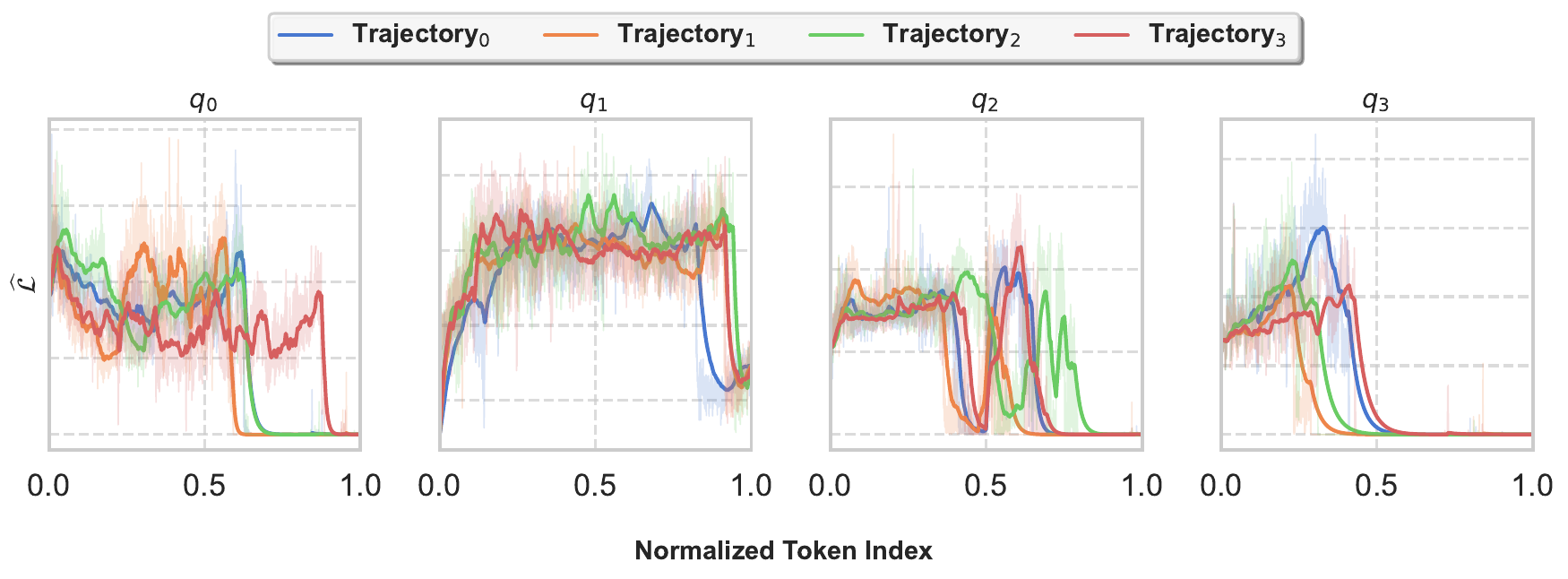}
    \end{subfigure}
    \begin{subfigure}{\textwidth}
        \centering
        \includegraphics[width=.9\textwidth]{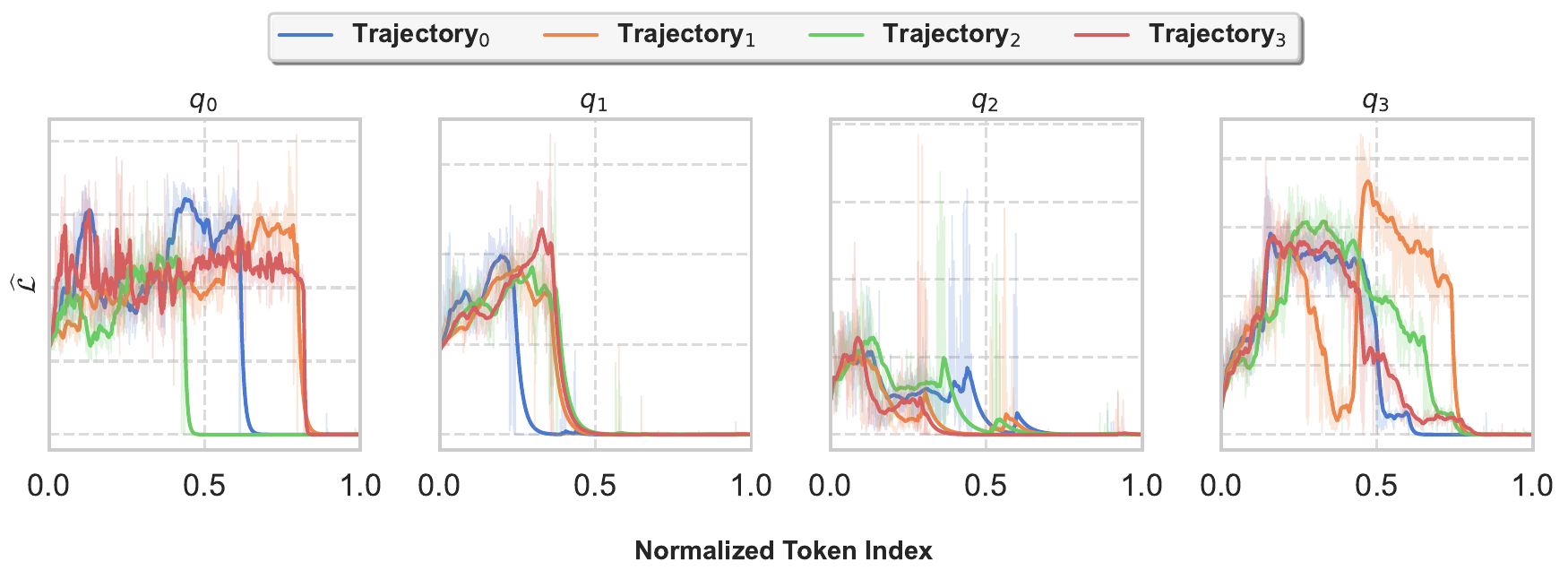}
    \end{subfigure}
    \begin{subfigure}{\textwidth}
        \centering
        \includegraphics[width=.9\textwidth]{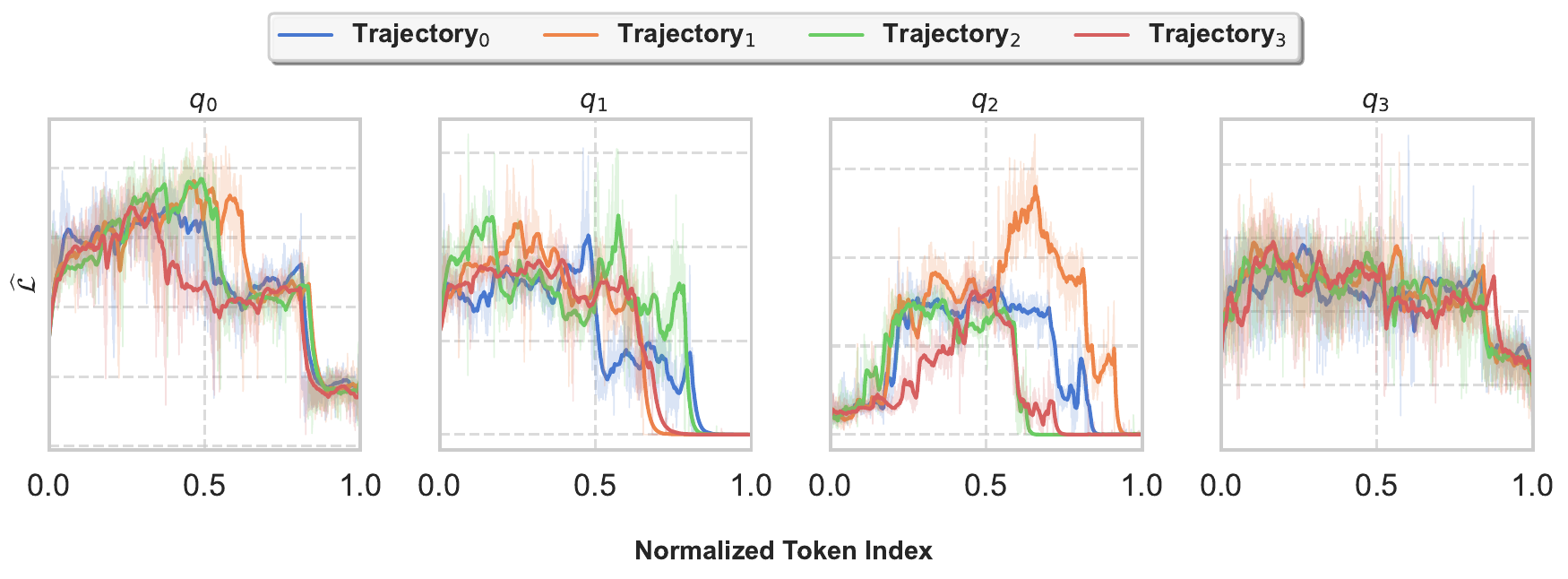}
    \end{subfigure}
    \caption{Visualizations of the \textit{pseudo-gradient update} of Qwen3.}
\end{figure}

\begin{figure} \ContinuedFloat
    \begin{subfigure}{\textwidth}
        \centering
        \includegraphics[width=.9\textwidth]{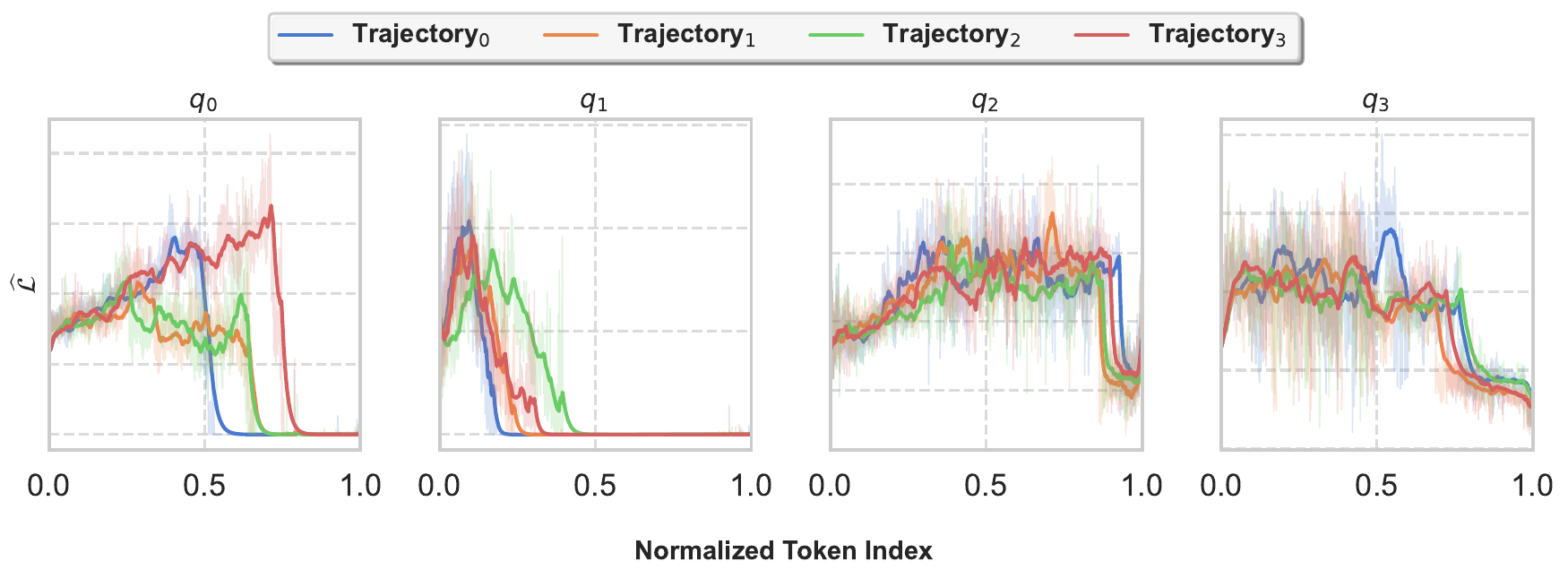}
    \end{subfigure}
    \caption{Visualizations of the \textit{pseudo-gradient update} of Qwen3.} \label{fig:qwen_gradient_decent}
\end{figure}


\begin{figure}
    \centering
    \begin{subfigure}{\textwidth}
        \centering
        \includegraphics[width=.9\textwidth]{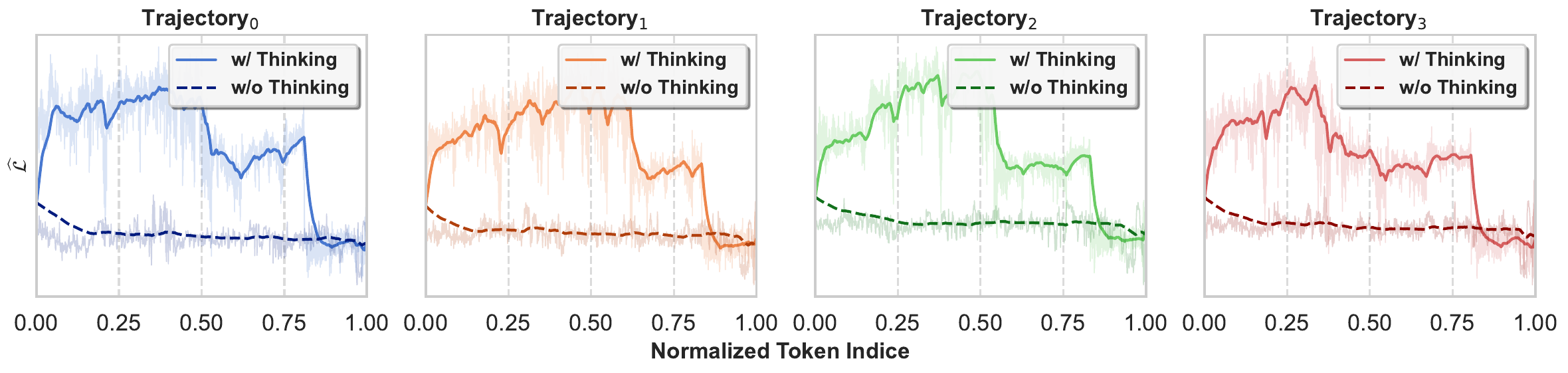}
    \end{subfigure}
    \begin{subfigure}{\textwidth}
        \centering
        \includegraphics[width=.9\textwidth]{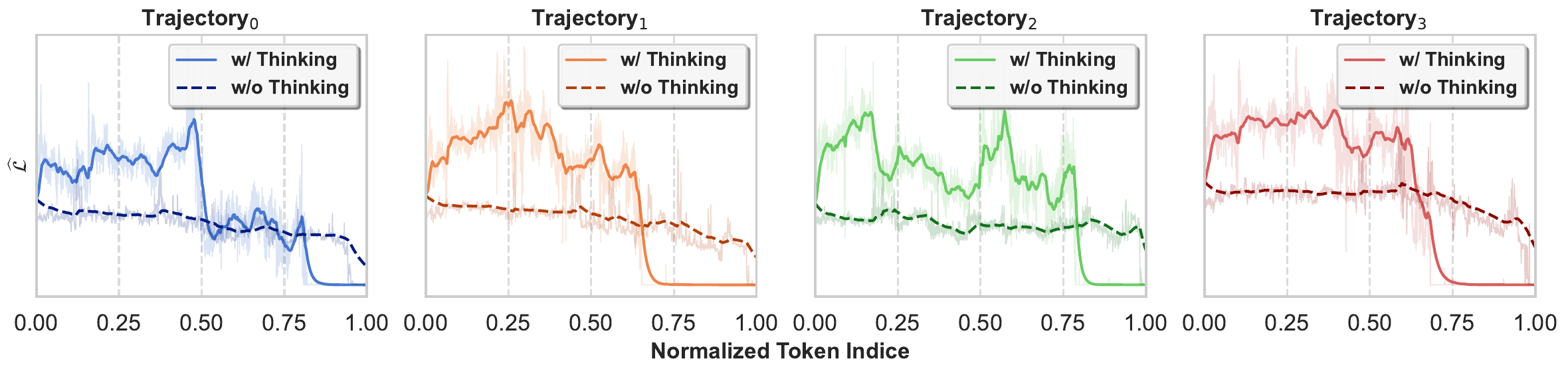}
    \end{subfigure}
    \begin{subfigure}{\textwidth}
        \centering
        \includegraphics[width=.9\textwidth]{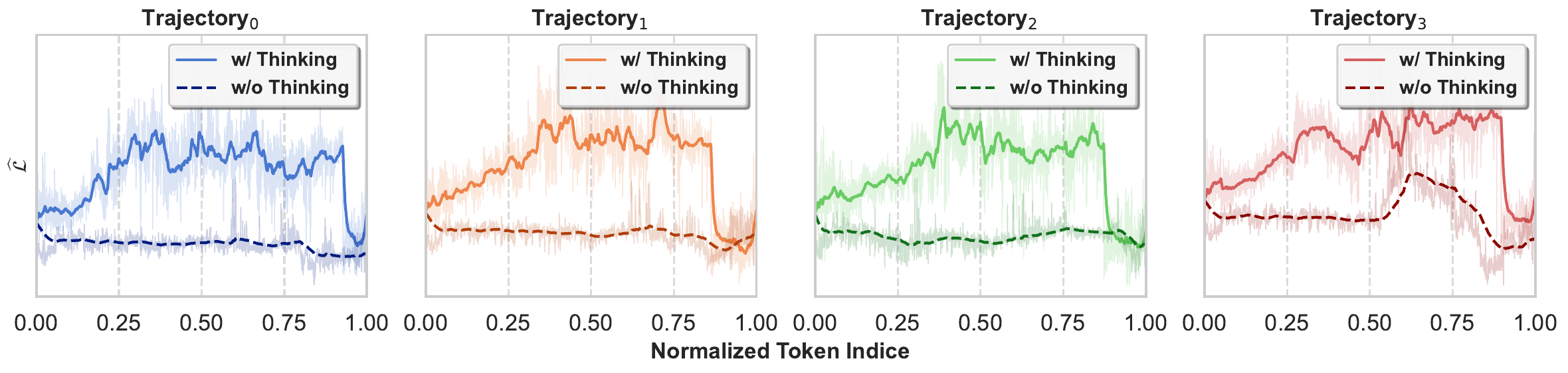}
    \end{subfigure}
    \begin{subfigure}{\textwidth}
        \centering
        \includegraphics[width=.9\textwidth]{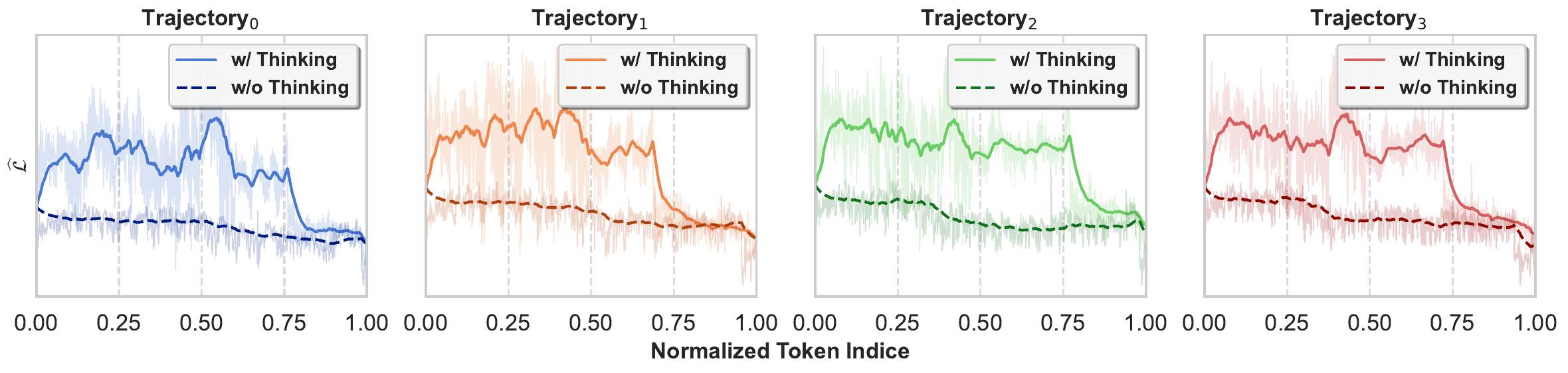}
    \end{subfigure}
    \begin{subfigure}{\textwidth}
        \centering
        \includegraphics[width=.9\textwidth]{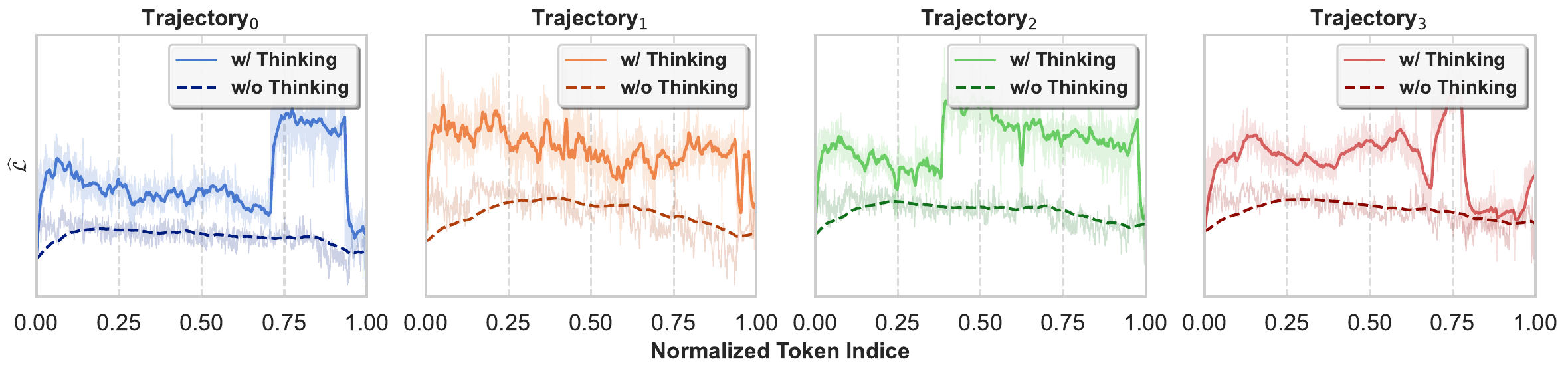}
    \end{subfigure}
    \caption{Visualizations of Qwen3's \textit{pseudo-gradient update} for both thinking and no-thinking modes.} \label{fig:qwen_nothinking_gradient_decent}
\end{figure}


\begin{figure} \ContinuedFloat*
    \centering
    \begin{subfigure}{\textwidth}
        \centering
        \includegraphics[width=.9\textwidth]{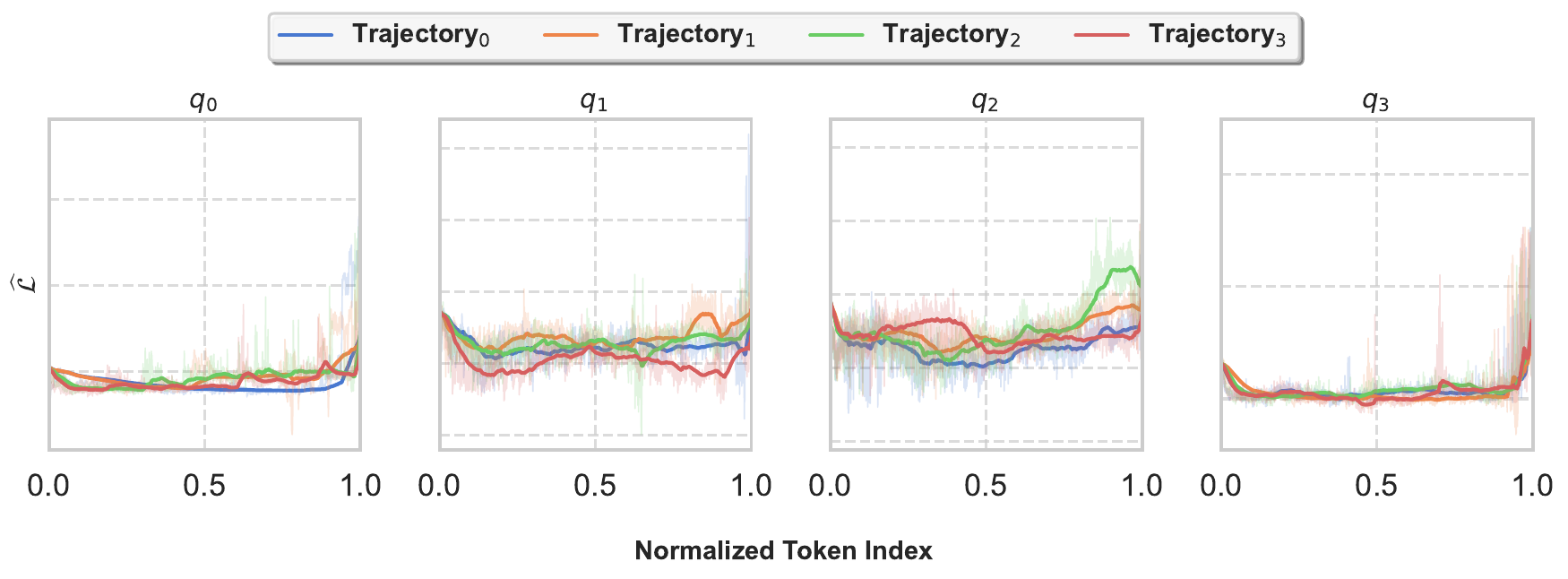}
    \end{subfigure}
    \begin{subfigure}{\textwidth}
        \centering
        \includegraphics[width=.9\textwidth]{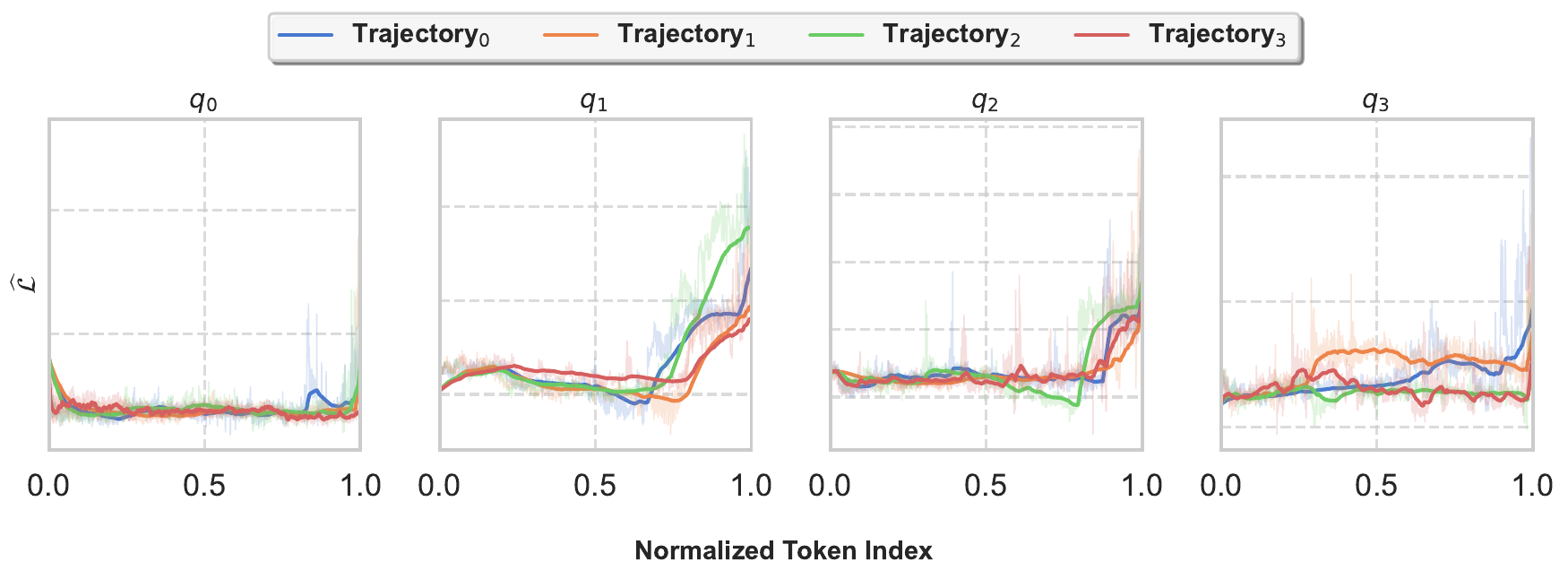}
    \end{subfigure}
    \begin{subfigure}{\textwidth}
        \centering
        \includegraphics[width=.9\textwidth]{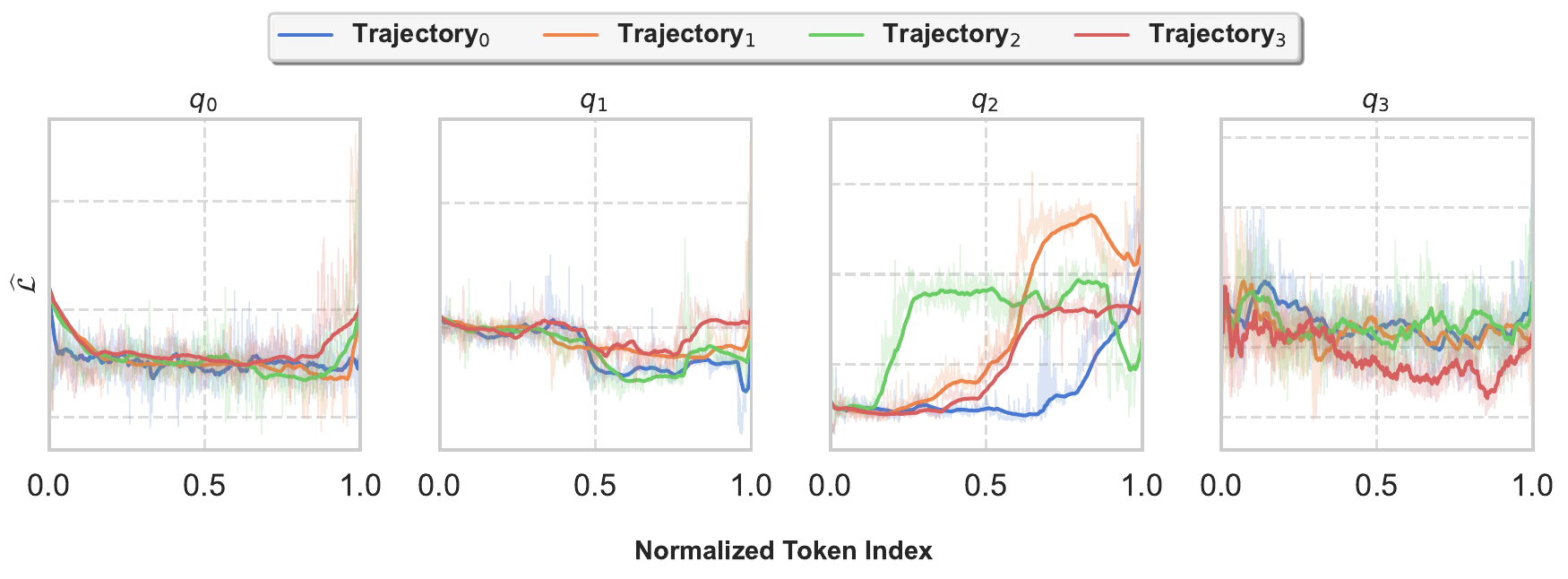}
    \end{subfigure}
    \begin{subfigure}{\textwidth}
        \centering
        \includegraphics[width=.9\textwidth]{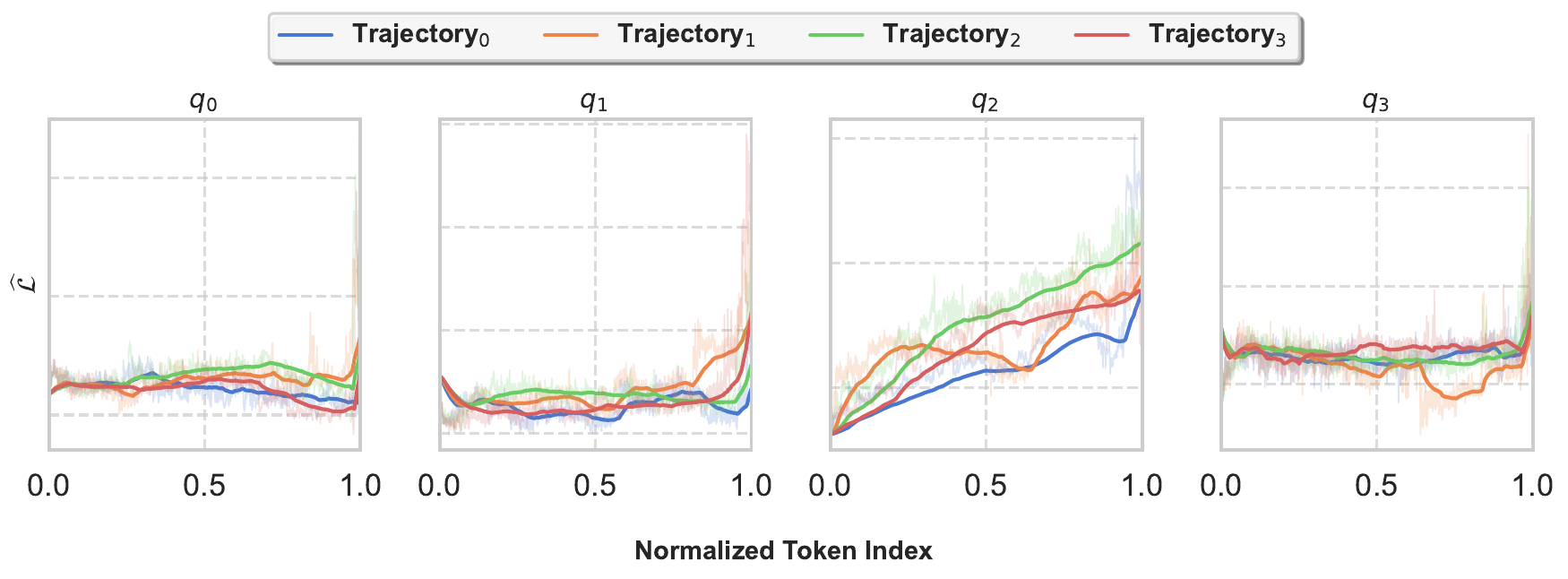}
    \end{subfigure}
    \caption{Illustration Qwen3's \textit{pseudo-gradient update} corresponding to false reasoning trajectories.}
\end{figure}

\begin{figure} \ContinuedFloat
    \begin{subfigure}{\textwidth}
        \centering
        \includegraphics[width=.9\textwidth]{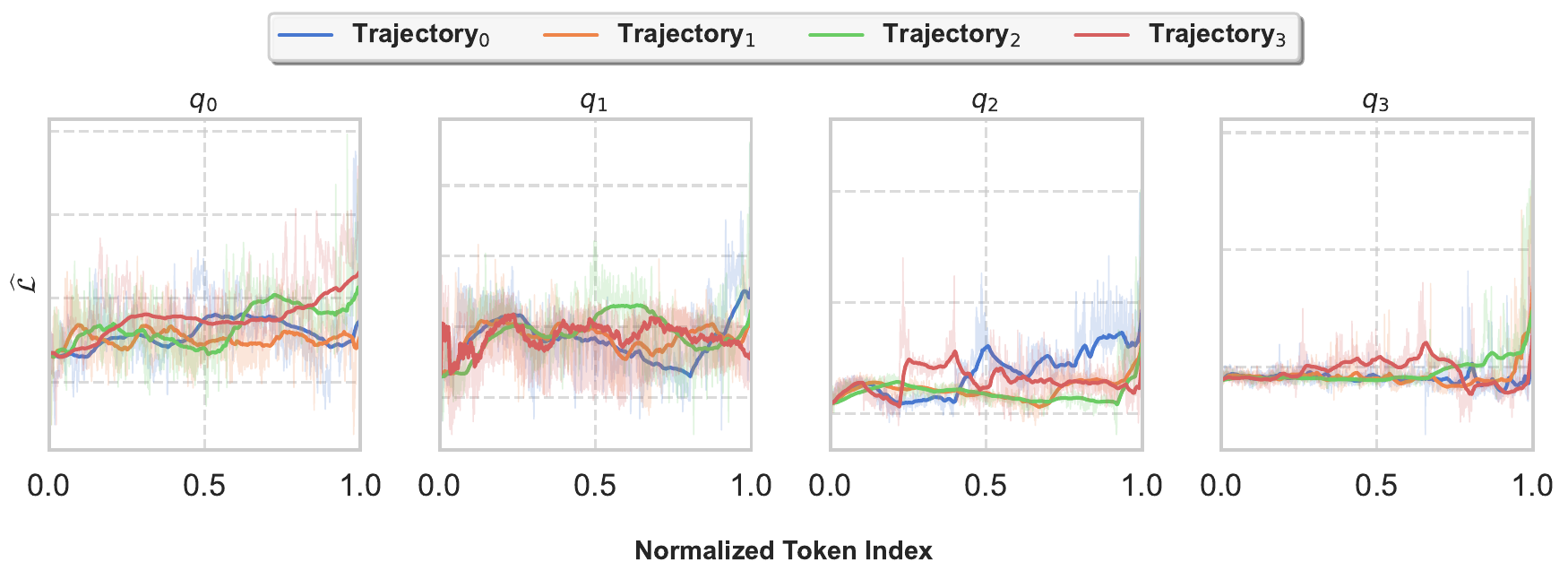}
    \end{subfigure}
    \caption{Illustration Qwen3's \textit{pseudo-gradient update} corresponding to false reasoning trajectories.} \label{fig:qwen_nothinking_false_gradient_decent}
\end{figure}



\end{document}